\documentclass[10pt,journal,compsoc]{IEEEtran}

\ifCLASSOPTIONcompsoc
  \usepackage[nocompress]{cite}
\else
  \usepackage{cite}
\fi
\ifCLASSINFOpdf
\else
\fi
\usepackage{amsmath}
\usepackage{algorithmic}

\usepackage{array}
\usepackage{algorithm}

\makeatletter
\let\NAT@parse\undefined
\makeatother
\usepackage{overpic}
\usepackage{amsfonts,amssymb}
\usepackage{booktabs}
\usepackage{makecell}
\usepackage{amsthm}

\usepackage{graphicx}
\usepackage{multirow}
\usepackage{xurl}
\usepackage{wrapfig,lipsum}
\usepackage{float}
\usepackage[numbers]{natbib}
\usepackage{xcolor}
\usepackage{pifont}
\newtheorem{lemma}{Lemma}

\newtheorem{theorem}{Theorem}
\newtheorem{corollary}{Corollary}

\newtheorem{remark}{Remark}
\usepackage{subfigure}

\usepackage[colorlinks,urlcolor=blue,anchorcolor=blue,citecolor=blue]{hyperref}

\begin{document}
%
\title{Connection Sensitivity Matters for Training-free DARTS: From Architecture-Level Scoring to Operation-Level Sensitivity Analysis}

%
%
%
%

\author{Miao Zhang,~\IEEEmembership{Member,~IEEE}, Wei Huang, Li Wang
\IEEEcompsocitemizethanks{
\IEEEcompsocthanksitem  M. Zhang is with School of Computer Science and Technology, Harbin Institute of Technology (Shenzhen), China}
\thanks{Manuscript received April 19, 2005; revised August 26, 2015.}}

%
%

\markboth{Journal of \LaTeX\ Class Files,~Vol.~14, No.~8, August~2015}%
{Shell \MakeLowercase{\textit{et al.}}: Bare Advanced Demo of IEEEtran.cls for IEEE Computer Society Journals}
%



\IEEEtitleabstractindextext{%
\begin{abstract}

The recently proposed training-free NAS methods abandon the training phase and design various \textit{zero-cost proxies} as scores to identify excellent architectures, arousing extreme computational efficiency for neural architecture search. However, recent works have observed that most training-free NAS, based upon the predictor-based NAS, encounters a \textit{parameter-intensive bias} that makes them prefer larger models when evaluating the whole model by summing up parameter-wise scores of all parameters. Different from predictor-based NAS, another more popular and simpler paradigm, \textit{Differentiable ARchiTecture Search} (DARTS), only compares among candidate operations for each edge, a.k.a. evaluate edge connectivity, after training the supernet to convergence. In this paper, we raise an interesting problem: \textit{{can we properly measure the operation importance in DARTS through a training-free way, with avoiding the parameter-intensive bias?}} We investigate this question through the lens of edge connectivity, and provide an affirmative answer by defining a connectivity concept, \textit{ZERo-cost Operation Sensitivity (\textbf{\textit{ZEROS}})}, to score the importance of candidate operations in DARTS at initialization. By devising an iterative and data-agnostic manner in utilizing \textit{ZEROS} for NAS, our novel trial leads to a framework called \textit{training free differentiable architecture search} (\textbf{FreeDARTS}). Based on the theory of Neural Tangent Kernel (NTK), we show the proposed connectivity score provably negatively correlated with the generalization bound of DARTS supernet after convergence under gradient descent training. In addition, we theoretically explain how \textit{\textit{ZEROS}} implicitly avoid \textit{parameter-intensive bias} in selecting architectures, and empirically show the searched architectures by FreeDARTS are of comparable size. Extensive experiments have been conducted on a series of search spaces, and results have demonstrated that FreeDARTS is a reliable and efficient baseline for neural architecture search. 
\end{abstract}

}

\maketitle

\IEEEdisplaynontitleabstractindextext

%
\IEEEpeerreviewmaketitle

\section{Introduction}\label{sec:introduction}
Neural Architecture Search (NAS) \cite{ren2020comprehensive} automates the neural network design process and has, therefore, received broad attention. However, the early NAS methods are often heavily computational-expensive since they require to evaluate a large number of architectures during the search \cite{real2018regularized,guo2018irlas}. To relieve this computational burden for efficient search, more recent studies \cite{bender2018understanding,zhang2018graph} shift to design performance proxies in evaluating architectures, so as avoiding truly training numerous architectures. A well-known paradigm is the \textit{weight-sharing}, which only trains a supernet as a predictor to evaluate all candidate architectures through inheriting weights from the supernet without extra training. In \cite{white2021powerful}, a performance predictor is trained based on a small portion of architectures along with their ground-truth performance, and the performance of remaining architectures is estimated based on the predictor, to boost the search efficiency, as shown in Fig. \ref{fig:difference} (a). Instead of devising proxies that require training, several recent studies \cite{mellor2020neural,shu2021nasi,chen2021neural} try to reveal how good an architecture is via a training-free way. A representative work is the \textit{zero-cost NAS} \cite{abdelfattah2021zero}, which introduces the training-free proxies in network pruning-at-initialization (\textbf{PaI}) to NAS, through {summing up} parameter-wise sensitivities of all parameters as the \textit{\textbf{architecture-level score}} to justify an architecture. Although this paradigm is extremely efficient, more recent studies \cite{ning2021evaluating,2022blog} empirically find it has a \textit{\textbf{parameter-intensive bias}} to select larger architectures, and those training-free metrics are highly correlated with the number of parameters \cite{yang2023revisiting}. More specific, since these zero-cost proxies are designed as non-negative, the final ``summing-up'' metric of an architecture with more parameters is supposed be higher, leading to an excessive preference for large architectures. Several recent works \cite{ning2021evaluating,2022blog} empirically and theoretically verified this phenomenon. As the large neural networks are generally supposed to have lower test loss \cite{yang2022neural}, the competitive performance from those searched architectures could hardly justify the effectiveness of training-free NAS methods, especially when recent works \cite{shu2022unifying} find that selecting architectures according to the number of parameters can achieve similar competitive performance as most existing training-free NAS methods.

\begin{figure*}[ht]
\centering
 \subfigure[Predictor-based NAS.]{
\includegraphics[height=4cm]{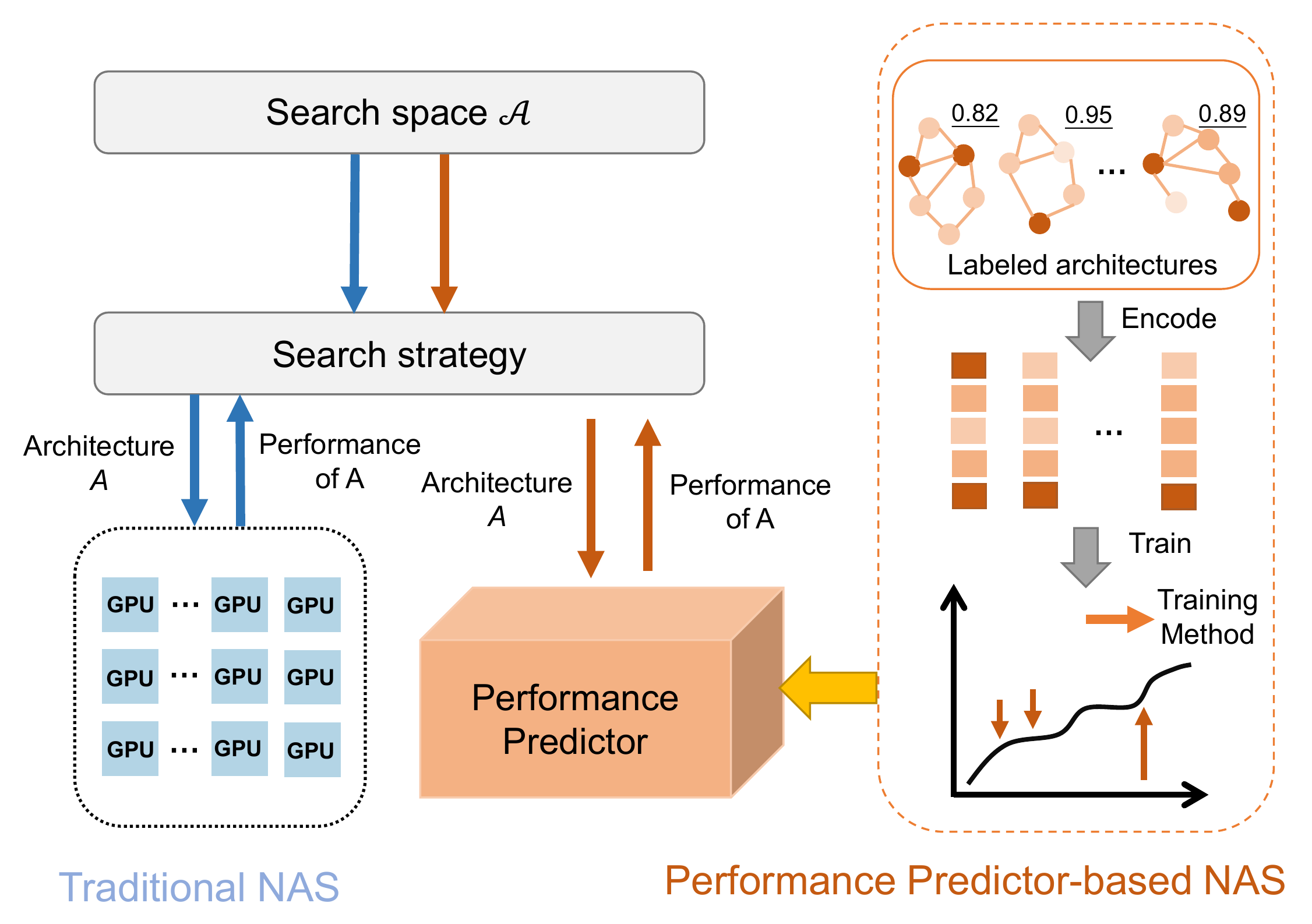}
}
 \subfigure[Differentiable NAS.]{
\includegraphics[height=4cm]{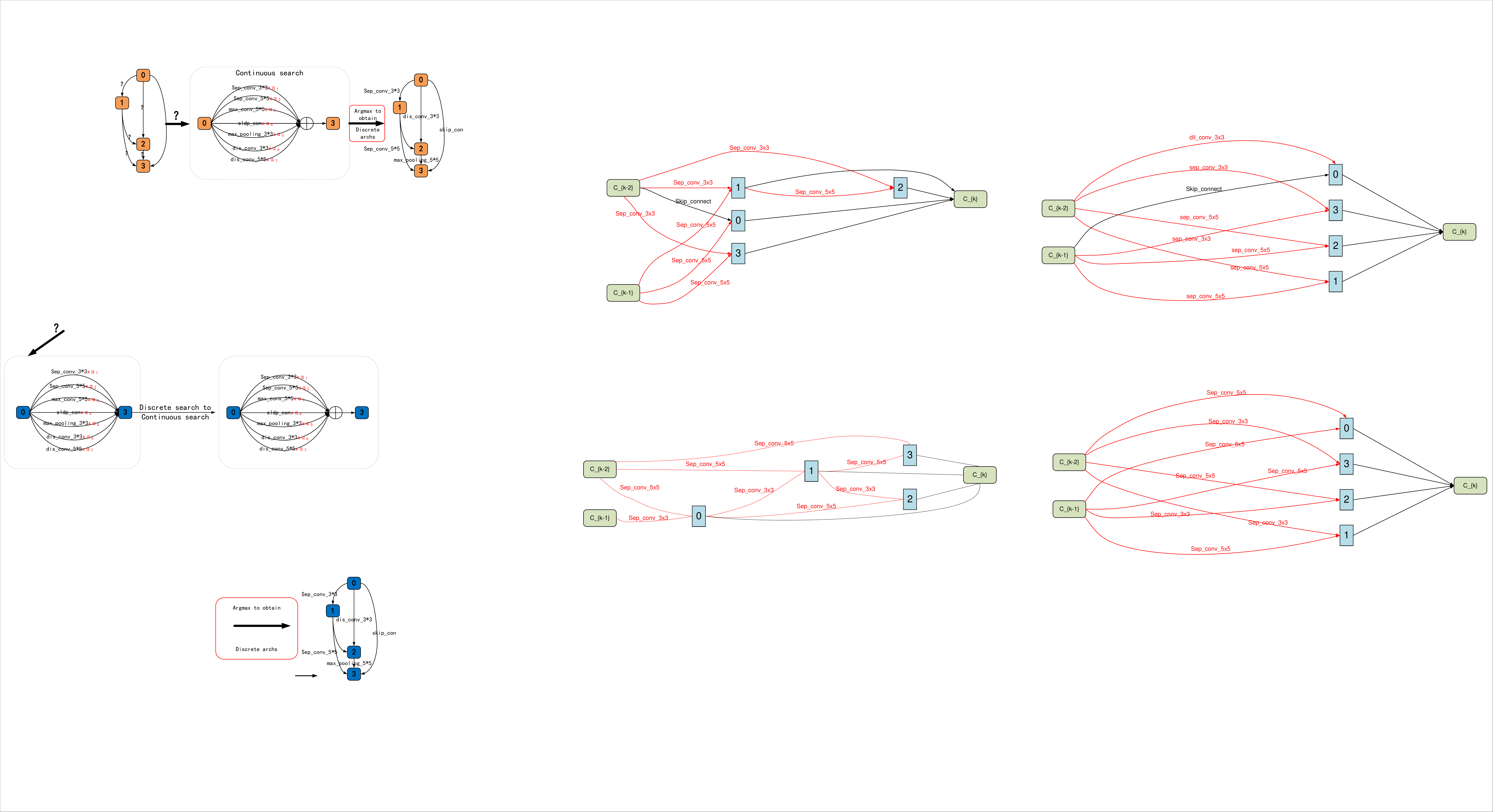}
}
 \caption{The main difference between predictor-based NAS and differentiable NAS. As shown, the predictor-based NAS requires evaluating the whole architecture while differentiable NAS only needs to compare among candidate operations, a.k.a. edge connectivity.}
 \label{fig:difference}
\end{figure*}

Different from the above methods, which can be categorized as the \textit{predictor-based NAS} as they require a predictor or proxy to score the whole architecture, another more popular and much simpler paradigm, \textit{differentiable NAS} or \textit{Differentiable ARchiTecture Search} (DARTS) \cite{liu2018darts}, only compares the strength among candidate operations for each edge. Figure \ref{fig:difference} visualizes the main difference between predictor-based NAS and differentiable NAS. As shown, different from predictor-based NAS that requires a predictor to compare among architectures, differentiable NAS only needs to select the most important operation for each edge. With continuous relaxation, DARTS converts the operation selection into a magnitude optimization problem, whereby the optimized magnitude is treated as the \textit{\textbf{operation-level score}} to measure the operation strength, as shown in Fig. \ref{fig:difference} (b). However, recent efforts \cite{Rethinking2021} showed the optimized magnitude after training is unreliable to indicate the importance of operations. Furthermore, DARTS is memory-unfriendly as it requires training the whole supernet in each step of architecture search. Intuitively, the above observations motivate us to abandon training in evaluating operation strength, instead of leveraging the trained magnitude as existing differentiable NAS.




In this paper, we aim to introduce the training-free paradigm into the \textit{differentiable NAS}, with raising the question: \textit{{can we properly measure the operation importance in DARTS through a training-free way, with avoiding the parameter-intensive bias?}} We affirmatively answer this question and well define a concept, called \textit{ZERo-cost Operation Sensitivity (\textbf{\textit{ZEROS}})}. Different from \cite{xiang2021zero,chen2021neural} that use a perturbation paradigm to score the operation importance while which also encountering parameter-intensive bias as discussed in Sec. \ref{sec:2.1}, our \textit{ZEROS} can directly score the edge connectivity in DARTS at initialization. By devising an iterative and data-agnostic manner in utilizing \textit{ZEROS} for NAS, we accordingly propose a novel framework called \textit{training free differentiable architecture search} (\textbf{FreeDARTS}). With leveraging the theory of Neural Tangent Kernel (NTK), we theoretically show the proposed connectivity score provably negatively correlated with the generalization bound of DARTS supernet after convergence under gradient descent training. In addition, we also theoretically explain how \textit{\textit{ZEROS}} implicitly avoid \textit{parameter-intensive bias} in selecting architectures, where the searched architectures by FreeDARTS are of comparable size. More inspiring, FreeDARTS can perform architecture search in a remarkably efficient way on any search space by discarding the memory and computation-consuming supernet training part, and also achieve competitive results against SOTA. Our work provides a competitive baseline and opens up a promising direction for NAS from the perspective of operation pruning, as which provides a more \textbf{reliable}, and \textbf{efficient} framework. 

The contributions of this paper are in four-fold:



\begin{itemize}

\item First, we define a concept, call \textit{ZERo-cost Operation Sensitivity (\textbf{\textit{ZEROS}})}, to compare edge connectivity in DARTS without training. With the theory of NTK, we show the proposed \textit{\textit{ZEROS}} provably negatively correlated with the generalization bound of supernet after convergence with gradient descent training, which can help find the architectures with the highest generalization performance (see Theorem \ref{theorem:generalization}).

\item Second, we perform a theoretical analysis on the
\textbf{parameter-intensive bias} in existing zero-cost NAS methods
(see Remark \ref{remark_bia}), and also theoretically demonstrate that the proposed \textit{\textit{ZEROS}} can implicitly relieve it (see Remark \ref{remark_relieve}). Experiments empirically verify that our operation-level scoring avoids the parameter-intensive bias.


\item Third, by devising an iterative and data-agnostic manner in utilizing \textit{\textit{ZEROS}} for NAS, we propose a novel training-free differentiable architecture search framework called \textbf{FreeDARTS} (see Algorithm \ref{alg:algorithm1}). The experimental results demonstrate FreeDARTS is a \textbf{reliable} solution to differentiable architecture search, which achieves competitive
performance on different search spaces, including 4 commonly-used NAS spaces \cite{BENCH102,zela2020nasbench1shot1,liu2018darts,howard2019searching}.

\item Fourth, FreeDARTS is hundreds of times more \textbf{efficient} than existing zero-cost NAS methods, completing the architecture search with only \textit{3.6s}, \textit{6.3s}, \textit{8.5s}, and \textit{12.7s} on the commonly-used NAS-Bench-201, NAS-Bench-1Shot1, DARTS space, and MobileNet space, respectively.


\end{itemize}

\section{Background}
This paper focuses on the differentiable architecture search under a training-free paradigm, with leveraging the theory of Neural Tangent Kernel (NTK) to theoretically verify the effectiveness of the proposed method. In this section, we briefly introduce the background of training-free NAS, differentiable architecture search, and Neural Tangent Kernel.
\subsection{Training-free Neural Architecture Search}
\label{sec:2.1}
Training-free NAS tries to identify promising architectures at initialization without incurring training. Zero-cost NAS \cite{abdelfattah2021zero} is the first work that leverages the scores at network initialization to estimate architecture performance. Specifically, \citet{abdelfattah2021zero} adapted saliency metrics in network pruning-at-initialization, including gradient, fisher metric, jacobian covariance, SNIP \cite{lee2018snip}, GraSP \cite{wang2019picking}, SynFlow \cite{tanaka2020pruning}, and so on, to evaluate the importance of every weight in a network. Zero-cost NAS is also the first work to propose the ``summing-up" manner, where the architecture score $\mathcal{M}$ is obtained by \textbf{summing up} parameter-wise scores of all parameters in an architecture, $\mathcal{M}=\sum_{i}^{M}\left | S_i\right |$, and $S_i$ can be any parameter-wise scoring manner for the \textit{i-th} parameter, e.g. SNIP, GraSP, SynFlow, Fisher, and so on. In Table \ref{table:repre_methods}, we summarize representative methods of training free NAS, including predictor-based NAS and differentiable NAS.

\begin{table*}[t]
\scriptsize
\setlength{\tabcolsep}{3pt}
  \caption{Representative methods of training-free NAS.}
  \centering
  \label{table:repre_methods}
  \begin{tabular}{l|ccccc}
    \toprule  
    Method & Criteria& Scoring-level& Scoring-manner &Category\\
    \midrule
    Zero-cost NAS with Grad\_norm  \cite{abdelfattah2021zero} & $\left | \frac{\partial \mathcal{L}}{\partial \theta} \right |_{2}$& \makecell{\multirow{7}*{Architecture-level}} & \makecell{\multirow{7}*{Summing up}} & \makecell{\multirow{7}*{ Predictor-based NAS }}\\
    Zero-cost NAS with SNIP \cite{abdelfattah2021zero} & $\sum_{i=1}^{N}\left | \frac{\partial \mathcal{L}}{\partial \theta_i}\odot \theta_i  \right |$ \\
    Zero-cost NAS with GraSP\cite{abdelfattah2021zero} & $\sum_{i=1}^{N}\left | H\frac{\partial \mathcal{L}}{\partial \theta_i}\odot \theta_i  \right |$ \\
    Zero-cost NAS with SynFlow \cite{abdelfattah2021zero} & $\sum_{i=1}^{N}\left | \frac{\partial \mathcal{L_s}}{\partial \theta_i}\odot \theta_i  \right |$ \\
    Zero-cost NAS with Fisher  \cite{abdelfattah2021zero} & $\sum_{i=1}^{M}\left | \frac{\partial \mathcal{L}}{\partial z_i}\odot z_i  \right |$  \\
    \midrule
    Zen-NAS \cite{lin2021zen}&$\sum_{i=1}^{N}\sqrt{\sum_{j}\sigma_{i,j}^{2}/m }$ &Architecture-level & Summing up & Predictor-based NAS\\
    TE-NAS \cite{chen2021neural}&$\kappa_{\mathcal{N}}+R_{\mathcal{N}}$& Operation-level & Perturbation & Differentiable NAS\\
    Zero-Cost-PT \cite{xiang2021zero} &the above criteria$^*$& Operation-level & Perturbation & Differentiable NAS\\    
    KNAS\cite{knas} & $\frac{1}{n^2}\sum_{i=1}^{n} \sum_{j=1}^{n}\left ( \frac{\partial y_j^{(L)}(t)}{\partial \theta(t)} \right )\left ( \frac{\partial y_i^{(L)}(t)}{\partial \theta(t)} \right )^T$ & Architecture-level & Summing up & Predictor-based NAS\\
    NASI \cite{shu2021nasi}&$ \left \| \Theta_0(\mathcal{A}) \right \|_{\textup{tr}}$ &Architecture-level & Direct & Predictor-based NAS  \\
    NASWOT \cite{mellor2020neural} &$\textup{log}\left | \mathbf{K}_H \right |$ & Architecture-level & Direct & Predictor-based NAS\\
    GradSign \cite{gradsign}& $\sum_k| \sum_i \textup{sign}([\triangledown_\theta l(f_\theta(x_i),y_i)\mid_{\theta_0}  ]_k)| $ & Architecture-level & Summing up & Predictor-based NAS \\
    TF-TAS \cite{zhou2022training} & $\sum_l \sum_m  \left \| \frac{\partial \mathcal{L}}{\partial \theta_m} \right \|_{nuc}\odot \left \| {\theta_m} \right \|_{nuc}+\sum_l \sum_n  \frac{\partial \mathcal{L}}{\partial \theta_n}\odot {\theta_n}$ & Architecture-level & Summing up & Predictor-based NAS\\
    Our FreeDARTS & $\left | \frac{\partial \mathcal{L}( W, \alpha)}{\partial \alpha_k}\cdot \alpha_k \right |$ & Operation-level & Direct & Differentiable NAS\\    
    \bottomrule
\end{tabular}
\flushleft{Note that ``Summing-up" denotes individually calculating every parameter's score where the architecture (or operation) score is obtained by summing all parameters' scores. ``Perturbation"  is usually used to score operations based on supernet changes when removing one candidate operation. ``Direct" treats the architecture (or operation) as a whole and the score is directly calculated without summing-up. *Zero-Cost-PT \cite{xiang2021zero} introduced the above 7 training-free scores into its framework.}
\end{table*}

Rather than using the saliency metrics which are defined as the Hadamard product in network pruning, \citet{knas} leveraged the mean of the Gram matrix of gradients to evaluate the quality of architectures. \citet{mellor2020neural} empirically found that the correlation between sample-wise input-output linear regions could indicate the architecture’s test performance, and proposed accounting linear regions to score a set of randomly sampled models with randomly initialized weights. In \cite{knas}, the authors further extend Gram matrix of gradients (GM) to measure the convergence rate, for evaluating architectures. With the theory of Neural Tangent Kernel (NTK), \citet{shu2021nasi} proposed a performance estimator based on the trace norm of the NTK matrix with initialized model parameters. In \cite{zenscore}, the authors introduce the Zen-Score measures to measure the architectures' expressivity as which positively correlates with the model accuracy. \citet{gradsign} analyses the the sample-wise optimization landscape of different networks to score an a randomly initialized network. Recent work \cite{zhou2022training} introduces the training free NAS to search for the transformers, where the synaptic diversity and synaptic saliency for all weights are summed-up to rank ViTs. DisWOT \cite{DisWOT} further extend training free NAS into knowledge distillation, to efficiently search for the best student architectures for a given teacher.

Instead of proposing a training-free metric, \citet{shu2022unifying} theoretically studied relationships among existing zero-cost proxies, to demonstrate that these training-free metrics provide a similar estimation of the generalization performance of architectures, leading to similar performance. Although incorporating the zero-cost proxies to NAS can extremely accelerate the search phase, recent studies \cite{ning2021evaluating} find most training-free NAS methods have a \textit{parameter-intensive bias}, and those training-free metrics are highly correlated with the number of parameters \cite{yang2023revisiting}, leading to an excessive preference for large architectures. The discussion in Sec.5.2 of \cite{ning2021evaluating} empirically verifies this bias, where zero-cost NAS methods generally selecte architectures containing operations with the most parameters for every edge. 




As shown in Table \ref{table:repre_methods}, most existing training-free NAS methods consider the predictor-based paradigm, and the most relevant studies to our paper are \cite{xiang2021zero,chen2021neural}, which consider a similar setting that finishes the differentiable architecture search without training. \citet{xiang2021zero} directly extend zero-cost proxies in \cite{abdelfattah2021zero} into differentiable NAS through leveraging perturbation as \cite{Rethinking2021} for the supernet prune. More specifically, they sum up the parameter-wise zero-cost scores of the supernet, and iteratively calculate the change of supernet score in removing each candidate operation to measure the operation importance. Similar to the zero-cost NAS, they also encounter parameter-intensive bias as operations containing more parameters usually receive a higher score. TE-NAS \cite{chen2021neural} also considers the perturbation-based paradigm to measure the operation importance, while which utilizes the spectrum of NTKs and the number of linear regions to analyze the trainability and expressivity of supernet. 
However, the trainability has a positive correlation with the network model size as a deep model usually obtains higher trainability \cite{yang2022neural}, which may also lead to parameter-intensive models. In addition, since the calculation of NTK is not easy, TE-NAS can only improve efficiency to a limited extent. Although these concurrent works introduce the training-free setting into the differentiable architecture search, they still inherit the architecture-level scoring to obtain the candidate operation importance (a.k.a. edge sensitivity) through a perturbation manner, leading to a similar parameter-intensive bias as the existing predictor-based training free NAS. This paper also focuses on the training-free differentiable architecture search, where we propose the zero-cost operation sensitivity (\textit{\textbf{\textit{ZEROS}}}) to replace the trained magnitude in DARTS to directly measure the importance of candidate operations.

\subsection{Differentiable Architecture Search}
Different from traditional NAS that selects a competitive architecture by evaluating numerous architectures, {Differentiable ARchiTecture Search (DARTS)} \cite{liu2018darts} only compares candidate operations for each edge in a supernet. By adopting the continuous relaxation to convert the operation selection into the continuous magnitude optimization, DARTS enables gradient descent for architecture search which can greatly improve efficiency. DARTS can be formulated as a bi-level optimization problem:
\begin{equation} \label{eq:darts_1}
 \begin{aligned}
&\underset{\alpha}{\textup{min}}\ \ \ \mathcal{L}_{val}(W^*(\alpha),\alpha) \\
&\textup{s.t.}\ \ \ W^*(\alpha)=\textup{argmin}_{W}\ \mathcal{L}_{train}(W,\alpha),
\end{aligned}
\end{equation}
where $\alpha$ is the architecture parameter, also called as operation magnitude, and $W$ is the supernet weights. After the bi-level optimization, a promising architecture is obtained by selecting most promising operation based on the optimized magnitude for each edge in the supernet. 

Despite notable concision of DARTS, more recent works find the optimized operation magnitudes are unreliable \cite{zela2019understanding,chen2020stabilizing}. For example, \citet{Rethinking2021} showed that the magnitude of the architecture parameter obtained by DARTS after supernet training is fundamentally wrong, where the optimized magnitude could hardly indicate the importance of operations. Rather than utilizing the optimized magnitudes to indicate the operation strength, the authors found the perturbation-based operation influence on supernet can more consistently extract significance of operation than magnitude-based counterparts. In addition, DARTS is unable to stably obtain excellent solutions as which yields deteriorative architectures with the search proceeding, performing even worse than random search in some cases \cite{sciuto2019evaluating}. \citet{zela2019understanding} interrupted the search based on the dominant eigenvalue of the Hessian matrix, and \citet{liang2019darts+} introduced another simple ``early stopping'' criteria, where the search procedure ended as one cell had two or more \textit{skip-connection} operations. They all show the supervised learning paradigm may be harmful in finding a promising architecture in DARTS.

\subsection{Neural Tangent Kernel}

With an over-parameterized deep neural network (DNN), a lot of studies on deep learning theory find that the gradient descent training can find global minimum \cite{du2019gradient,du2018gradient,allen2019convergence}. For example, \citet{jacot2018neural} resorted to the Neural Tangent Kernel (NTK) to understand DNNs, where they found that, under a mild assumption that the lowest eigenvalue of the NTK is greater than zero, the loss that is convex regards output would convergence to the global minimum. Give a $L$-layer DNN with width $m_l$ in each layer parameterized as $\theta$, let $f(x,\theta)$ be the output of this DNN with input $x\in \mathbb{R}^{n_0}$ on a dataset with size $m$, the NTK matrix $\Theta (x,x';\theta)\in \mathbb{R}^{m\times m}$ over the input dataset for the DNN is defined as:
\begin{equation} \label{eq:ntk}
\Theta (x,x';\theta)=\nabla_\theta f(x,\theta) \nabla_\theta f(x,\theta)^\top.
\end{equation}
In general, the NTK varies during the training, however, as shown by \cite{jacot2018neural}, the NTK will converge to an explicit limiting kernel $\Theta_\infty$ which will stay constant during training that $\Theta_0= \Theta_\infty$ in the infinitely wide DNNs. Recent work \cite{allen2019convergence} further proves that this similar phenomenon can be obtained in an over-parameterized DNN with finite width that that $\Theta_0 \approx \Theta_\infty$. More interestingly, \cite{arora2019exact} and \cite{lee2019wide} have revealed that the training dynamics of DNNs can be well-characterized with NTK matrix at initialization. Several studies on NAS \cite{chen2021neural,shu2021nasi} also resort to NTK to find competitive structures at initialization.


\section{Methodology}
This paper aims to investigate measuring the operation importance of the supernet for DARTS at initialization. More specifically, we try to measure the edge connectivity in DARTS without any training. 
In this section, we first theoretically revisit the zero-cost NAS \cite{abdelfattah2021zero} in a unified view, along with discussing several other existing training-free NAS methods. Following \cite{shu2022unifying}, we theoretically unify those metrics, and demonstrate most existing training-free NAS methods can bring \textit{parameter-intensive bias}, since which are correlated with the number of parameters. Then, we propose the \textit{zero-cost operation sensitivity (\textbf{\textit{ZEROS}})} to compare candidate operations in DARTS without any training. By leveraging the theory of NTK, we theoretically show the \textit{\textit{ZEROS}} can help find architectures with high generalization performance and implicitly regularize the model size. To make our ZEROS more effective in finding excellent architecture, we further design an iterative manner to leverage ZEROS, and custom a label-agnostic ZEROS and a data-agnostic ZEROS. Our \textit{training free differentiable architecture search} (\textbf{FreeDARTS}) is also proposed accordingly.

\subsection{Revisit Zero-cost NAS from a Theoretical View}
Zero-cost NAS \cite{abdelfattah2021zero} is the first paper to leverage training-free saliency metrics (those metrics can be obtained from an initialized architecture) from  network pruning at initialization for the neural architecture search. Following efforts \cite{xiang2021zero,chen2021neural,knas} further design other training-free metrics, and empirically demonstrate these metrics can help find excellent architecture in practice, which also follow the summing-up manner. However, more recent studies \cite{yang2023revisiting} find those training-free metrics through the ``summing-up" manner have a high correlation with the number of parameters, leading to an excessive preference for large architectures \cite{ning2021evaluating}. In this subsection, we try to theoretically explain why this phenomenon happens.

Different from zero-cost NAS that uses the saliency metrics in network pruning, the following-up works \cite{knas,zhou2022training,lin2021zen} try to design different scores while still considering the summing-up manner to evaluate the quality of architectures. By leveraging the trace norm of NTK matrix at initialization, a recent work \cite{shu2022unifying} theoretically connected the above metrics which could provide similar characterization for the generalization performance of neural architectures, as shown in the following lemma.
\begin{lemma}
\label{lemma_unify}
\cite{shu2022unifying} When assuming the loss function $\mathcal{L}$ in calculating zero-cost scores \cite{abdelfattah2021zero} is $\beta$-Lipschitz continuous and $\gamma$-Lipschitz smooth, there exist the constant $C_1 > 0$ such that
\begin{equation} 
\label{eq:zero-cost-connect}
\mathcal{M}=\sum_{i}^{M}S_i \leq C_1 \cdot \mathcal{M}_\textup{Trace}, 
\end{equation}
where $M$ is the number of parameter in an architecture, and $S_i$ can be parameter-wise scoring manner for the \textit{i-th} parameter, e.g. gradient, SNIP, GraSP, SynFlow, and so on \cite{abdelfattah2021zero}. $\mathcal{M}_\textup{Trace}$ is the trace norm that $\mathcal{M}_\textup{Trace}=\sqrt{\left \| \Theta_0  \right \|_{\textup{tr}}/n}$, and $\Theta_0$ is the NTK matrix based on initialized model parameters $\theta$ on dataset $S=\left \{ (x_i,y_i )\right \}_{i=1}^{n}$.
\end{lemma}
The main proof of Lemma \ref{lemma_unify} can be found in \cite{shu2022unifying}. In \cite{shu2022unifying}, the authors also extend it to the Gram matrix of gradients of KNAS \cite{knas}. The above lemma reveals that zero-cost NAS with different metrics prefers architecture with higher trace norm, leading to similar performance. However, more recent works \cite{yang2022neural,yang2023revisiting} found that architectures with more parameters were supposed to have a higher trace norm. We formally state this relationship in the following theorem which theoretically reveals architectures with more parameters are supposed to have higher trace norm.

\begin{theorem}\label{lemma_complexity} \textbf{More parameters, larger trace norm.}  Consider two neural networks, $f$ and $\hat f$, which have the same structure except that the width of $h$-th layer for the two networks are different: network $f$ has $m$ neurons while network $\hat f$ has $\hat m$ neurons, and $ \hat m = \rho m $ with $0 < \rho < 1 $. Then, as the width goes to infinity 
\begin{equation}
  \lim_{ m \rightarrow \infty }  \mathcal{M}_\textup{Trace}(\hat f)  = \sqrt{\rho}   \lim_{ m \rightarrow \infty } \mathcal{M}_\textup{Trace}(f) 
\end{equation} 
\end{theorem}
The main proof of Theorem \ref{lemma_complexity} can be found in Appendix \ref{proof_section}. We connect Theorem \ref{lemma_complexity} and Lemma \ref{lemma_unify}, and explain why parameter-intensive bias in existing training-free NAS under the summing-up manner in Remark \ref{remark_bia}.

\begin{remark}\label{remark_bia} 
\textit{\textbf{Parameter-intensive bias in existing training-free NAS.}} Theorem \ref{lemma_complexity} indicates that an architecture with more parameters is supposed to have higher trace norm $\mathcal{M}_\textup{Trace}$, which will lead to larger zero-cost metric values due to the connections shown in Lemma \ref{lemma_unify}. The two lemmas reveal why zero-cost metrics are correlated with the number of parameters--An architecture with more parameters is supposed to have higher zero-cost scores. As a result, selecting architectures based on the zero-cost metrics in \cite{abdelfattah2021zero} has a {parameter-intensive bias} toward larger models.
\end{remark}

Apart from zero-cost NAS \cite{abdelfattah2021zero}, other summing-up based training-free NAS, e.g., Zen-NAS, TF-TAS, and GradSign, also easily encounter parameter-intensive bias since more parameters are likely to incur larger final score after summing-up non-negative individual scores. More interesting, \cite{shu2021nasi} states that maximizing $ \left \| \Theta_0(\mathcal{A}) \right \|_{\textup{tr}}$ leads to architectures with large complexity, and the authors add a constraint to avoid parameter-intensive bias. As discussed in \ref{sec:2.1}, the zero-cost-PT also has a bias to select operations with the most parameters for every edge, since removing more parameters will change the supernet score more. As trainability has a positive correlation with the network model size \cite{yang2022neural}, considering more trainability in TE-NAS \cite{chen2021neural} also leads to parameter-intensive model.


\subsection{ZERo-cost Operation Sensitivity ({\textit{ZEROS}})}
\label{sec3.2}

Rather than scoring the importance of weights before summing up, this paper aims to design a metric to directly score the importance of each operation $\alpha_k$. As described in Eq.\eqref{eq:darts_1}, there are two parts of parameters needed to be optimized in DARTS, supernet weights $W$ and architecture parameters $\alpha$. The final stage of DARTS can be seen as operation-level pruning, where DARTS calculates the magnitude of operations after training to prune the supernet. This stage is also called discretization, where a discrete architecture is derived from a continuous representation. More interesting, this discretization stage is very similar to the magnitude-based network pruning \cite{han2015learning} where the redundant weights are determined by their magnitude after training. However, training-based network pruning is extremely time-consuming, and researchers shift to designing training-free saliency metrics to evaluate the weight-level connection sensitivity, to serve the network pruning at initialization. Motivated by the above, the intuitive purpose of this paper is to design a training-free operation metric to replace the optimised operation magnitude after training in DARTS. 

We borrow the concept of connection sensitivity from the network pruning at initialization \cite{wang2022recent}, e.g. SNIP \cite{lee2018snip}, while we focus on the operation connection rather than the weight connection. More specifically, we define a concept of how removing a candidate operation $\alpha_k$ will affect the validation loss of the supernet to measure the connection sensitivity at initialization:
\begin{equation} \label{eq:free_nas}
\resizebox{.9\linewidth}{!}{$
\mathcal{F}({\alpha_k})=\lim_{\epsilon \rightarrow 0}\left | \frac{\mathcal{L}(W, \alpha+\epsilon\delta_{k})-\mathcal{L}( W, \alpha)}{\epsilon}\right |=\left | \frac{\partial \mathcal{L}( W, \alpha)}{\partial \alpha_k}\cdot \alpha_k \right |,\\
$}
\end{equation}
which we call as \textbf{\textit{zero-cost operation sensitivity (\textit{ZEROS})}}. From Eq.\eqref{eq:free_nas}, we can see that $\mathcal{F}({\alpha_k})$ is to measure the sensitivity of $\alpha_k$ which is similar to the network pruning, while focusing the architecture parameter $\alpha$ rather than the model weights $W$. Instead of using the summing-up paradigm \cite{abdelfattah2021zero,xiang2021zero}, our \textit{\textit{ZEROS}} is able to directly measure the importance of operations.

\subsection{Theoretical Analysis on \textit{ZEROS}}
\label{sec3.3}

In the following, we theoretically demonstrate that the architecture obtained based on the proposed \textit{zero-cost operation sensitivity (\textit{ZEROS})} leads to high generalization performance, which can also implicitly avoid the parameters-intensive bias.


In theoretically analyzing {zero-cost operation sensitivity} for DARTS, we consider a simple situation that the supernet $f(\alpha) =  h^{(L)} $ with $L$ sequential nodes and there are $M$ candidate edges between two consecutive nodes. In addition, we only consider those operations containing parameters. The supernet can be expressed as follows:
\begin{equation}\nonumber
        \quad h^{(l)} = \sum_{k =1 }^M \alpha^{(l)}_{k} h^{(l)}_{k} \quad h_k^{(l)} = \sigma(\theta^{(l)}_k h^{(l-1)} ) 
\end{equation}
where $ l \in [1,L]$, $\theta^{(l)}_e$ is the parameter and $\sigma(\cdot)$ is the activation function. Following \cite{shu2022unifying}, Lemma \ref{lemma:trace} connect the proposed \textit{zero-cost operation sensitivity} metrics.

\begin{theorem}\label{lemma:trace} 
    \textit{\textbf{Connecting sensitivity with trace norm.}} When assuming loss function $\mathcal{L}$ is $\beta$-Lipschitz continuous and $\gamma$-Lipschitz smooth, there exist constants $B> 0$ such that
\begin{equation} \label{eq:trace}
\begin{aligned}
    & \mathcal{F}({\alpha_k}) \le B   |\alpha_k|  \mathcal{M}_{\mathrm{Trace}}(\Theta_k),  \\
     \end{aligned}
\end{equation}
where $\mathcal{M}_\textup{Trace}(\Theta_k)=\sqrt{\left \| \Theta^k_0  \right \|_{\textup{tr}}/n} =  \sqrt{ \sum_{i=1}^n \left \| \frac{\partial f_i}{\partial \theta_k}  \right \|_2^2 /n} $, and $\Theta^k_0$ is the NTK matrix based on initialized parameters $\theta_k$ of operation $k$. 
\end{theorem}

The proof is in Appendix \ref{proof_section}. Different from Lemma \ref{lemma_unify} that the zero-cost scores are only related to the trace norm, the proposed \textit{\textit{ZEROS}} metrics are related to the product of the trace norm and the randomly initialized magnitude $\alpha_k$. In the following Remark \ref{remark_relieve}, we state how \textit{ZEROS} avoids
parameter-intensive bias.
\begin{remark}\label{remark_relieve} 
\textit{\textbf{Avoiding parameter-intensive bias.}} Different from zero-cost NAS that selects architectures only based on the trace norm (Lemma \ref{lemma_unify}) which will bring the parameter-intensive bias (Remark \ref{remark_bia}), \textit{ZEROS} assigns a random noise (since $\alpha$ is randomly initialized) on the trace norm in selecting operations, so it can implicitly relieve the parameter-intensive bias as a by-product which avoid always selecting the operation with the highest trace norm (or the most parameters).
\end{remark}

In the following theorem, we derive the convergence rate of $\mathcal{L}$ for the supernet performance, which is correlated with the {zero-cost operation sensitivity}.

\begin{theorem} \label{theorem:generalization}
 Suppose dataset $S = \{ (x_i,y_i)\}_{i=1}^n$ are i.i.d. samples from a non-degenerate distribution $\mathcal{D}(x,y)$, and $m \ge {\rm poly}(n, \lambda_0^{-1}, \delta^{-1})$, where $\lambda_0 = \lambda_{\min}(\Theta) > 0$. Consider any loss function $\ell: \mathbb{R} \times \mathbb{R} \rightarrow [0,1]$ that is $1$-Lipschitz, then with probability at least $1-\delta$ over the random initialization, the supernet trained by gradient descent for $T \ge \Omega(\frac{1}{\eta \lambda_0} \log \frac{n}{\delta})$ iterations has population risk $\mathcal{L}_\mathcal{D} = \mathbb{E}_{(x,y)\sim \mathcal{D}(x,y)}[\ell(f_T(x; \theta),y)]$ that is bounded as:
\begin{equation}\label{eq:gap}
\begin{small}
\begin{aligned} 
\mathcal{L}_{\mathcal{D}} \le \mathcal{L}_{\mathcal{S}} + \frac{C}{\mathcal{M}_\textup{Trace}(\Theta, \alpha)} + O \bigg(\sqrt{\frac{\log \frac{n}{\lambda_0 \delta}}{n}} \bigg) .
\end{aligned}
\end{small}
\end{equation}
where $C$ is a constant, $\mathcal{M}^2_\textup{Trace}(\Theta, \alpha)=\sum_{l=1}^{L}\sum_{k=1}^{M} \alpha^2_{lk} \mathcal{M}^2_{\mathrm{Trace}}(\Theta_{lk})$ and $(\Theta_{lk})$ is NTK matrix for the $m$-th operation in the $l$-th layer. 
\end{theorem}
The proof can be found in Appendix \ref{proof_section}.The second term $\frac{C}{\mathcal{M}_{\mathrm{Trace}}(\Theta, \alpha)}$ in Eq.\eqref{eq:gap} represents the generalization gap of supernet, which is correlated to the product of operation magnitude and trace norm. As shown in Lemma \ref{lemma:trace}, our \textit{ZEROS} $\mathcal{F}$ is correlated to this product, so as provides an explicit theoretical connection between $\mathcal{F}$ and the supernet performance, where preserving those operations with large \textit{ZEROS} score $\mathcal{F}$ can guarantee a lower supernet loss bound.

\begin{algorithm}[t]
\caption{FreeDARTS}
\label{alg:algorithm1}

\begin{algorithmic}[1]
\STATE \textbf{input}: Initialized supernet weights $W$ and architecture parameters $\alpha$; Set of edges $\mathcal{E}$ and candidate $\mathcal{O}$.
\WHILE{An architecture is not obtained}
\STATE Initilized supernet.
\FOR{\textbf{all} edge $e\in \mathcal{E}$}
\FOR{\textbf{all} operations $o\in \mathcal{O}_e$}
\STATE Calculate the operation saliency score for each operation $\alpha_{e,o}$ based on Eq. \eqref{eq:free_nas}.
\ENDFOR
\ENDFOR
\STATE Prune the candidate operation from the supernet with the lowest \textit{ZEROS} score $\mathcal{F}(\alpha_{e,o})$ for those edges contain more than one candidate operation;
\STATE Update the pruned supernet;
\ENDWHILE
\STATE \textbf{output:} Obtain a valid architecture $\alpha^*$.
\end{algorithmic}
\end{algorithm}

\subsection{Training Free Differentiable Architecture Search}
\label{sec3.4}

 
After defining the proposed connection sensitivity \textit{ZEROS}, we also raise two concerns when we leverage it to find competitive architectures:

\begin{enumerate}
    \item In the discretization stage, differentiable NAS needs to remove most candidate operations except one for each edge. However, our \textit{ZEROS} only evaluates the effect of removing each operation individually, without considering the dependencies between different candidate operations. 
    \item \textit{ZEROS} mainly focuses on edge sensitivity, especially how it impacts the information flow through the supernet when removing a candidate operation. This raises the question, of whether the input data are necessary since we can observe the information flow with any inputs \cite{tanaka2020pruning}. 
\end{enumerate}

\noindent\textbf{Iterative Operation Pruning with \textit{ZEROS}.}
To resolve the first concern, we consider an iterative manner to progressively remove inferior operations, where several recent works \cite{verdenius2020pruning,tanaka2020pruning} also empirically confirm that
iteration improves the performance of network pruning at initialization. Specifically, rather than pruning most operations by one shot as DARTS, we iteratively remove the least promising operation at each step. After each operation removal, we re-evaluate all remaining candidate operations in the supernet based on the \textit{ZEROS}. Since our \textit{ZEROS} score is extremely efficient to be obtained, this iterative way would not incur too much computational cost. With \textit{ZEROS} and the iterative manner, we can now easily implement our approach, which we call \textit{training free differentiable architecture search} (\textbf{FreeDARTS}). Given the initialized supernet $W$ and architecture parameter $\alpha$, we can calculate the operation strength based on Eq.\eqref{eq:free_nas} for all candidate operations. Then, we remove $\alpha_{e,o}$ with the least $\mathcal{F}({\alpha_{e,o}})$ and update $\alpha$. This process is looped when a discrete architecture is obtained. Algorithm \ref{alg:algorithm1} outlines our FreeDARTS for the differentiable neural architecture search. \textbf{The main difference between our FreeDARTS and DARTS includes that FreeDARTS replaces the trained magnitude $\alpha$ in DARTS with the proposed \textbf{\textit{\textit{ZEROS}}} $\mathcal{F}(\alpha)$, and the discrete architecture is obtained by iterative pruning operations instead of applying \textit{argmax}}.



\vspace{0.2cm}
\noindent\textbf{Label-agnostic \textit{FreeDARTS}.} As discussed in Sec. \ref{sec3.2}, our \textit{ZEROS} aims to calculate the connection sensitivity, a.k.a. how the information flow change with $\alpha$. Although \textit{ZEROS} in Eq.\eqref{eq:free_nas} seemingly depends on the labels of a dataset, \textit{ZEROS} can, however, be derived using random labels. In more detail, we can assign random labels to the inputs, and our \textit{ZEROS} score is also able to be obtained, and this label-agnostic training free differentiable architecture search is accordingly denoted as \textbf{FreeDARTS-L}. Since the aim of NAS is to observe the connection sensitivity and find excellent architectures, rather than those trained models, the labels are not compulsively needed. Several NAS methods have also developed a similar label-agnostic search \cite{shu2021nasi, liu2020labels}, which implies the reasonableness of the label-agnostic search. In Sec.\ref{sec4.2}, we also empirically verify its effectiveness.

\vspace{0.2cm}
\noindent\textbf{Data-agnostic \textit{FreeDARTS}.} Besides being label-agnostic, our \textit{ZEROS} can also be designed as data-agnostic. Motivated by SynFlow \cite{tanaka2020pruning}, we introduce a new loss function:
\begin{equation}\label{eq:free_nas_synflow}
\mathcal{L}_{\textup{SF}}=\mathbf{1}^{T}f(\alpha,\left | W \right |)\mathbf{1},
\end{equation}
where $\mathbf{1}$ is the all ones vector, $f$ is the supernet and $\left | W \right |$ denotes putting the element-wise absolute on all parameters. As shown in Eq.\eqref{eq:free_nas_synflow}, we replace the input data with all ones vector, and the loss is the sum of the model's output. Accordingly, our data-agnostic \textit{ZEROS} is defined as:
\begin{equation} \label{eq:free_nas_data}
\mathcal{F}({\alpha_k})=\left | \frac{\partial \mathcal{L}_{\textup{SF}}( W, \alpha)}{\partial \alpha_k}\cdot \alpha_k \right |,\\
\end{equation}
and the according training free differentiable architecture search is denoted as \textbf{FreeDARTS-D}. This perspective implies that our \textit{ZEROS} can monitor the information flow strengths for all candidate operations without knowing the dataset. In Sec.\ref{sec4.2}, we empirically find that the data-agnostic \textit{ZEROS} even achieves the best performance, compared with label-agnostic \textit{ZEROS} and the vanilla \textit{ZEROS} in our FreeDARTS. One potential reason is that, the \textit{ZEROS} score is obtained based on only several batches of datasets, and it is hard to catch the statistics of the whole dataset. In addition, our \textit{ZEROS} should be robust to the input data, and is able to find the important operation when the input change. In the zero-cost NAS \cite{abdelfattah2021zero} and network pruning at initialization \cite{tanaka2020pruning}, the data-agnostic SynFlow also achieves the best results comparing other saliency metrics, supporting the reasonableness of our label-agnostic \textit{ZEROS}.

\begin{corollary}\label{remark_grasynflow} 
\textit{\textbf{Label-agnostic and data-agnostic \textit{ZEROS} can also avoid parameter-intensive bias and achieve a similar convergence rate.}} When we assume the loss functions in label-agnostic and data-agnostic \textit{ZEROS} are also $\beta$-Lipschitz continuous and $\gamma$-Lipschitz smooth, we can obtain the simialr theoretical results as Theorem \ref{lemma:trace} and Theorem \ref{theorem:generalization}.
\end{corollary}


\section{Experiments}\label{sec4}
In the above, we designed a zero-cost operation sensitivity to measure the operation importance at initialization, and a simple framework \textbf{FreeDARTS} for the architecture search without any training is accordingly proposed. In this section, we conduct a series of experiments to verify the foundational question: \textit{{can we find high-quality architectures through our FreeDARTS}}? We consider five cases to analyze the proposed framework, including three benchmark datasets \cite{BENCH102,zela2020nasbench1shot1,siems2020bench}, DARTS search space \cite{liu2018darts}, and MobileNet search space \cite{howard2019searching}.


\subsection{Experimental Setting}

In our experiments, we consider two scenarios, NAS benchmark datasets, including NAS-Bench-101, NAS-Bench-1shot1, NAS-Bench-201, and NAS-Bench-301 \cite{ying2019bench,BENCH102,zela2020nasbench1shot1,siems2020bench}, and the common DARTS space \cite{liu2018darts} and the MobileNet space \cite{zhang2021neural} , to analyze the proposed framework FreeDARTS. The search spaces of NAS-Bench-101, NAS-Bench-1shot1, and NAS-Bench-201 are much smaller than the real-world DARTS and MobileNet space, while the ground-truth for all candidate architectures in the benchmark datasets is known. The NAS-Bench-301 shares the same search space with DARTS space, while the performance of candidate architectures are obtained by a predictor fitted with $\sim$60k architecture ground-truths. 

The search space in NAS-Bench-201 \cite{BENCH102} contains four nodes with five associated operations, resulting in 15,625 cell candidates, where the performance of CIFAR-100, CIFAR-100, and ImageNet for all architectures in this search space are reported. The NAS-Bench-101 \cite{ying2019bench} is another famous NAS benchmark dataset, which is much larger than NAS-Bench-201 while only the CIFAR-10 performance for all architectures are reported. More important, the architectures in NAS-Bench-101 contain different number of nodes, which makes it impossible to build a generalized supernet for one-shot nor differential NAS methods. To leverage the NAS-Bench-101 for analyzing the differentiable NAS methods, NAS-Bench-1Shot1 \cite{zela2020nasbench1shot1} builds from the NAS-Bench-101 benchmark dataset by dividing all architectures in NAS-Bench-101 into 3 different unified cell-based search spaces, which contain 6240, 29160, and 363648 architectures, respectively. The architectures in each search space have the same number of nodes and connections, making the differentiable NAS could be directly applied to each search space. We choose the third search space in NAS-Bench-1Shot1 to analyse FreeDARTS, since it is much more complicated than the remaining two search spaces. 

As to the most common search space in NAS, DARTS needs to search for two types of cells: a normal cell $\alpha_{normal}$ and a reduction cell $\alpha_{reduce}$. Cell structures are repeatedly stacked to form the final CNN structure. There are seven nodes in each cell: two input nodes, four operation nodes, and one output node. Each input node will select one operation from $|\mathcal{O}|=8$ candidate operations, including: $3\times 3$ max pooling and average pooling operation, $3\times 3$ and $5\times 5$ separable convolution operation, $3\times 3$ and $5\times 5$ dilated separable convolutions operation, identity, and $zero$. The common practice in DARTS is to search on CIFAR-10, and the best searched cell structures are directly transferred to CIFAR-100 and ImageNet. We conduct the architecture search with 5 different \textit{random seeds}, and the best one is selected after the evaluation on CIFAR-10. The best one is then transferred to CIFAR-100 and ImageNet. 

The NAS-Bench-301 \cite{siems2020bench} shares the same search space with DARTS, which contains about $10^{18}$ architectures, making it impossible to report the ground-truths for all architectures. Rather than training from the scratch to get the ground-truths for all architectures, NAS-Bench-301 fits a Graph Isomorphism Network based on the ground-truths of $\sim$60k architecture to predict the performance of all remaining architectures. The prediction usually could hardly indicate the the true performance in practice. For example, the performance of an architecture containing all parameter-free operations still receive competitive predictive performance, in contrast to the extremely poor true performance. However, the authors showed that the prediction shows a positive correlation with the ground truth, that the resulting search trajectories by the prediction closely resemble the ground truth trajectories when evaluating a differentiable NAS method.

We follows existing works \cite{zhang2021neural} to build the MobileNet architecture space. More specifically, different from the above search spaces, the MobileNet search only need to select the kernel size and expansion ratios, which uses the MobileNetV2-shape backbone. In this way, we need to select mobile inverted bottleneck convolution (MBConv) with different kernel size $\left \{ 3,5,7 \right \}$ and expansion ratios $\left \{ 3,6 \right \}$ for each block.

\begin{table}
\caption{Comparisons with NAS baselines on NAS-Bench-201. 
}
\scriptsize
\setlength{\tabcolsep}{3pt}
{\centering
\begin{tabular}
{lcccc}
\toprule

Method&CIFAR-10&CIFAR-100&ImageNet&Search Cost\\
\midrule
SETN&87.64$\pm$0.00&59.05$\pm$0.24&32.52$\pm$0.21&31010s\\
GDAS&93.40$\pm$0.49&70.33$\pm$0.87&41.47$\pm$0.21&28925.91s\\
DARTS (1st)&54.30$\pm$0.00&15.61$\pm$0.00&16.32$\pm$0.00&10889.87s\\
DARTS (2nd)&54.30$\pm$0.00&15.61$\pm$0.00&16.32$\pm$0.00&29901.67s\\
PC-DARTS &93.41$\pm$0.30&67.48$\pm$0.89&41.31$\pm$0.22&10023s\\
SNAS &92.77$\pm$0.83&69.34$\pm$1.98&43.16$\pm$2.64&32345s\\
\midrule
Random baseline&86.61$\pm$13.46&60.83$\pm$12.58&33.13$\pm$9.66&-\\
NASWOT &92.45$\pm$1.12&68.66$\pm$2.02&41.35$\pm$4.08&30.01s\\
Zero-Cost NAS &93.45$\pm$0.28&70.73$\pm$1.36&43.64$\pm$2.42&115.2s\\
Zero-Cost-PT $\ddagger$ &93.75$\pm$0.00&71.11$\pm$0.00&41.43$\pm$0.00&647.5s\\
TE-NAS &\textbf{93.90$\pm$0.47}&71.24$\pm$0.56&42.38$\pm$0.46&1558s\\
NASI&93.55$\pm$0.10&71.20$\pm$0.14&44.84$\pm$1.41&120s\\
GradSign&93.31$\pm$0.47&{70.33$\pm$1.28}&42.42$\pm$2.81&-\\
KNAS &93.43&71.05&45.05&20000s\\
\textbf{FreeDARTS} &93.64$\pm$0.24&\textbf{71.30$\pm$1.01}&\textbf{45.62$\pm$0.39}&\textbf{3.6s}\\

\midrule

\textbf{optimal}&94.37&73.51&47.31&-\\
\bottomrule
\end{tabular}
\par
}
We consider the iterative paradigm with the data-agnostic \textit{ZEROS} for our FreeDARTS in this experiment. Our best single run achieves \textbf{93.91\%}, \textbf{72.40\%}, and \textbf{46.16\%} test accuracy on three datasets, respectively. ``$\ddagger$" indicates Zero-Cost-PT considers the same SynFlow metric as our FreeDARTS.
\label{tab:nasbench201}
\end{table}

\subsection{Experiments on Benchmark Datasets}
\subsubsection{Experimental results on NAS-Bench-201}
\label{sec4.1}
\noindent\textbf{Reproducible comparison with existing works on NAS-Bench-201.} The results for FreeDARTS and the weight-sharing NAS baselines on the NAS-Bench-201 are provided in Table \ref{tab:nasbench201}. FreeDARTS produced competitive results on all three datasets, significantly outperforming the DARTS and other elaborately designed methods. Moreover, the best single-run of FreeDARTS achieves a performance of \textbf{93.91\%} on CIFAR-10, \textbf{72.40\%} on CIFAR-100, and \textbf{46.16\%} on ImageNet, which are very close to the optimal test accuracies in the NAS-Bench-201 dataset. The second block in Table \ref{tab:nasbench201} contains the comparison results of FreeDARTS with existing train-free NAS methods. The Random baseline is to randomly generate architectures without training. NASWOT \cite{mellor2020neural} uses the Jacobian to score architectures, and TE-NAS \cite{mellor2020neural} uses the spectrum of NTKs and the number of linear regions to rank the architectures. Zero-cost NAS \cite{abdelfattah2021zero}, Zero-cost-PT NAS \cite{xiang2021zero}, and our FreeDARTS all consider Synflow metric in this experiment, where our FreeDARTS obtains competitive results, especially on the large ImageNet dataset. As shown, FreeDARTS outperforms the random baseline with a large margin, showing the effectiveness of FreeDARTS. Moreover, compared with the two elaborately designed train-free NAS, NASWOT and TE-NAS, FreeDARTS also achieves more competitive results, further showing the \textbf{reliability} of the proposed operation saliency metric. Table \ref{tab:nasbench201} also summarizes the search cost of several weight-sharing NAS baselines and training-free NAS methods. As shown, a significant advantage of our FreeDARTS is the \textbf{efficiency}, which only costs \textbf{3.6s} to find competitive architectures. 

\begin{table}[t]
\scriptsize
\centering
\caption{Comparison on different scoring paradigms.}
\begin{tabular}
{lccccccc}
\toprule

{Method}&\multicolumn{1}{c}{CIFAR-10}&\multicolumn{1}{c}{CIFAR-100}&\multicolumn{1}{c}{ImageNet}\\
\midrule
One-shot FreeDARTS &92.53$\pm$0.84&69.27$\pm$1.78&43.31$\pm$3.39\\
FreeDARTS&93.04$\pm$0.54&70.10$\pm$1.48&43.55$\pm$1.56\\
One-shot FreeDARTS-L&92.46$\pm$0.75&68.47$\pm$1.77&42.57$\pm$2.19\\
FreeDARTS-L&92.86$\pm$0.60&69.67$\pm$0.87&43.17$\pm$1.34\\
One-shot FreeDARTS-D &93.63$\pm$0.06&70.04$\pm$0.60&44.77$\pm$1.81\\
FreeDARTS-D &\textbf{93.64$\pm$0.24}&\textbf{71.30$\pm$1.01}&\textbf{45.62$\pm$0.39}\\
\bottomrule
\end{tabular}
\label{tab:paradigm}
\end{table}

\begin{figure*}[t]
\centering
 \subfigure[CIFAR-10]{
  \begin{minipage}{4.2cm}
      \includegraphics[width=5.2cm,height=3.8cm]{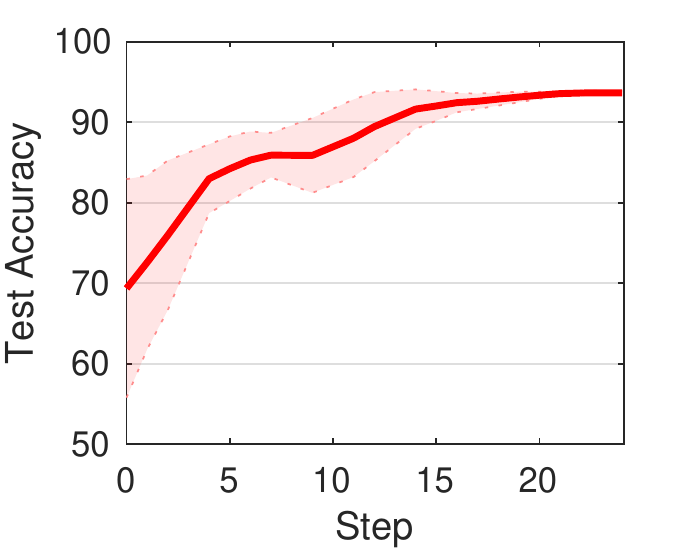}
  \end{minipage}
 }
  \subfigure[CIFAR-100]{
  \begin{minipage}{5.2cm}
       \includegraphics[width=5.2cm,height=3.8cm]{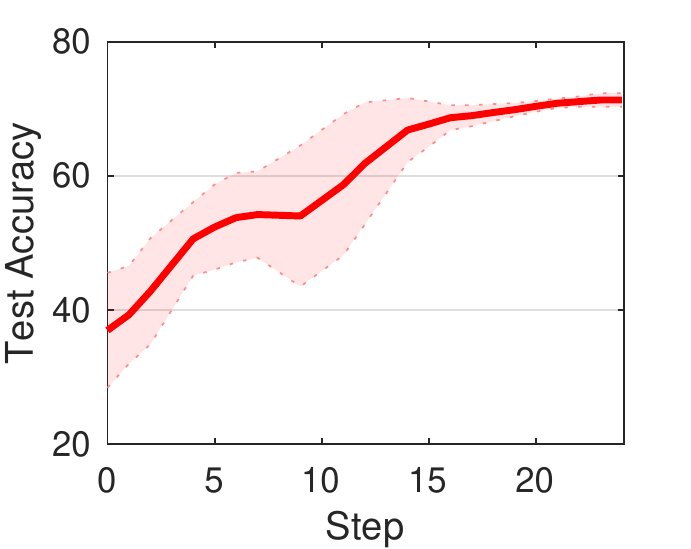}
  \end{minipage} 
  }
  \subfigure[ImageNet]{
  \begin{minipage}{5.2cm}
       \includegraphics[width=5.2cm,height=3.8cm]{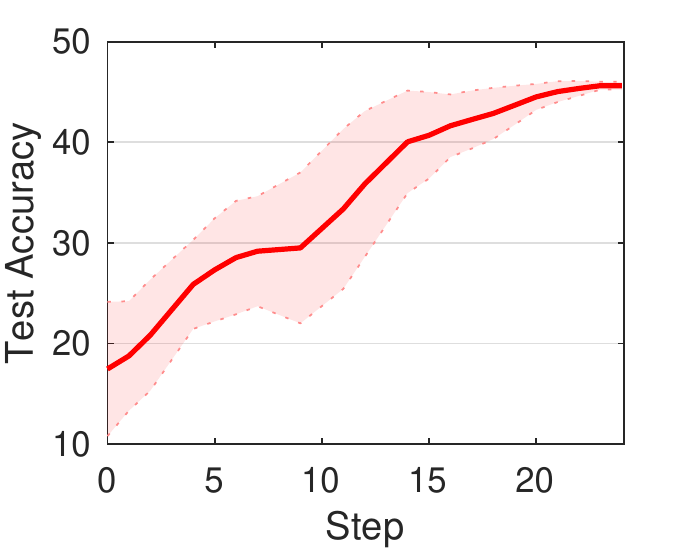}
  \end{minipage} 
  }  
 \caption{Track of test accuracy during the pruning for FreeDARTS on the NAS-Bench-201.}
 \label{fig:Track_nas201}
\end{figure*}

\vspace{0.2cm}
\noindent\textbf{Comparison among different scoring paradigms.} In Sec.\ref{sec3.4}, we discussed different implementations of our FreeDARTS, including one-shot FreeDARTS, label-agnostic FreeDARTS (FreeDARTS-L), and data-agnostic FreeDARTS (FreeDARTS-L), where one-shot FreeDARTS prune the supernet by one-shot rather than iterative, FreeDARTS-L assigns random labels to input data and FreeDARTS-D replaces the input data with all-ones vector. Table \ref{tab:paradigm} summarizes their performance comparison. Interestingly, we can find that FreeDARTS-L and FreeDARTS-D achieve similar or even better results than FreeDARTS, which confirms the label-agnostic and data-agnostic search achieved by FreeDARTS. Compared with FreeDARTS which considers an iterative way to prune supernet, one-shot FreeDARTS obtains slightly worse results, which therefore also further validates the effectiveness of the iterative pruning by FreeDARTS. Since our FreeDARTS-D is the most efficient and achieves competitive performance, we default consider the data-agnostic paradigm for our FreeDARTS in the remaining experiments.

\vspace{0.2cm}
\noindent\textbf{Search performance during the pruning.} To investigate whether our FreeDARTS is effective in removing inferior operations, we consider only pruning one inferior operation from the supernet in each step based on the proposed \textit{zero-cost operation sensitivity}. Figure \ref{fig:Track_nas201} tracks the quality of the pruned supernet, by applying the \textit{argmax} on the architecture parameter $\alpha$ of the pruned supernet after each step of operation pruning to get the architecture (the same as randomly selecting an operation from remaining operations as $\alpha$ is randomly generated). With the pruning proceeding, the performance increases, verifying that the FreeDARTS can effectively remove inferior operations.
\begin{table}
\footnotesize
\setlength{\tabcolsep}{3pt}
\caption{Search results of CIFAR-10 on NAS-Bench-1shot1.}
\setlength{\tabcolsep}{2.5pt}
\centering
\begin{tabular}{lccccc}
\toprule
\makecell[l]{\multirow{2}*{Method}}&\multicolumn{2}{c}{Average (\%)}&\multicolumn{2}{c}{Best (\%)}&{Search}\\
~&Valid&Test&Valid&{Test}&{Cost}\\
\midrule
GDAS&6.8$\pm$0.1&6.1$\pm$0.2&6.7&5.9&11425s\\
PC-DARTS&6.7$\pm$0.1&6.2$\pm$0.2&6.6&5.9&14760s\\
DARTS (1st)&6.8$\pm$0.05&6.1$\pm$0.2&6.6&5.9&8280s\\
DARTS (2nd)&6.8$\pm$0.05&6.2$\pm$0.05&6.6&6.2&19800s\\
\midrule
Random&24.4$\pm$32.8&24.1$\pm$33.2&7.8&7.5&N/A\\
Zero-Cost-PT& 8.05$\pm$1.24&7.45$\pm$0.94&6.8&6.5&718s\\
FreeDARTS&7.8$\pm$2.4&7.3$\pm$2.4&\textbf{6.0}&\textbf{5.3}&\textbf{6.3s}\\
\bottomrule
\end{tabular}
\label{tab:results_nas101}
\end{table}


\vspace{0.2cm}
\noindent\noindent\textbf{Hyperparameter study.} As described in Algorithm \ref{alg:algorithm1}, FreeDARTS is train-free, making our method simple and concise to implement without tuning too many hyperparameters. As described in Eq.\eqref{eq:free_nas}, the operation saliency score is the product of the value of $\alpha$ and the gradient of $\alpha$, while $\alpha$ is transformed by softmax before conducting the forward to calculate the gradient. Generally, the $\alpha$ is initialized with $a*randn$, where $a$ is a weighted scale which is the only hyperparameter in our FreeDARTS denoting the trade-off between $\frac{\partial \mathcal{R}_\textup{FF}}{\partial \alpha}$ (or $\frac{\partial \mathcal{L}}{\partial \alpha}$) and $ \alpha$. Figure \ref{fig:H_nas201} analyzes the hyperparameter $a$ with summarizing the performance of FreeDARTS with different $a$ on the NAS-Bench-201. In general, our FreeDARTS is robust to this hyperparameter, which with different $a$ in a lager range ($1e^{-5} \sim 1e^{-2}$) all achieve competitive results.

\subsubsection{Experimental results on NAS-Bench-1shot1.} We also conduct the experiments on the NAS-Bench-1shot1 space, where the comparison results for FreeDARTS and the weight-sharing NAS baselines are provided in Table \ref{tab:results_nas101}. We report not only the average results but also the best results after several independent runs with different random seeds. As show, our FreeDARTS outperforms the Random baseline by large margins, showing the effectiveness of the proposed framework. As verified before, the most attractive advantage of our FreeDARTS is the efficiency, and it also completes the architecture search NAS-Bench-1shot1 space within much less time compared with the common differentiable NAS baselines, with only \textbf{6.3s}. Although these differentiable NAS baselines achieve better results according to the average test error, our FreeDARTS could find more competitive architectures based on the best test error. 

\subsubsection{Experimental Results on NAS-Bench-301} 

We also conduct the experiments on the NAS-Bench-301, the largest existing benchmark dataset, where the performance of our FreeDARTS with different operation saliency and several differentiable NAS baselines are provided in Table \ref{tab:results_nas301}. The ``Average" reports the average predictive performance from the benchmark after several independent runs, and the ``Best" reports the predictive performance of the best searched architectures. The ``Ground-Truth" is the validation results based on the train-from-the-scratch. As show, our FreeDARTS also outperforms the Random baseline by large margins in the NAS-Bench-301, again showing the effectiveness of the proposed framework. As we can observed, the predictive results in NAS-Bench-301 are not very consistent with the ground truth in Table \ref{tab:results_nas301}. For example, although GDAS achieves much lower ground truth performance, it obtains the best predictive performance from the NAS-Bench-301. Although the predictive performance from NAS-Bench-301 could not exactly indicate the true performance, the results still present several consistent results with our previous observation in other benchmark datasets. For example, our FreeDARTS or one-shot FreeDARTS both outperform ``Random" baseline by large margins, verifying the effectiveness of the proposed method. Compared with several training based NAS baselines, our FreeDARTS without any training also achieves comparable results.

\begin{table}[t]
\centering
\caption{Statistic search results (test error) on NAS-Bench-301.}
\begin{tabular}{lccc}
\toprule
{Method}&{Average}&{Best}&{Ground-True}\\
\midrule
GDAS\cite{GDAS}&6.52$\pm$0.62 (\%)&5.38\%&3.07$\pm$0.16 (\%)\\
PC-DARTS\cite{xu2019pcdarts}&6.42$\pm$0.43 (\%)&5.46\%&2.57$\pm$0.07 (\%)\\
DARTS (2nd) \cite{liu2018darts}&6.74$\pm$0.58 (\%)&5.87\%&2.76$\pm$0.09 (\%)\\
\midrule
Random&7.11$\pm$0.58 (\%)&6.21\%&3.29$\pm$0.15 (\%)\\
One-shot FreeDARTS&6.60$\pm$0.47 (\%)&5.71\%&2.69$\pm$0.08 (\%)\\
FreeDARTS&6.65$\pm$0.52 (\%)&5.50\%&2.50$\pm$0.05 (\%)\\
\bottomrule
\end{tabular}
\label{tab:results_nas301}
\end{table}


\begin{figure*}
\centering
 \subfigure[CIFAR-10]{
  \begin{minipage}{5.2cm}
      \includegraphics[width=5.2cm,height=3.8cm]{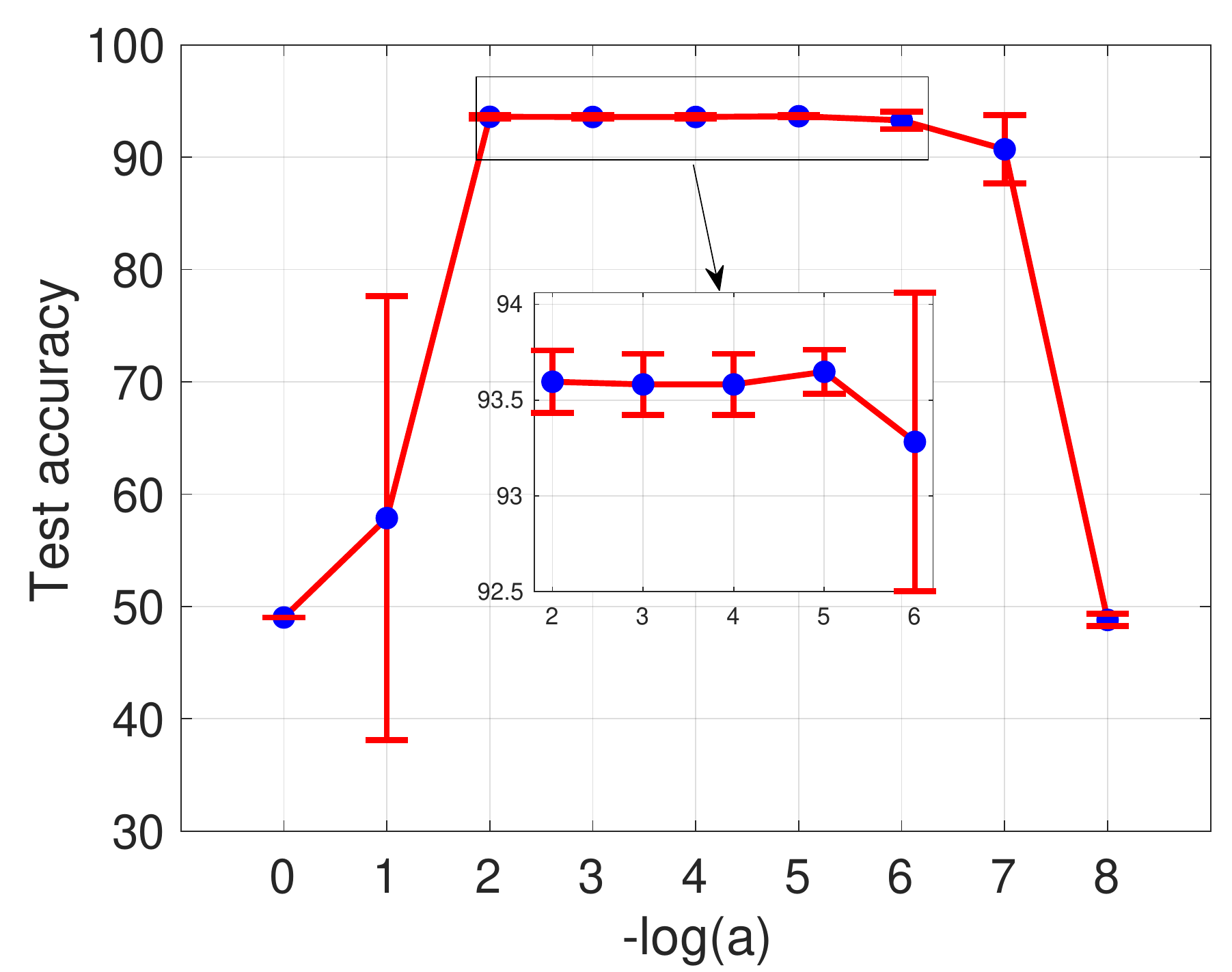}
  \end{minipage}
 }
  \subfigure[CIFAR-100]{
  \begin{minipage}{5.2cm}
       \includegraphics[width=5.2cm,height=3.8cm]{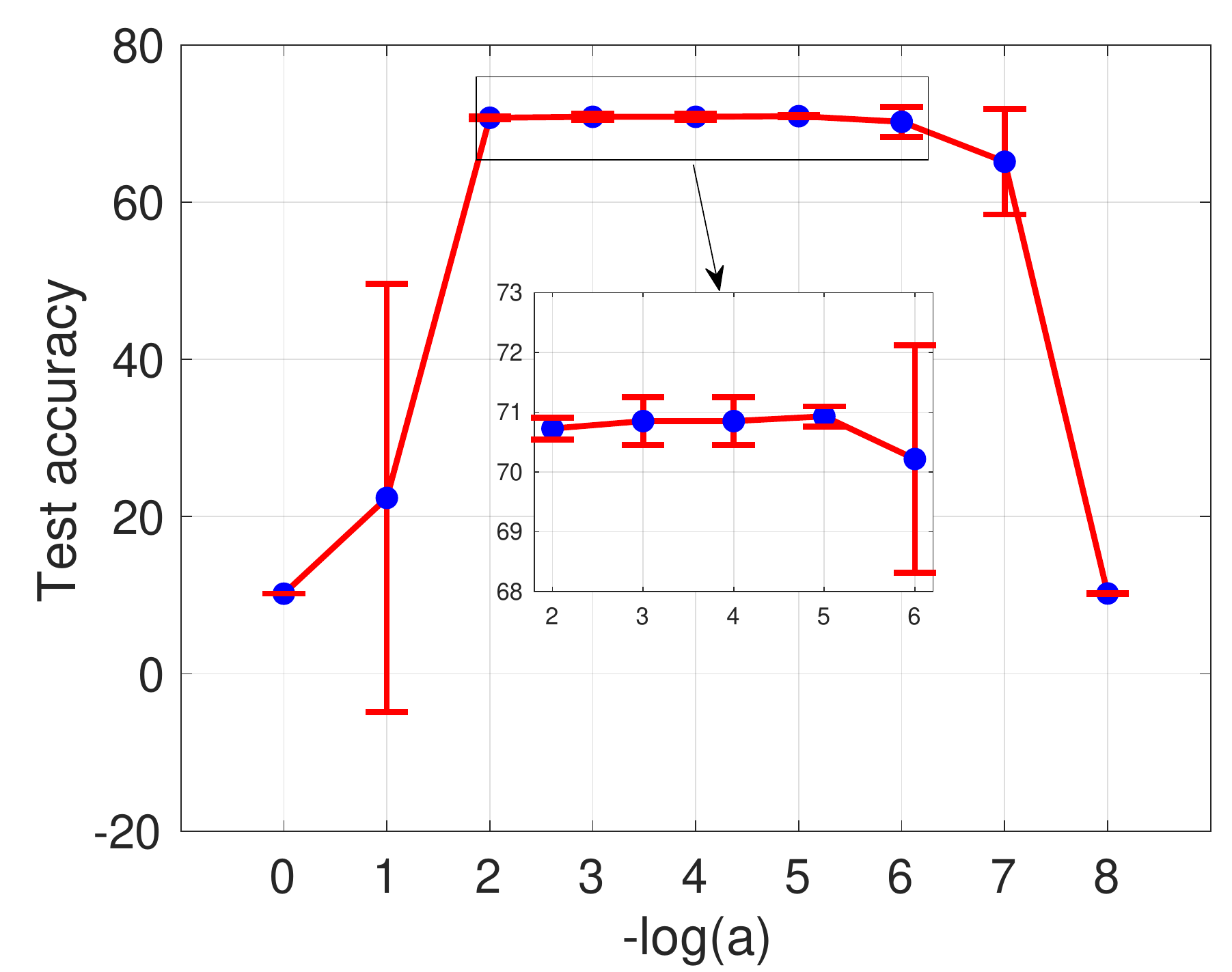}
  \end{minipage} 
  }
  \subfigure[ImageNet]{
  \begin{minipage}{5.2cm}
       \includegraphics[width=5.2cm,height=3.8cm]{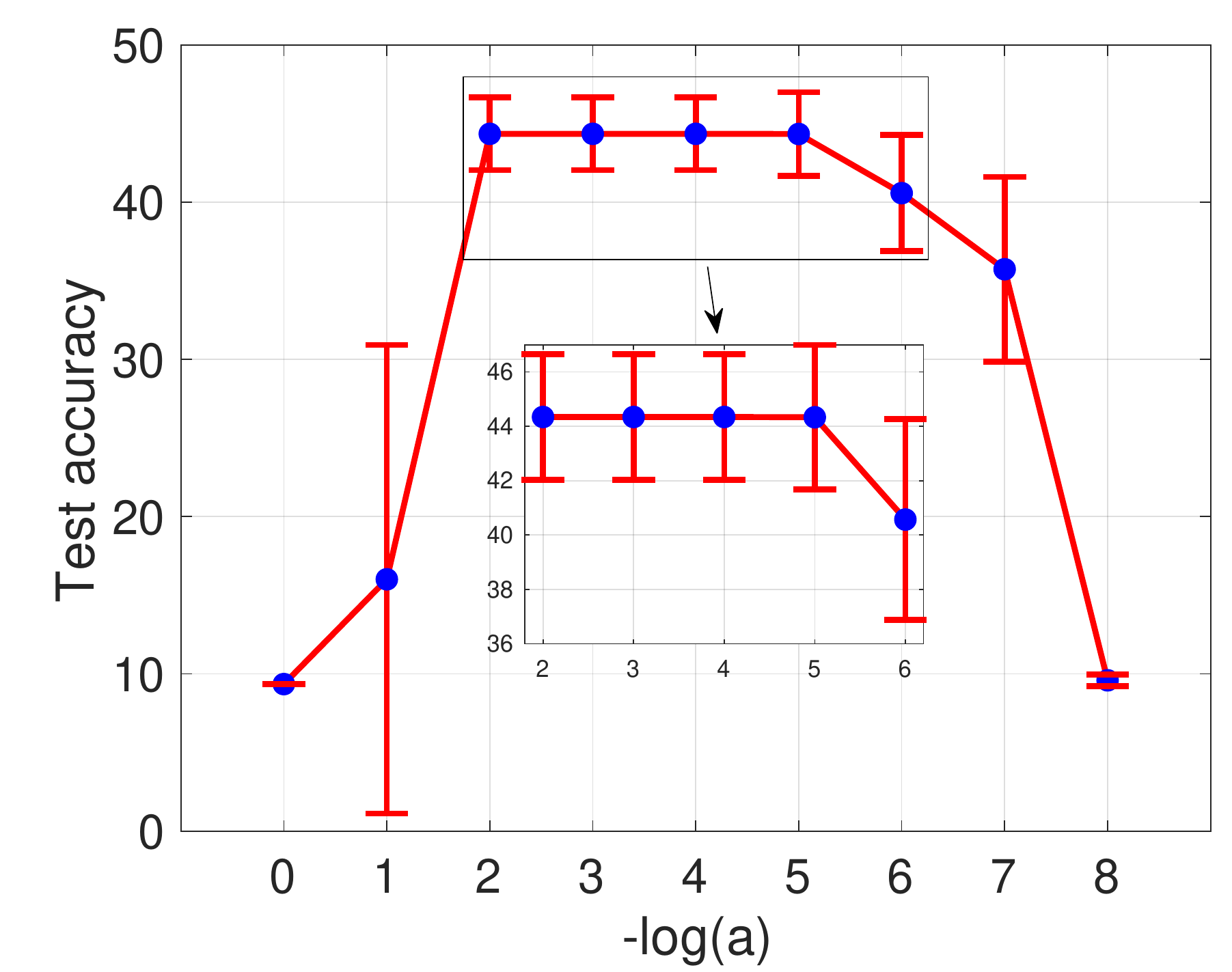}
  \end{minipage} 
  }  

 \caption{Hyperparameter analysis of FreeDARTS on the NAS-Bench-201 benchmark dataset. }
 \label{fig:H_nas201}
 \vspace{-2mm}
\end{figure*}

\begin{table*}[t]
\caption{Comparison results with state-of-the-art NAS approaches on DARTS search space.}
\setlength{\tabcolsep}{3pt}
{\centering

\begin{tabular}{lccccccccc}
\toprule
\multirow{2}*{Method}&\multicolumn{3}{c}{Test Error (\%)}&\multicolumn{1}{c}{Param}&\multicolumn{1}{c}{$+ \times$}&\multicolumn{1}{c}{Train}&{Data}&{Search}\\
~&\multicolumn{1}{c}{CIFAR-10}&\multicolumn{1}{c}{CIFAR-100}&\multicolumn{1}{c}{ImageNet}&{(M)}&{(M)}&{Free}&{Agnostic}&{Cost}\\
\midrule
RandomNAS \cite{li2019random}&2.85$\pm$0.08&17.63 &27.1&4.3&613&$\times$ &$\times$&0.4d\\
SNAS \cite{xie2018snas}&2.85$\pm$0.02&20.09 &27.3 / 9.2&2.8&474&$\times$ &$\times$ &1.5d\\
BayesNAS \cite{zhou2019bayesnas}&2.81$\pm$0.04&-&26.5 / 8.9&3.4&-&$\times$&$\times$&0.2d\\
GDAS \cite{GDAS}&2.93&18.38&26.0 /  8.5&3.4&545&$\times$&$\times$&0.2d\\
PDARTS \cite{chen2019progressive}&2.50&16.63& 24.4 / 7.4&3.4&557&$\times$&$\times$&0.3d\\
PC-DARTS \cite{xu2019pcdarts}&2.57$\pm$0.07&17.11&25.1 / 7.8&3.6&586&$\times$&$\times$&0.3d\\
DrNAS \cite{chen2020drnas}&2.54$\pm$0.03&16.30&24.2 / 7.3&4.0&644&$\times$&$\times$&0.4d\\
DARTS \cite{liu2018darts}&2.76$\pm$0.09&17.54&26.9 / 8.7&3.4&574&$\times$&$\times$&4d\\
\midrule
TE-NAS \cite{chen2021neural}&2.63 & 17.83&26.2 / 8.3&3.8&610&\checkmark&$\times$&0.17d\\
Zero-Cost-PT \cite{xiang2021zero}&2.68$\pm$0.17& 17.53&24.4 / 7.5&4.7&817&\checkmark&\checkmark&0.018d\\
FreeDARTS&2.78$\pm$0.06 & 18.03&26.1 / 8.2&3.6&634&\checkmark&\checkmark&\textbf{8.5s}\\
FreeDARTS$\dagger$&\textbf{2.50$\pm$0.05} & 17.08&25.4 / 7.8&3.6&577&\checkmark&\checkmark&\textbf{8.5s}\\
FreeDARTS$\ddagger$&2.67$\pm$0.04& \textbf{16.35}&\textbf{24.4 / 7.3}&4.1&655&\checkmark&\checkmark&\textbf{8.5s}\\
\bottomrule
\end{tabular}
\par
}
``Param" is the model size on CIFAR-10, while ``$+\times$" is calculated on ImageNet dataset. ``d" is the GPU days and ``s" is the GPU seconds. We consider Synflow in this space. The best single run of FreeDARTS$\dagger$ in CIFAR-10 is 2.45\%.
\label{tab:results_CIFAR}
\end{table*}

\begin{table}[t]
\footnotesize
\caption{Statistic search results on MobileNet space.}
\centering
\begin{tabular}{lcccc}
\toprule
\multirow{2}*{Method}&{Paras}&{Top-1}&{Search Cost}\\
~&(M)&(\%)&{(GPU days)}\\
\midrule
NASNet-A\cite{zoph2018learning}&5.3&74.0&1800\\
MnasNet-A2\cite{tan2019mnasnet}&4.8&75.6&4000\\
MobileNetV2\cite{sandler2018mobilenetv2}&3.4&72.0&-\\
MobileNetV3\cite{howard2019searching}&5.4&75.2&3000\\
Efficienet-B0\cite{tan2019efficientnet}&5.3&76.3&3000\\
Single-Path NAS\cite{stamoulis2019single}&4.3&75.0&1.25\\
BigNAS-S\cite{yu2020bignas}&4.5&76.5&48\\
Proxyless-R\cite{cai2018proxylessnas}&4.1&74.6&8.3\\
Proxyless-GPU\cite{cai2018proxylessnas}&7.1&75.1&8.3\\
FairNAS-A\cite{chu2019fairnas}&4.6&75.3&12\\
FairNAS-B\cite{chu2019fairnas}&4.5&75.1&12\\
FairNAS-C\cite{chu2019fairnas}&4.4&74.7&12\\
Zero-Cost NAS\cite{abdelfattah2021zero}&8.4&76.1&-\\
Zero-Cost-PT\cite{xiang2021zero}&8.1&76.2&0.041\\
\midrule
FreeDARTS&6.2&76.4&12.7s$^*$\\
\bottomrule
\end{tabular}
\flushleft{$^*$FreeDARTS finish searching in seconds while others in days.}
\label{tab:results_mobilenet}
\end{table}

\begin{figure*}[t]
\centering
 \subfigure[Normal cells by FreeDARTS with seed4,  FreeDARTS$\dagger$ with seed1, and FreeDARTS$\ddagger$ with seed3.]{
  \begin{minipage}{4.5cm}
      \includegraphics[width=4.5cm,height=3cm]{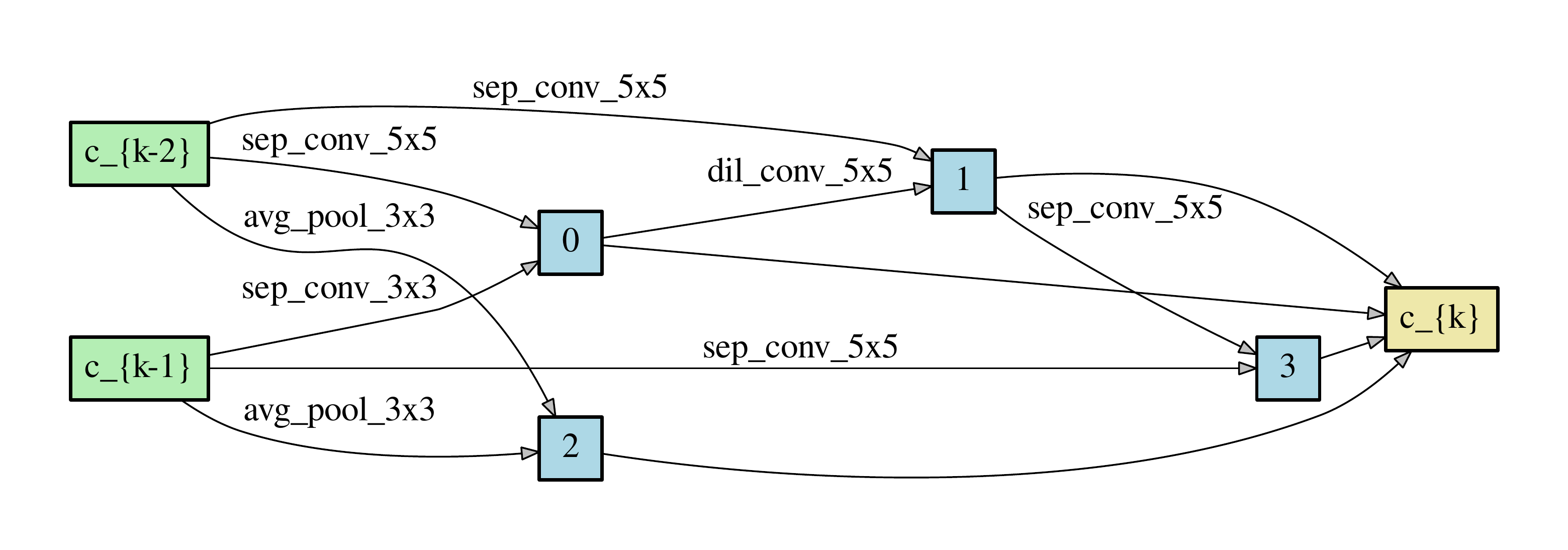}
  \end{minipage}
  \begin{minipage}{4.5cm}
      \includegraphics[width=4.5cm,height=3cm]{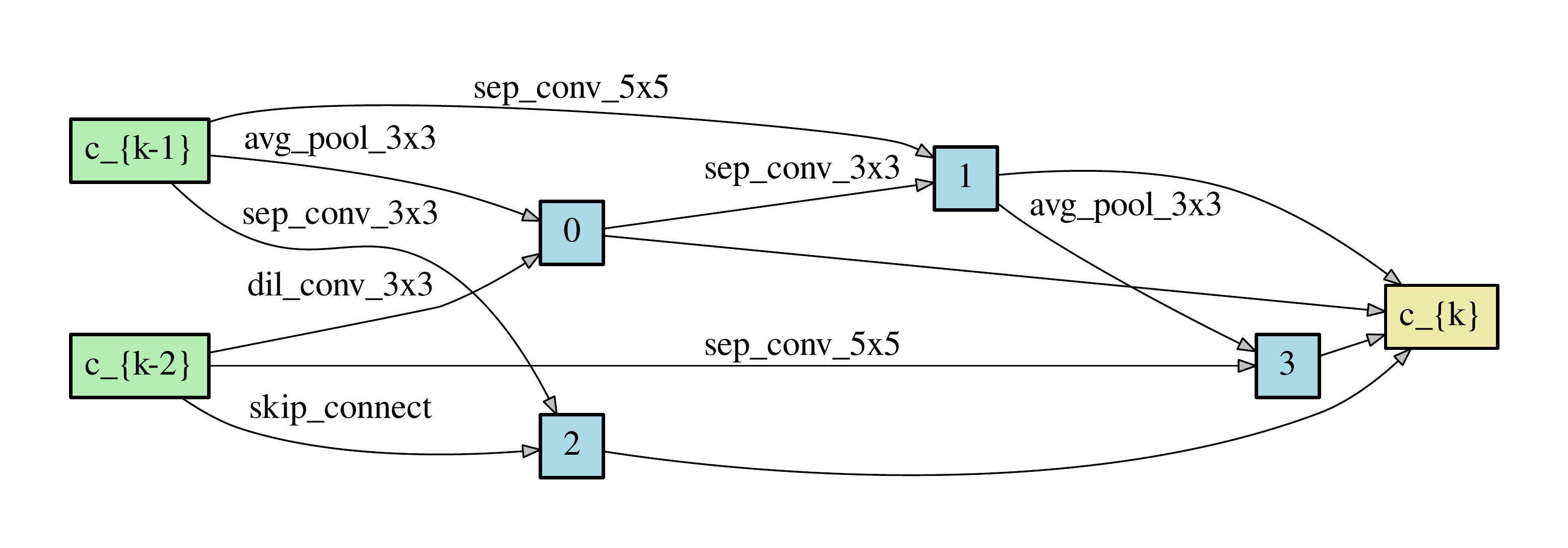}
  \end{minipage}
  \begin{minipage}{4.5cm}
      \includegraphics[width=4.5cm,height=3cm]{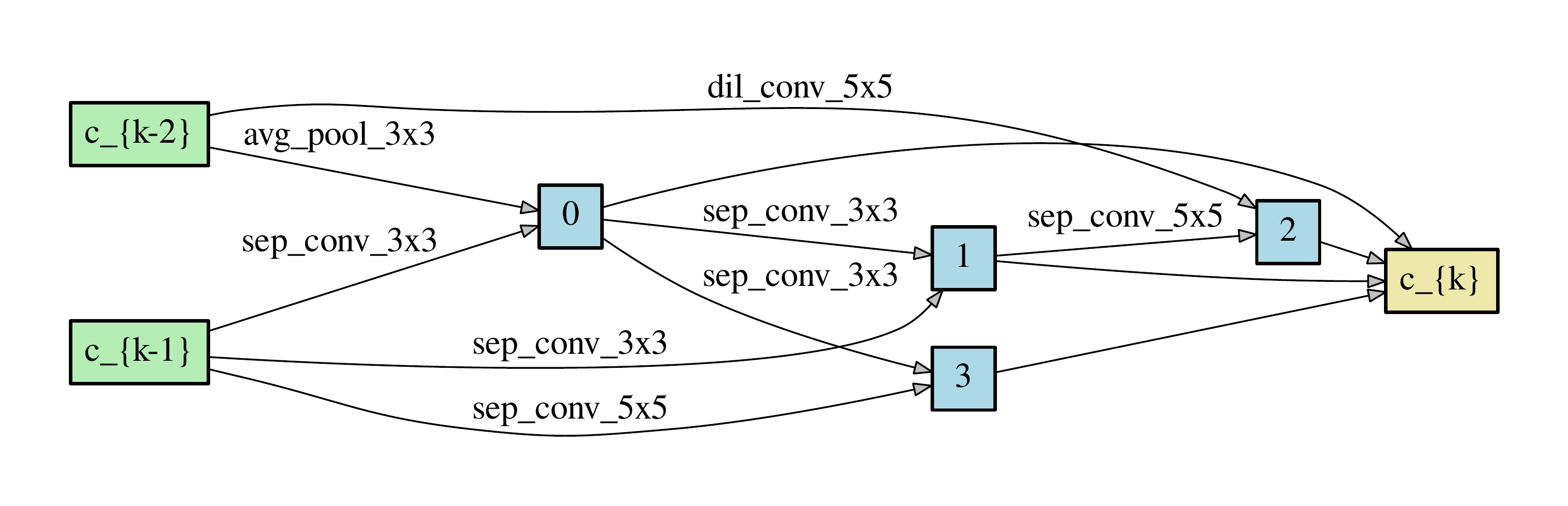}
  \end{minipage}  
  }
  
   \subfigure[Reduction cells by FreeDARTS with seed4,  FreeDARTS$\dagger$ with seed1, and FreeDARTS$\ddagger$ with seed3.]{
    \begin{minipage}{4.5cm}
       \includegraphics[width=4.5cm,height=3cm]{figures/reduction_SYNFLOW_oneshot_seed2}
  \end{minipage} 
  \begin{minipage}{4.5cm}
       \includegraphics[width=4.5cm,height=3cm]{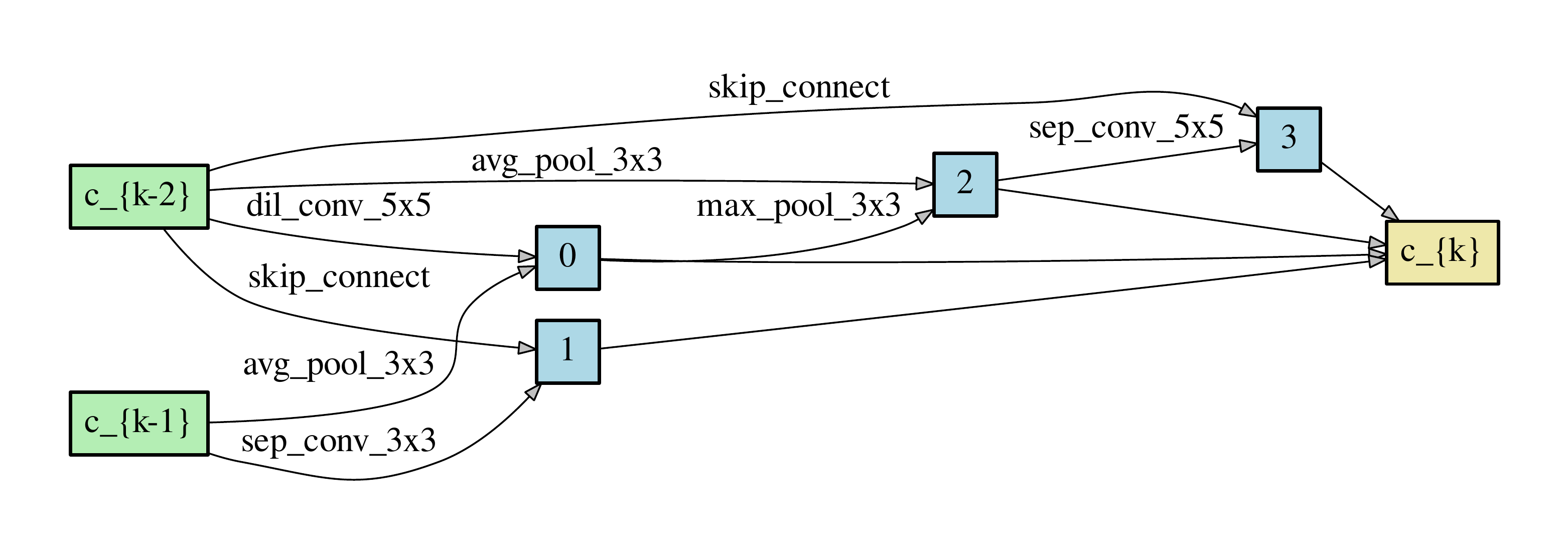}
  \end{minipage} 
  \begin{minipage}{4.5cm}
       \includegraphics[width=4.5cm,height=3cm]{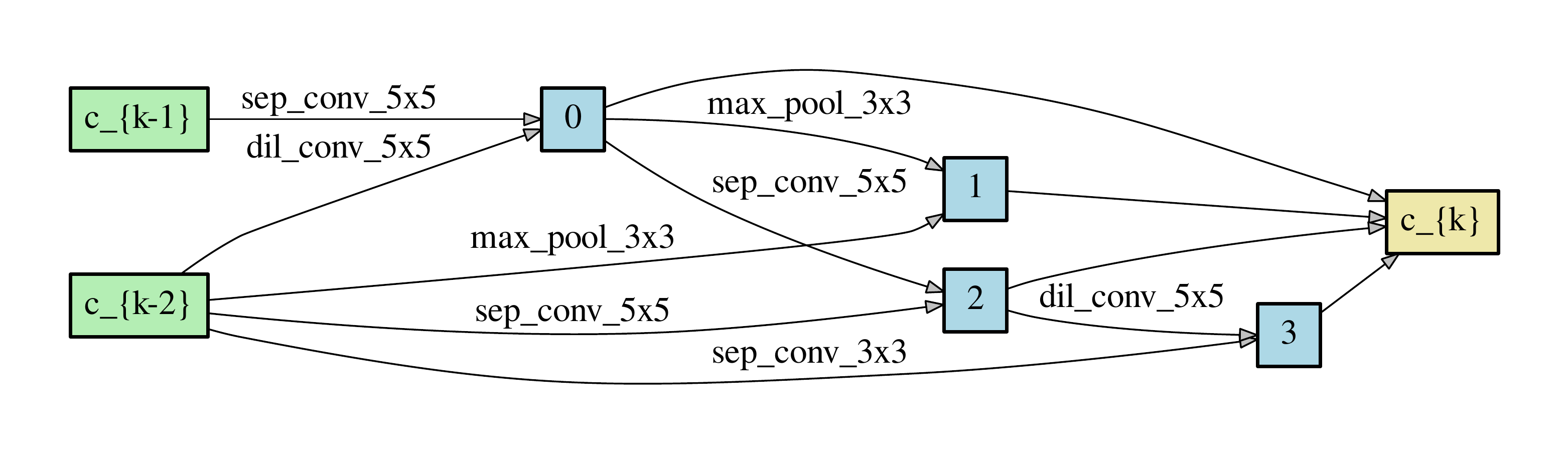}
  \end{minipage}   
  }
 \caption{The best cells discovered by FreeDARTS,  FreeDARTS$\dagger$, and  FreeDARTS$\ddagger$(from top to bottom) on the DARTS search space. }
 \label{fig:archs_darts}
\end{figure*}

\subsection{Experiments on DARTS Search Space}
\label{sec4.2}

Due to the high memory consumption, the vanilla DARTS and its variants  \cite{liu2018darts,xu2019pcdarts} only searches on a shallow supernet with the CIFAR-10 dataset, and the best-found cell is repeatedly stacked to form a deeper architecture for evaluation on CIFAR-10, CIFAR-100, and ImageNet datasets. In this way, there are \textit{depth gap} and \textit{dataset gap} between search and evaluation. Thanks to the memory and computation efficiency of our FreeDARTS, we are able to directly search with a larger backbone supernet on a larger ImageNet dataset, showing the flexibility of FreeDARTS. In this experiment, our FreeDARTS is the same as DARTS \cite{liu2018darts} during the search, while FreeDARTS$\dagger$ stacks 20 cells to form a larger backbone to eliminate \textit{depth gap}, and FreeDARTS$\ddagger$ directly searches on the ImageNet to eliminate \textit{dataset gap}. The comparison results on the DARTS space are presented in Table \ref{tab:results_CIFAR}. As shown, the best architecture searched by our FreeDARTS achieves a 2.78$\pm$0.06 \% test error on CIFAR-10, which is on par with the DARTS while only costing \textbf{8.5} seconds. More interesting, after eliminating the \textit{depth gap} by our FreeDARTS$\dagger$, we achieve much better results with a 2.45 \% test error on CIFAR-10, which outperforms baselines by large margins, again demonstrating the effectiveness of the proposed method. When directly searching on the ImageNet to further eliminate \textit{dataset gap}, FreeDARTS$\ddagger$ obtains competitive 16.35\% test error on the CIFAR-100, and 24.8 / 7.5 \% top1 / top5 test error on the ImageNet. 



\begin{figure*}
\centering
  \begin{minipage}{14cm}
      \includegraphics[width=14cm,height=2cm]{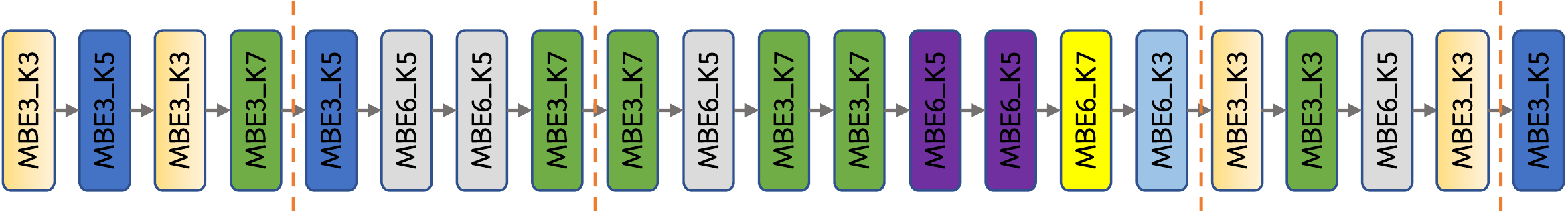}
  \end{minipage}
 \caption{The searched architecture by FreeDARTS searched with \textit{seed0} on MobileNet search space. }
 \label{fig:search_mobile}
\end{figure*}

\subsection{Experiments on MobileNet Search Space}
\label{sec4.3}

Apart from DARTS and several NAS benchmark search spaces, we further conduct the experiments on the MobileNet search space \cite{howard2019searching}, which contains 21 blocks and each block is with 7 candidate options: $3\times 2=6$ inverted bottleneck residual (MBConv) blocks, with 3 different kernel sizes $\left \{ 3, 5, 7 \right \}$ and 2 expansion ratio $\left \{ 3, 6 \right \}$, or the \textit{skip-connection}. Similar to FreeDARTS$\ddagger$, our search is conducted on the ImageNet dataset, and the experimental setting for evaluation phase as ProxylessNAS \cite{howard2019searching} with only changing batch size to 512 on 2 NVIDIA V100 GPUs and initial learning rate to 0.25. Table \ref{tab:results_mobilenet} compares FreeDARTS with baselines, whose results can be also found in \cite{chu2019fairnas}. We can see that, FreeDARTS can marginally outperform other baselines with only several GPU seconds. These results again verify the versatility of our FreeDARTS which can easily obtain competitive results in MobileNet space. More interesting, zero-cost NAS and zero-cost-PT prefer larger models, since they both select the most parameter-intensive operation for each block. Differently, our FreeDARTS can find architectures with moderate model size. 

\subsection{Discussion on Avoiding Parameter-intensive bias}
\label{sect_avoid}
As discussed in Sec.\ref{sec3.3}, our FreeDARTS can avoid parameter-intensive bias, which is a common issue in existing training-free NAS methods. For better visualization, zero-cost-PT \cite{xiang2021zero} and zero-cost NAS \cite{abdelfattah2021zero}, which are representative ``summing-up" based training-free NAS methods, in Figure \ref{fig:zero-pt} and Figure \ref{fig:zerocost_archs}. We can clearly see that these architectures are with very large model sizes. For example, the architecture search by zero-cost-PT on MobileNet space selects 18 ``MBE6-K7" out of 21 blocks, where ``MBE6-K7" contains the most parameters among all candidate operations. As to the DARTS space, we can see that zero-cost-PT almost selects ``sep\_conv\_5$\times$5" for every edge, while which are also those most parameter-intensive operations. Figure \ref{fig:zerocost_archs} visualizes the architectures searched by zero-cost NAS on NAS-Bench-201 and DARTS space. Similar to zero-cost-PT, we can see that the architecture selects the most parameter-intensive operation ``nor\_conv\_3$\times$3" for every edge on NAS-Bench-201, and the most parameter-intensive operation ``sep\_conv\_5$\times$5" on DARTS space. These results verified the parameter-intensive bias in existing zero-cost NAS methods. In contrast, those architectures searched by our FreeDARTS in Fig.\ref{fig:archs_darts} and Fig.\ref{fig:search_mobile} are more diverse, which contain different types of operations rather than only selecting those architectures with the most parameters. This phenomenon empirically confirms that our FreeDARTS has the ability to avoid parameter-intensive bias.

\begin{figure*}[t]
\centering
 \subfigure[Normal and reduction cell using CIFAR10.]{
\includegraphics[width=8cm,height=2.5cm]{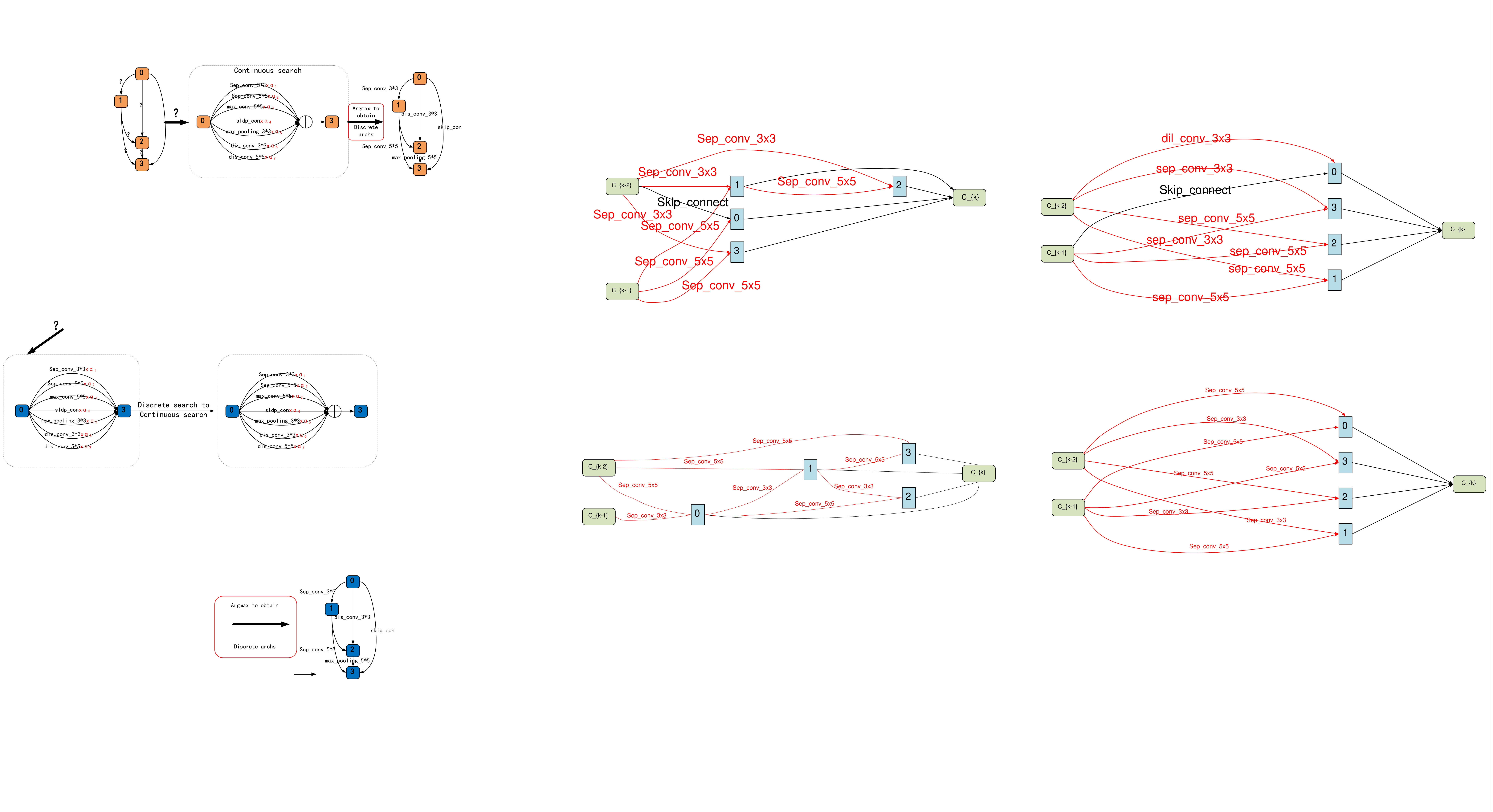}
}
 \subfigure[Normal and reduction cell using ImageNet.]{
\includegraphics[width=8cm,height=2.5cm]{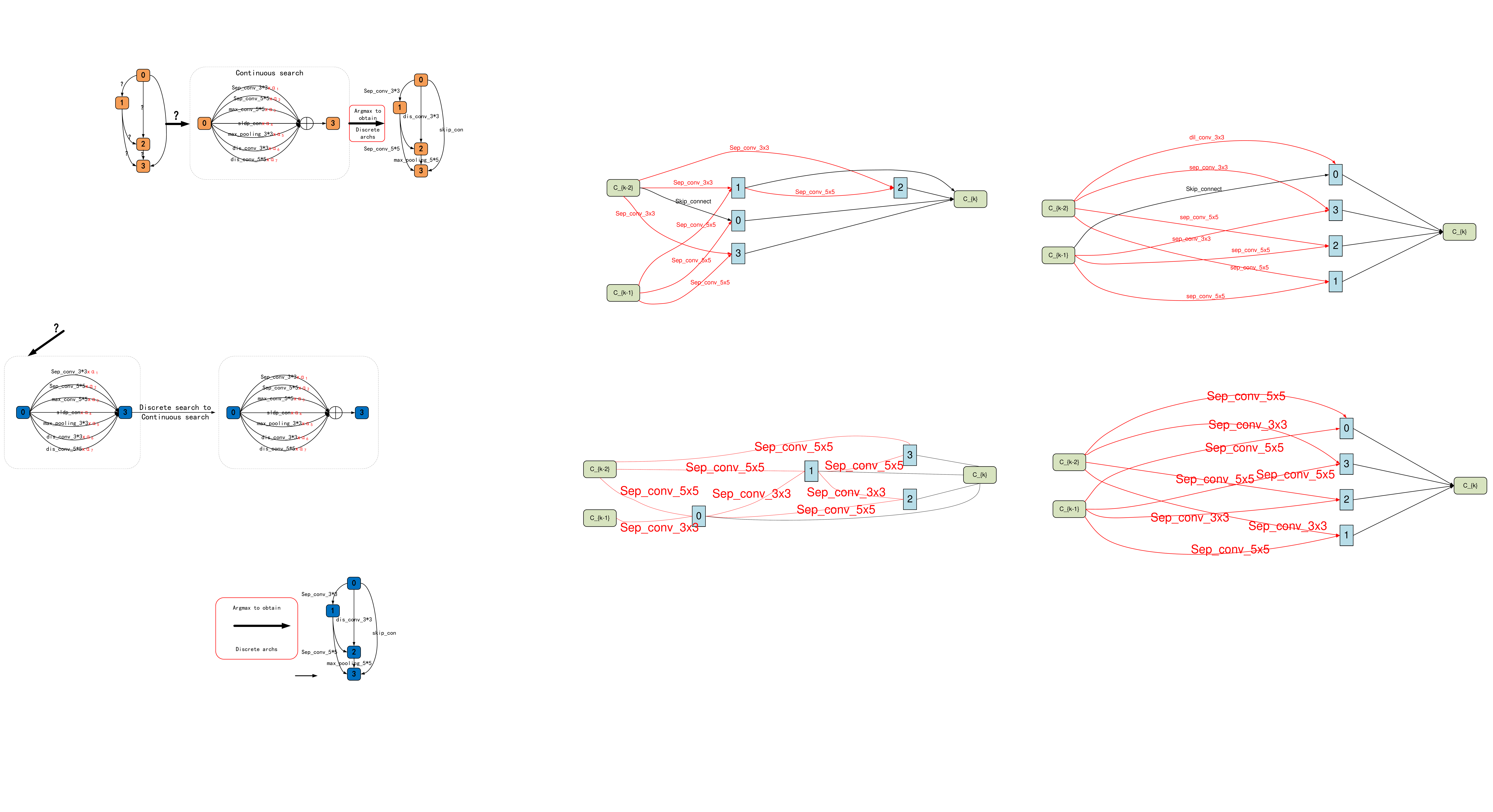}
}

\subfigure[Discovered architectures in MobileNet-like Search Space.]{
\includegraphics[height=2.2cm]{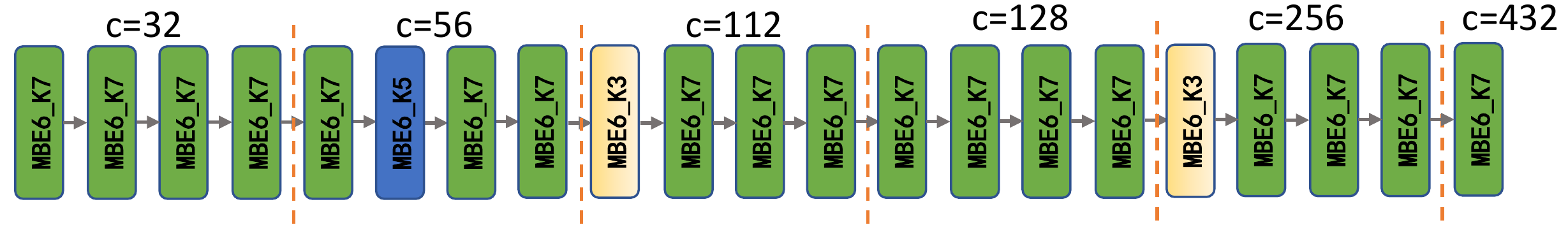}
}
 \caption{The cells discovered by Zero-Cost-PT \cite{xiang2021zero} with different datasets on the DARTS search space.}
 \label{fig:zero-pt}
\end{figure*}

\begin{figure}[t]
\centering
 \subfigure[NAS-Bench-201]{
\includegraphics[width=4cm,height=2cm]{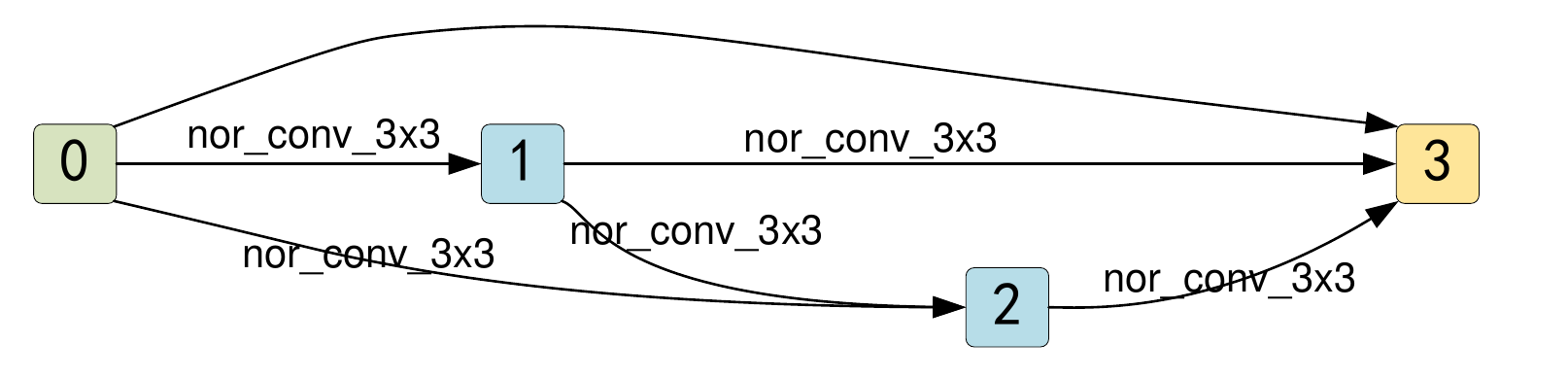}
}
 \subfigure[{Normal Cell on DARTS}]{
\includegraphics[width=4cm,height=2cm]{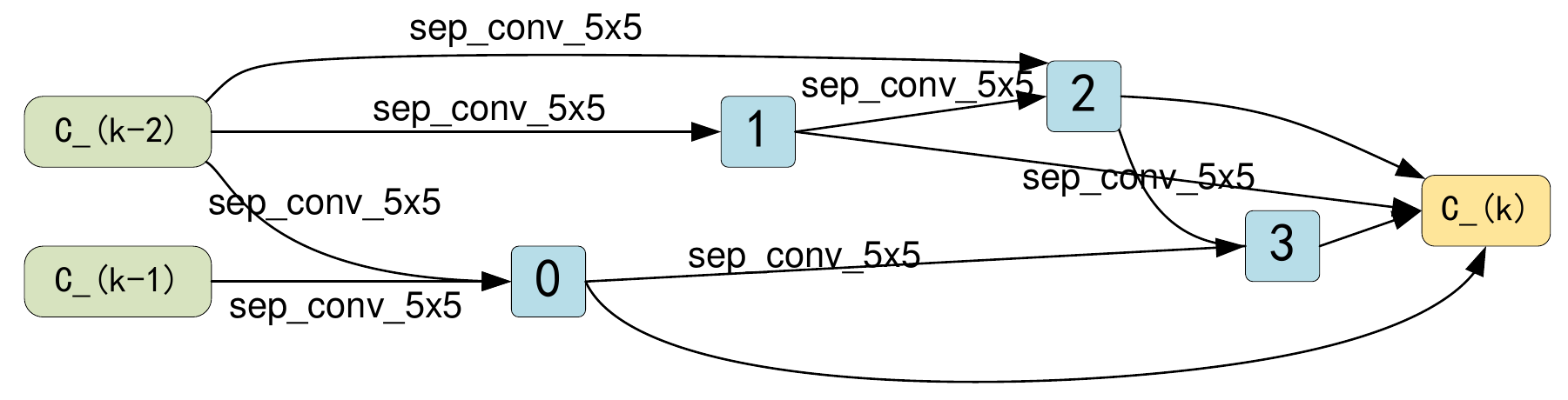}
}
 \caption{The architectures searched by zero-cost NAS on NAS-Bench-201 and DARTS space \cite{ning2021evaluating}.}
 \label{fig:zerocost_archs}
\end{figure}

\subsection{Discussion on Advantages and Limitations}
\label{sect_discuss}
This section discusses the advantages and the limitations of the proposed FreeDARTS framework by summarizing the above experimental findings. An attractive advantage of our FreeDARTS is the efficiency and effectiveness, which completes the entire architecture search in seconds with obtaining competitive results. More important, our FreeDARTS can be designed in a label-agnostic or data-agnostic way, which can still achieve competitive performance. In addition, thanks to the memory efficiency by our FreeDARTS, we can eliminate the \textit{depth gap} and \textit{dataset gap} that we can conduct architecture search and evaluation in the same space and same dataset. More important, our FreeDARTS is hyperparameter-efficient with only one hyperparameter to be tuned. 

Apart from our FreeDARTS, several recent works \cite{mellor2020neural,chen2021neural} also share a similar motivation as us to identify promising architectures at initialization. However, comparing with ``summing-up'' paradigm which is consider for most train-free NAS, FreeDARTS directly calculate the connection sensitivity based on the operation sensitivity. In addition, we have theoretically and empirically verified that FreeDARTS can avoid the parameter-intensive bias which is the main issue for existing zero-cost NAS methods \cite{2022blog,ning2021evaluating,yang2023revisiting}. Without any elaborately considerations, FreeDARTS could achieve more reliable results than differentiable NAS baselines with much less memory and computational consumption. This also implies that explicitly designing a more specific and explainable score function for NAS is a promising direction.

In theoretically analyzing \textit{zero-cost operation sensitivity} for DARTS in Sec.\ref{sec3.2}, we only consider a simple situation that the supernet $f(\alpha) =  h^{(L)} $ with $L$ sequential nodes and there are $E$ candidate edges between two consecutive nodes, which is different from most practical search spaces. In addition, for easy analysis, we only consider those candidate operations containing trainable parameters in the theoretical analysis. In our future work, we aim to consider a more practical and complicated situation in theoretically analyzing our FreeDARTS.

\section{Conclusion and Future Work}

We argue the common practice in neural architecture search by raising the question: \textit{can we properly measure the operation importance in DARTS through a training-free way, with avoiding the parameter-intensive bias?} We approach this question by reformulating differentiable architecture search from an operation pruning-at-initialization perspective, and introducing the \textit{zero-cost operation sensitivity} to score the importance of operations in the supernet at initialization, which is able to alleviate the parameter-intensive bias. By devising an iterative and data-agnostic manner in utilizing ZEROS for NAS, a novel framework called \textit{training free neural architecture search} (\textbf{FreeDARTS}) is accordingly devised, which performs architecture search extremely efficiently and flexibly, achieving competitive or even better results against SOTA. Extensive results on NAS benchmark datasets, the common DARTS space, and the MobileNet space have verified the effectiveness and efficiency of the proposed framework. We bridge the gap between the pruning-at-initialization (\textit{PaI}) and the differentiable architecture search. In addition we hypothesize the information flow is a more appropriate indicator to find a good architecture than the validation performance in differentiable architecture search. Our findings are capable of encouraging the community to further explore NAS through a lens of operation sensitivity.

{\small
\bibliographystyle{IEEEtranN}
\bibliography{FREEDARTS}

\begin{thebibliography}{63}
\providecommand{\natexlab}[1]{#1}
\providecommand{\url}[1]{#1}
\csname url@samestyle\endcsname
\providecommand{\newblock}{\relax}
\providecommand{\bibinfo}[2]{#2}
\providecommand{\BIBentrySTDinterwordspacing}{\spaceskip=0pt\relax}
\providecommand{\BIBentryALTinterwordstretchfactor}{4}
\providecommand{\BIBentryALTinterwordspacing}{\spaceskip=\fontdimen2\font plus
\BIBentryALTinterwordstretchfactor\fontdimen3\font minus
  \fontdimen4\font\relax}
\providecommand{\BIBforeignlanguage}[2]{{%
\expandafter\ifx\csname l@#1\endcsname\relax
\typeout{** WARNING: IEEEtranN.bst: No hyphenation pattern has been}%
\typeout{** loaded for the language `#1'. Using the pattern for}%
\typeout{** the default language instead.}%
\else
\language=\csname l@#1\endcsname
\fi
#2}}
\providecommand{\BIBdecl}{\relax}
\BIBdecl

\bibitem[Ren et~al.(2020)Ren, Xiao, Chang, Huang, Li, Chen, and
  Wang]{ren2020comprehensive}
P.~Ren, Y.~Xiao, X.~Chang, P.-Y. Huang, Z.~Li, X.~Chen, and X.~Wang, ``A
  comprehensive survey of neural architecture search: Challenges and
  solutions,'' \emph{arXiv preprint arXiv:2006.02903}, 2020.

\bibitem[Real et~al.(2019)Real, Aggarwal, Huang, and Le]{real2018regularized}
E.~Real, A.~Aggarwal, Y.~Huang, and Q.~V. Le, ``Regularized evolution for image
  classifier architecture search,'' \emph{AAAI}, 2019.

\bibitem[Guo et~al.(2019)Guo, Zhong, Wu, Lin, and Yan]{guo2018irlas}
M.~Guo, Z.~Zhong, W.~Wu, D.~Lin, and J.~Yan, ``Irlas: Inverse reinforcement
  learning for architecture search,'' in \emph{CVPR}, 2019.

\bibitem[Bender et~al.(2018)Bender, Kindermans, Zoph, Vasudevan, and
  Le]{bender2018understanding}
G.~Bender, P.-J. Kindermans, B.~Zoph, V.~Vasudevan, and Q.~Le, ``Understanding
  and simplifying one-shot architecture search,'' in \emph{International
  Conference on Machine Learning}, 2018, pp. 549--558.

\bibitem[Zhang et~al.(2019)Zhang, Ren, and Urtasun]{zhang2018graph}
C.~Zhang, M.~Ren, and R.~Urtasun, ``Graph hypernetworks for neural architecture
  search,'' in \emph{7th International Conference on Learning Representations},
  2019.

\bibitem[White et~al.(2021{\natexlab{a}})White, Zela, Ru, Liu, and
  Hutter]{white2021powerful}
C.~White, A.~Zela, R.~Ru, Y.~Liu, and F.~Hutter, ``How powerful are performance
  predictors in neural architecture search?'' in \emph{Advances in Neural
  Information Processing Systems}, vol.~34, 2021, pp. 28\,454--28\,469.

\bibitem[Mellor et~al.(2021)Mellor, Turner, Storkey, and
  Crowley]{mellor2020neural}
J.~Mellor, J.~Turner, A.~Storkey, and E.~J. Crowley, ``Neural architecture
  search without training,'' in \emph{ICML}, 2021.

\bibitem[Shu et~al.(2021)Shu, Cai, Dai, Ooi, and Low]{shu2021nasi}
Y.~Shu, S.~Cai, Z.~Dai, B.~C. Ooi, and B.~K.~H. Low, ``Nasi: Label-and
  data-agnostic neural architecture search at initialization,'' in
  \emph{International Conference on Learning Representations}, 2021.

\bibitem[Chen et~al.(2021)Chen, Gong, and Wang]{chen2021neural}
W.~Chen, X.~Gong, and Z.~Wang, ``Neural architecture search on imagenet in four
  gpu hours: A theoretically inspired perspective,'' in \emph{ICLR}, 2021.

\bibitem[Abdelfattah et~al.(2021)Abdelfattah, Mehrotra, Dudziak, and
  Lane]{abdelfattah2021zero}
M.~S. Abdelfattah, A.~Mehrotra, {\L}.~Dudziak, and N.~D. Lane, ``Zero-cost
  proxies for lightweight nas,'' in \emph{ICLR}, 2021.

\bibitem[Ning et~al.(2021)Ning, Tang, Li, Zhou, Liang, Yang, and
  Wang]{ning2021evaluating}
X.~Ning, C.~Tang, W.~Li, Z.~Zhou, S.~Liang, H.~Yang, and Y.~Wang, ``Evaluating
  efficient performance estimators of neural architectures,'' in \emph{Advances
  in Neural Information Processing Systems}, vol.~34, 2021.

\bibitem[White et~al.(2021{\natexlab{b}})White, Khodak, Tu, Shah, Bubeck, and
  Dey]{2022blog}
C.~White, M.~Khodak, R.~Tu, S.~Shah, S.~Bubeck, and D.~Dey, ``Neural
  architecture search on imagenet in four gpu hours: A theoretically inspired
  perspective,'' in \emph{ICLR}, 2021.

\bibitem[Yang et~al.(2023)Yang, Yang, Jin, and Chen]{yang2023revisiting}
T.~Yang, L.~Yang, X.~Jin, and C.~Chen, ``Revisiting training-free nas metrics:
  An efficient training-based method,'' in \emph{Proceedings of the IEEE/CVF
  Winter Conference on Applications of Computer Vision}, 2023, pp. 4751--4760.

\bibitem[Yang and Wang(2022)]{yang2022neural}
H.~Yang and Z.~Wang, ``On the neural tangent kernel analysis of randomly pruned
  wide neural networks,'' \emph{arXiv preprint arXiv:2203.14328}, 2022.

\bibitem[Shu et~al.(2022)Shu, Dai, Wu, and Low]{shu2022unifying}
Y.~Shu, Z.~Dai, Z.~Wu, and B.~K.~H. Low, ``Unifying and boosting gradient-based
  training-free neural architecture search,'' in \emph{Advances in Neural
  Information Processing Systems}, 2022.

\bibitem[Liu et~al.(2019)Liu, Simonyan, and Yang]{liu2018darts}
H.~Liu, K.~Simonyan, and Y.~Yang, ``Darts: Differentiable architecture
  search,'' \emph{ICLR}, 2019.

\bibitem[Wang et~al.(2021)Wang, Cheng, Chen, Tang, and Hsieh]{Rethinking2021}
R.~Wang, M.~Cheng, X.~Chen, X.~Tang, and C.-J. Hsieh, ``Rethinking architecture
  selection in differentiable nas,'' in \emph{ICLR}, 2021.

\bibitem[Xiang et~al.(2021)Xiang, Dudziak, Abdelfattah, Chau, Lane, and
  Wen]{xiang2021zero}
L.~Xiang, {\L}.~Dudziak, M.~S. Abdelfattah, T.~Chau, N.~D. Lane, and H.~Wen,
  ``Zero-cost proxies meet differentiable architecture search,'' \emph{arXiv
  preprint arXiv:2106.06799}, 2021.

\bibitem[Dong and Yang(2020)]{BENCH102}
X.~Dong and Y.~Yang, ``Nas-bench-201: Extending the scope of reproducible
  neural architecture search,'' \emph{ICLR}, 2020.

\bibitem[Zela et~al.(2020{\natexlab{a}})Zela, Siems, and
  Hutter]{zela2020nasbench1shot1}
A.~Zela, J.~Siems, and F.~Hutter, ``Nas-bench-1shot1: Benchmarking and
  dissecting one-shot neural architecture search,'' in \emph{ICLR}, 2020.

\bibitem[Howard et~al.(2019)Howard, Sandler, Chu, Chen, Chen, Tan, Wang, Zhu,
  Pang, Vasudevan, et~al.]{howard2019searching}
A.~Howard, M.~Sandler, G.~Chu, L.-C. Chen, B.~Chen, M.~Tan, W.~Wang, Y.~Zhu,
  R.~Pang, V.~Vasudevan \emph{et~al.}, ``Searching for mobilenetv3,'' in
  \emph{Proceedings of the IEEE/CVF International Conference on Computer
  Vision}, 2019, pp. 1314--1324.

\bibitem[Lee et~al.(2019{\natexlab{a}})Lee, Ajanthan, and Torr]{lee2018snip}
N.~Lee, T.~Ajanthan, and P.~Torr, ``Snip: Single-shot network pruning based on
  connection sensitivity,'' in \emph{International Conference on Learning
  Representations}, 2019.

\bibitem[Wang et~al.(2020)Wang, Zhang, and Grosse]{wang2019picking}
C.~Wang, G.~Zhang, and R.~Grosse, ``Picking winning tickets before training by
  preserving gradient flow,'' in \emph{International Conference on Learning
  Representations}, 2020.

\bibitem[Tanaka et~al.(2020)Tanaka, Kunin, Yamins, and
  Ganguli]{tanaka2020pruning}
H.~Tanaka, D.~Kunin, D.~L. Yamins, and S.~Ganguli, ``Pruning neural networks
  without any data by iteratively conserving synaptic flow,'' \emph{Advances in
  Neural Information Processing Systems}, vol.~33, 2020.

\bibitem[Lin et~al.(2021{\natexlab{a}})Lin, Wang, Sun, Chen, Sun, Qian, Li, and
  Jin]{lin2021zen}
M.~Lin, P.~Wang, Z.~Sun, H.~Chen, X.~Sun, Q.~Qian, H.~Li, and R.~Jin,
  ``Zen-nas: A zero-shot nas for high-performance image recognition,'' in
  \emph{Proceedings of the IEEE/CVF International Conference on Computer
  Vision}, 2021, pp. 347--356.

\bibitem[Xu et~al.(2021)Xu, Zhao, Lin, Gao, Xu, and Hongxia]{knas}
J.~Xu, L.~Zhao, J.~Lin, R.~Gao, S.~Xu, and Y.~Hongxia, ``Knas: green neural
  architecture search,'' in \emph{International Conference on Machine
  Learning}, 2021, pp. 11\,613--11\,625.

\bibitem[Zhang and Jia()]{gradsign}
Z.~Zhang and Z.~Jia, ``Gradsign: Model performance inference with theoretical
  insights,'' in \emph{International Conference on Learning Representations}.

\bibitem[Zhou et~al.(2022)Zhou, Sheng, Zheng, Li, Sun, Tian, Chen, and
  Ji]{zhou2022training}
Q.~Zhou, K.~Sheng, X.~Zheng, K.~Li, X.~Sun, Y.~Tian, J.~Chen, and R.~Ji,
  ``Training-free transformer architecture search,'' in \emph{Proceedings of
  the IEEE/CVF Conference on Computer Vision and Pattern Recognition}, 2022,
  pp. 10\,894--10\,903.

\bibitem[Lin et~al.(2021{\natexlab{b}})Lin, Wang, Sun, Chen, Sun, Qian, Li, and
  Jin]{zenscore}
M.~Lin, P.~Wang, Z.~Sun, H.~Chen, X.~Sun, Q.~Qian, H.~Li, and R.~Jin,
  ``Zen-nas: A zero-shot nas for high-performance image recognition,'' in
  \emph{Proceedings of the IEEE/CVF International Conference on Computer
  Vision}, 2021, pp. 347--356.

\bibitem[Dong et~al.(2023)Dong, Li, and Wei]{DisWOT}
P.~Dong, L.~Li, and Z.~Wei, ``Diswot: Student architecture search for
  distillation without training,'' in \emph{CVPR}, 2023.

\bibitem[Zela et~al.(2020{\natexlab{b}})Zela, Elsken, Saikia, Marrakchi, Brox,
  and Hutter]{zela2019understanding}
A.~Zela, T.~Elsken, T.~Saikia, Y.~Marrakchi, T.~Brox, and F.~Hutter,
  ``Understanding and robustifying differentiable architecture search,'' in
  \emph{ICLR}, 2020.

\bibitem[Chen and Hsieh(2020)]{chen2020stabilizing}
X.~Chen and C.-J. Hsieh, ``Stabilizing differentiable architecture search via
  perturbation-based regularization,'' in \emph{ICML}, 2020.

\bibitem[Sciuto et~al.(2019)Sciuto, Yu, Jaggi, Musat, and
  Salzmann]{sciuto2019evaluating}
C.~Sciuto, K.~Yu, M.~Jaggi, C.~Musat, and M.~Salzmann, ``Evaluating the search
  phase of neural architecture search,'' \emph{arXiv preprint
  arXiv:1902.08142}, 2019.

\bibitem[Liang et~al.(2019)Liang, Zhang, Sun, He, Huang, Zhuang, and
  Li]{liang2019darts+}
H.~Liang, S.~Zhang, J.~Sun, X.~He, W.~Huang, K.~Zhuang, and Z.~Li, ``Darts+:
  Improved differentiable architecture search with early stopping,''
  \emph{arXiv preprint arXiv:1909.06035}, 2019.

\bibitem[Du et~al.(2019)Du, Lee, Li, Wang, and Zhai]{du2019gradient}
S.~Du, J.~Lee, H.~Li, L.~Wang, and X.~Zhai, ``Gradient descent finds global
  minima of deep neural networks,'' in \emph{International Conference on
  Machine Learning}.\hskip 1em plus 0.5em minus 0.4em\relax PMLR, 2019, pp.
  1675--1685.

\bibitem[Du et~al.(2018)Du, Zhai, Poczos, and Singh]{du2018gradient}
S.~S. Du, X.~Zhai, B.~Poczos, and A.~Singh, ``Gradient descent provably
  optimizes over-parameterized neural networks,'' \emph{arXiv preprint
  arXiv:1810.02054}, 2018.

\bibitem[Allen-Zhu et~al.(2019)Allen-Zhu, Li, and Song]{allen2019convergence}
Z.~Allen-Zhu, Y.~Li, and Z.~Song, ``A convergence theory for deep learning via
  over-parameterization,'' in \emph{International Conference on Machine
  Learning}.\hskip 1em plus 0.5em minus 0.4em\relax PMLR, 2019, pp. 242--252.

\bibitem[Jacot et~al.(2018)Jacot, Gabriel, and Hongler]{jacot2018neural}
A.~Jacot, F.~Gabriel, and C.~Hongler, ``Neural tangent kernel: Convergence and
  generalization in neural networks,'' \emph{arXiv preprint arXiv:1806.07572},
  2018.

\bibitem[Arora et~al.(2019{\natexlab{a}})Arora, Du, Hu, Li, Salakhutdinov, and
  Wang]{arora2019exact}
S.~Arora, S.~S. Du, W.~Hu, Z.~Li, R.~Salakhutdinov, and R.~Wang, ``On exact
  computation with an infinitely wide neural net,'' \emph{arXiv preprint
  arXiv:1904.11955}, 2019.

\bibitem[Lee et~al.(2019{\natexlab{b}})Lee, Xiao, Schoenholz, Bahri, Novak,
  Sohl-Dickstein, and Pennington]{lee2019wide}
J.~Lee, L.~Xiao, S.~Schoenholz, Y.~Bahri, R.~Novak, J.~Sohl-Dickstein, and
  J.~Pennington, ``Wide neural networks of any depth evolve as linear models
  under gradient descent,'' \emph{Advances in neural information processing
  systems}, vol.~32, 2019.

\bibitem[Han et~al.(2015)Han, Pool, Tran, and Dally]{han2015learning}
S.~Han, J.~Pool, J.~Tran, and W.~J. Dally, ``Learning both weights and
  connections for efficient neural networks,'' in \emph{Proceedings of the 28th
  International Conference on Neural Information Processing Systems-Volume 1},
  2015, pp. 1135--1143.

\bibitem[Wang et~al.(2022)Wang, Qin, Bai, Zhang, and Fu]{wang2022recent}
H.~Wang, C.~Qin, Y.~Bai, Y.~Zhang, and Y.~Fu, ``Recent advances on neural
  network pruning at initialization,'' in \emph{Proceedings of the
  International Joint Conference on Artificial Intelligence, IJCAI, Vienna,
  Austria}, 2022, pp. 23--29.

\bibitem[Verdenius et~al.(2020)Verdenius, Stol, and
  Forr{\'e}]{verdenius2020pruning}
S.~Verdenius, M.~Stol, and P.~Forr{\'e}, ``Pruning via iterative ranking of
  sensitivity statistics,'' \emph{arXiv preprint arXiv:2006.00896}, 2020.

\bibitem[Liu et~al.(2020)Liu, Doll{\'a}r, He, Girshick, Yuille, and
  Xie]{liu2020labels}
C.~Liu, P.~Doll{\'a}r, K.~He, R.~Girshick, A.~Yuille, and S.~Xie, ``Are labels
  necessary for neural architecture search?'' in \emph{European Conference on
  Computer Vision}.\hskip 1em plus 0.5em minus 0.4em\relax Springer, 2020, pp.
  798--813.

\bibitem[Siems et~al.(2020)Siems, Zimmer, Zela, Lukasik, Keuper, and
  Hutter]{siems2020bench}
J.~Siems, L.~Zimmer, A.~Zela, J.~Lukasik, M.~Keuper, and F.~Hutter,
  ``Nas-bench-301 and the case for surrogate benchmarks for neural architecture
  search,'' \emph{arXiv preprint arXiv:2008.09777}, 2020.

\bibitem[Ying et~al.(2019)Ying, Klein, Christiansen, Real, Murphy, and
  Hutter]{ying2019bench}
C.~Ying, A.~Klein, E.~Christiansen, E.~Real, K.~Murphy, and F.~Hutter,
  ``Nas-bench-101: Towards reproducible neural architecture search,'' in
  \emph{ICML}, 2019, pp. 7105--7114.

\bibitem[Zhang et~al.(2021)Zhang, Hou, Zhang, and Sun]{zhang2021neural}
X.~Zhang, P.~Hou, X.~Zhang, and J.~Sun, ``Neural architecture search with
  random labels,'' \emph{arXiv preprint arXiv:2101.11834}, 2021.

\bibitem[Dong and Yang(2019)]{GDAS}
X.~Dong and Y.~Yang, ``Searching for a robust neural architecture in four gpu
  hours,'' in \emph{IEEE Conference on Computer Vision and Pattern
  Recognition}.\hskip 1em plus 0.5em minus 0.4em\relax IEEE Computer Society,
  2019.

\bibitem[Xu et~al.(2020)Xu, Xie, Zhang, Chen, Qi, Tian, and
  Xiong]{xu2019pcdarts}
Y.~Xu, L.~Xie, X.~Zhang, X.~Chen, G.-J. Qi, Q.~Tian, and H.~Xiong, ``Pc-darts:
  Partial channel connections for memory-efficient architecture search,'' in
  \emph{ICLR}, 2020.

\bibitem[Li and Talwalkar(2019)]{li2019random}
L.~Li and A.~Talwalkar, ``Random search and reproducibility for neural
  architecture search,'' \emph{arXiv preprint arXiv:1902.07638}, 2019.

\bibitem[Xie et~al.(2019)Xie, Zheng, Liu, and Lin]{xie2018snas}
S.~Xie, H.~Zheng, C.~Liu, and L.~Lin, ``Snas: stochastic neural architecture
  search,'' \emph{ICLR}, 2019.

\bibitem[Zhou et~al.(2019)Zhou, Yang, Wang, and Pan]{zhou2019bayesnas}
H.~Zhou, M.~Yang, J.~Wang, and W.~Pan, ``Bayesnas: A bayesian approach for
  neural architecture search,'' in \emph{International Conference on Machine
  Learning}, 2019, pp. 7603--7613.

\bibitem[Chen et~al.(2019)Chen, Xie, Wu, and Tian]{chen2019progressive}
X.~Chen, L.~Xie, J.~Wu, and Q.~Tian, ``Progressive differentiable architecture
  search: Bridging the depth gap between search and evaluation,'' in
  \emph{Proceedings of the IEEE International Conference on Computer Vision},
  2019, pp. 1294--1303.

\bibitem[Chen et~al.(2020)Chen, Wang, Cheng, Tang, and Hsieh]{chen2020drnas}
X.~Chen, R.~Wang, M.~Cheng, X.~Tang, and C.-J. Hsieh, ``Drnas: Dirichlet neural
  architecture search,'' \emph{arXiv preprint arXiv:2006.10355}, 2020.

\bibitem[Zoph et~al.(2018)Zoph, Vasudevan, Shlens, and Le]{zoph2018learning}
B.~Zoph, V.~Vasudevan, J.~Shlens, and Q.~V. Le, ``Learning transferable
  architectures for scalable image recognition,'' in \emph{Proceedings of the
  IEEE conference on computer vision and pattern recognition}, 2018, pp.
  8697--8710.

\bibitem[Tan et~al.(2019)Tan, Chen, Pang, Vasudevan, Sandler, Howard, and
  Le]{tan2019mnasnet}
M.~Tan, B.~Chen, R.~Pang, V.~Vasudevan, M.~Sandler, A.~Howard, and Q.~V. Le,
  ``Mnasnet: Platform-aware neural architecture search for mobile,'' in
  \emph{Proceedings of the IEEE/CVF Conference on Computer Vision and Pattern
  Recognition}, 2019, pp. 2820--2828.

\bibitem[Sandler et~al.(2018)Sandler, Howard, Zhu, Zhmoginov, and
  Chen]{sandler2018mobilenetv2}
M.~Sandler, A.~Howard, M.~Zhu, A.~Zhmoginov, and L.-C. Chen, ``Mobilenetv2:
  Inverted residuals and linear bottlenecks,'' in \emph{Proceedings of the IEEE
  conference on computer vision and pattern recognition}, 2018, pp. 4510--4520.

\bibitem[Tan and Le(2019)]{tan2019efficientnet}
M.~Tan and Q.~Le, ``Efficientnet: Rethinking model scaling for convolutional
  neural networks,'' in \emph{International conference on machine
  learning}.\hskip 1em plus 0.5em minus 0.4em\relax PMLR, 2019, pp. 6105--6114.

\bibitem[Stamoulis et~al.(2019)Stamoulis, Ding, Wang, Lymberopoulos, Priyantha,
  Liu, and Marculescu]{stamoulis2019single}
D.~Stamoulis, R.~Ding, D.~Wang, D.~Lymberopoulos, B.~Priyantha, J.~Liu, and
  D.~Marculescu, ``Single-path nas: Designing hardware-efficient convnets in
  less than 4 hours,'' in \emph{Joint European Conference on Machine Learning
  and Knowledge Discovery in Databases}.\hskip 1em plus 0.5em minus 0.4em\relax
  Springer, 2019, pp. 481--497.

\bibitem[Yu et~al.(2020)Yu, Jin, Liu, Bender, Kindermans, Tan, Huang, Song,
  Pang, and Le]{yu2020bignas}
J.~Yu, P.~Jin, H.~Liu, G.~Bender, P.-J. Kindermans, M.~Tan, T.~Huang, X.~Song,
  R.~Pang, and Q.~Le, ``Bignas: Scaling up neural architecture search with big
  single-stage models,'' in \emph{European Conference on Computer
  Vision}.\hskip 1em plus 0.5em minus 0.4em\relax Springer, 2020, pp. 702--717.

\bibitem[Cai et~al.(2019)Cai, Zhu, and Han]{cai2018proxylessnas}
H.~Cai, L.~Zhu, and S.~Han, ``Proxylessnas: Direct neural architecture search
  on target task and hardware,'' \emph{ICLR}, 2019.

\bibitem[Chu et~al.(2019)Chu, Zhang, Xu, and Li]{chu2019fairnas}
X.~Chu, B.~Zhang, R.~Xu, and J.~Li, ``Fairnas: Rethinking evaluation fairness
  of weight sharing neural architecture search,'' \emph{arXiv preprint
  arXiv:1907.01845}, 2019.

\bibitem[Arora et~al.(2019{\natexlab{b}})Arora, Du, Hu, Li, and
  Wang]{arora2019fine}
S.~Arora, S.~Du, W.~Hu, Z.~Li, and R.~Wang, ``Fine-grained analysis of
  optimization and generalization for overparameterized two-layer neural
  networks,'' in \emph{International Conference on Machine Learning}.\hskip 1em
  plus 0.5em minus 0.4em\relax PMLR, 2019, pp. 322--332.

\end{thebibliography}
}

\newpage
\appendices
\onecolumn
\section{Visualization of the Searched Architectures}
Since the inputs of data-agnostic FreeDARTS are all-ones vector, the search results will be the same for architecture search on CIFAR-10 and CIFAR-100. So, we only perform the architecture search on CIFAR-10 and ImageNet. In result, there are 15 architectures searched by our FreeDARTS, FreeDARTS$\dagger$, and FreeDARTS$\ddagger$ with 5 different \textit{random seeds 0-4}. In Figure \ref{fig:archs_darts}, we present the best architectures searched in DARTS space by FreeDARTS(seed4), FreeDARTS$\dagger$ (seed1), and FreeDARTS$\ddagger$ (seed3), and the remaining searched architectures are presented in Figure \ref{fig:archs_remianings}.

\begin{figure*}[h]
\centering
 \subfigure[normal and reduction cell by FreeDARTS with seed0]{
  \begin{minipage}{4cm}
      \includegraphics[width=4cm,height=2.5cm]{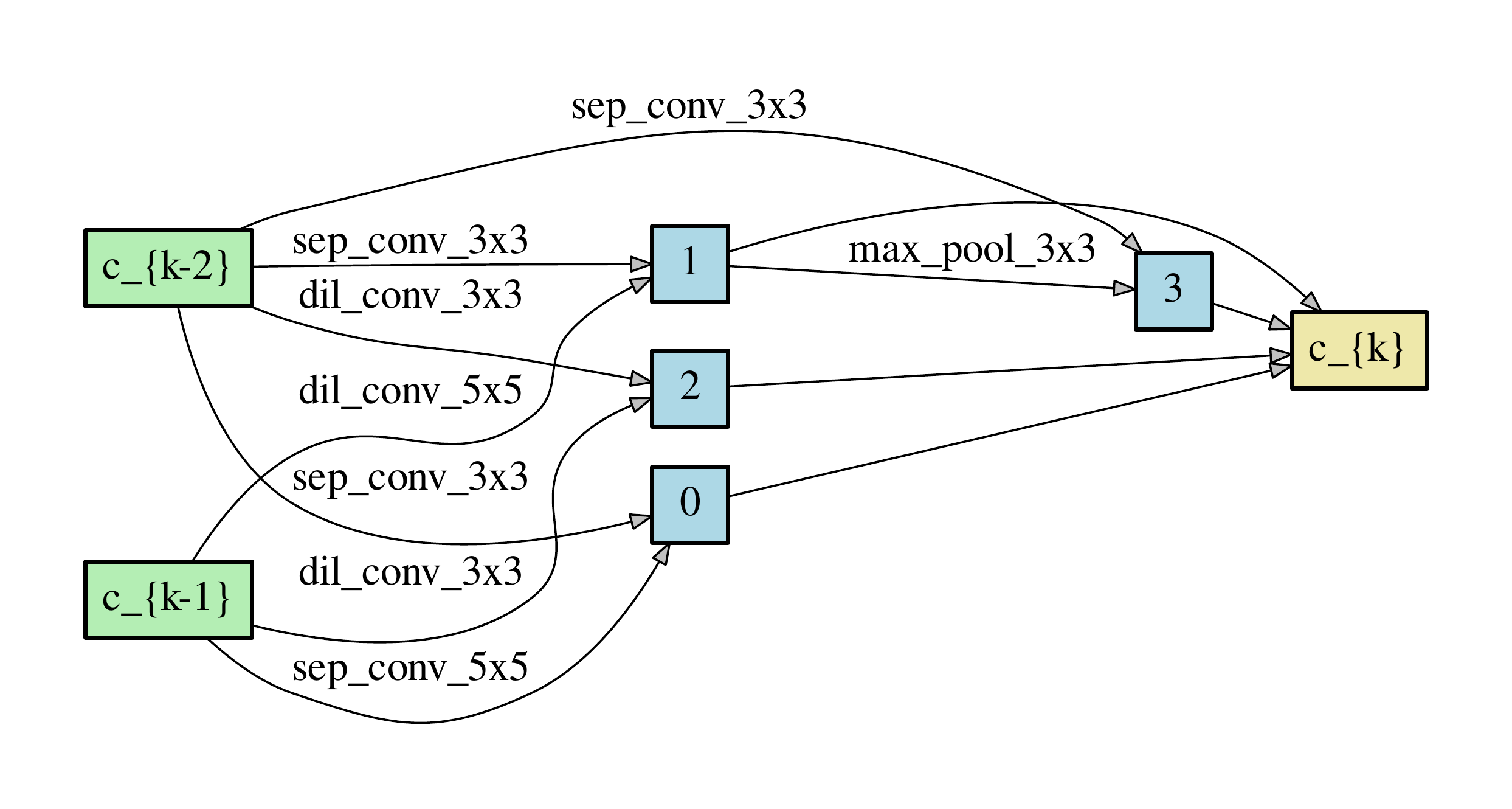}
  \end{minipage}
  \begin{minipage}{4cm}
      \includegraphics[width=4cm,height=2.5cm]{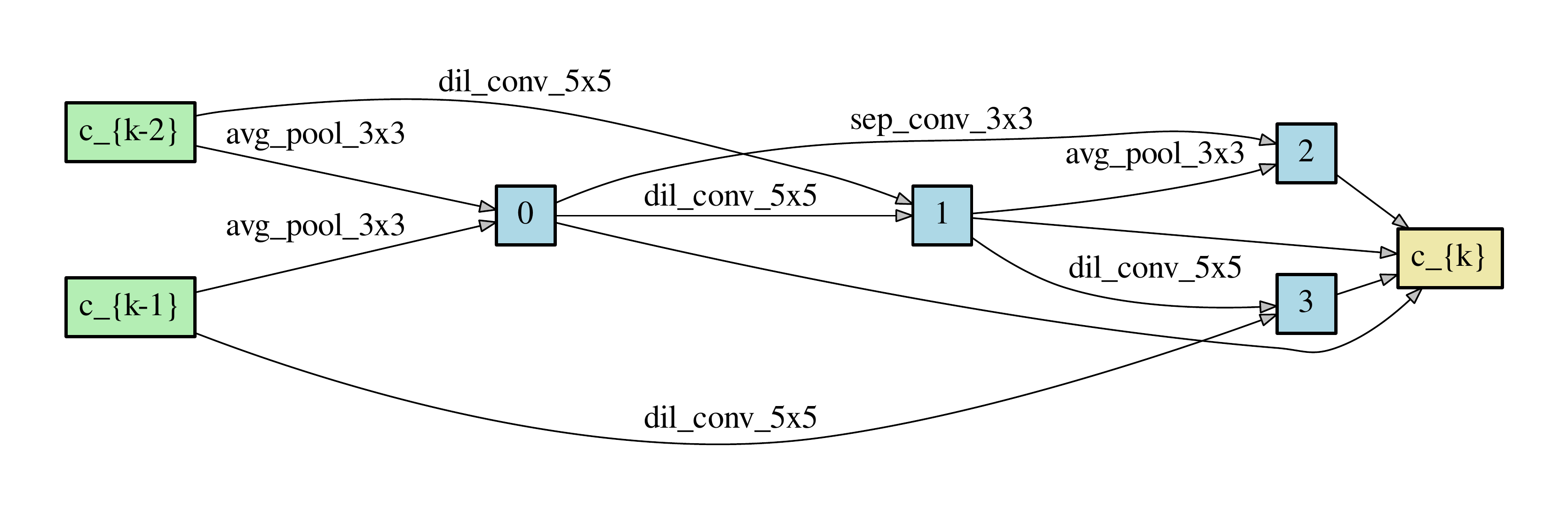}
  \end{minipage} 
  }
   \subfigure[normal and reduction cell by FreeDARTS with seed1]{
  \begin{minipage}{4cm}
      \includegraphics[width=4cm,height=2.5cm]{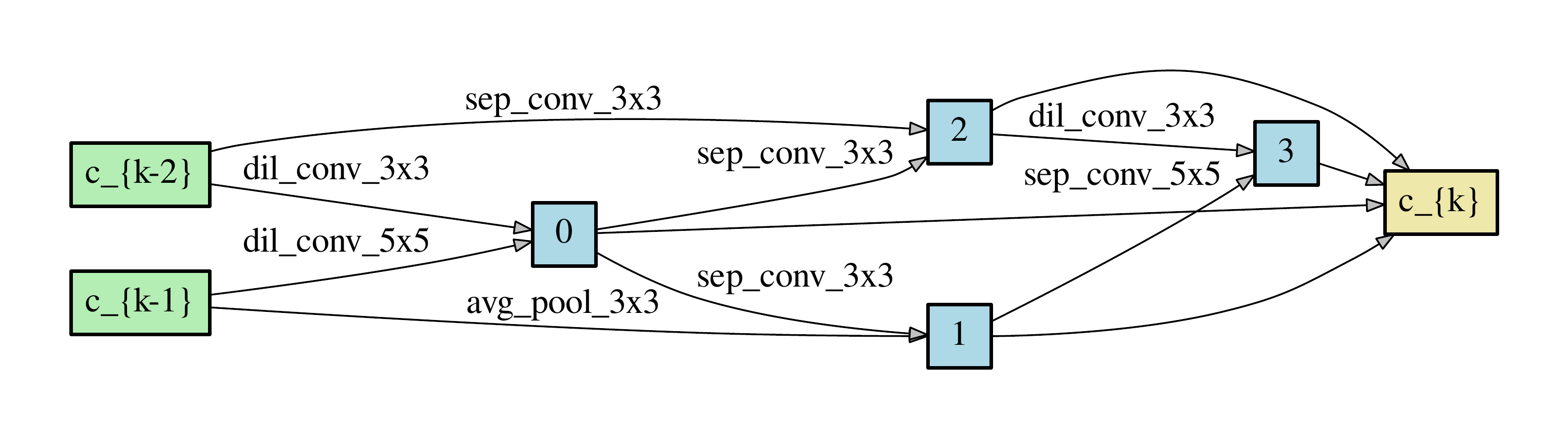}
  \end{minipage}
  \begin{minipage}{4cm}
      \includegraphics[width=4cm,height=2.5cm]{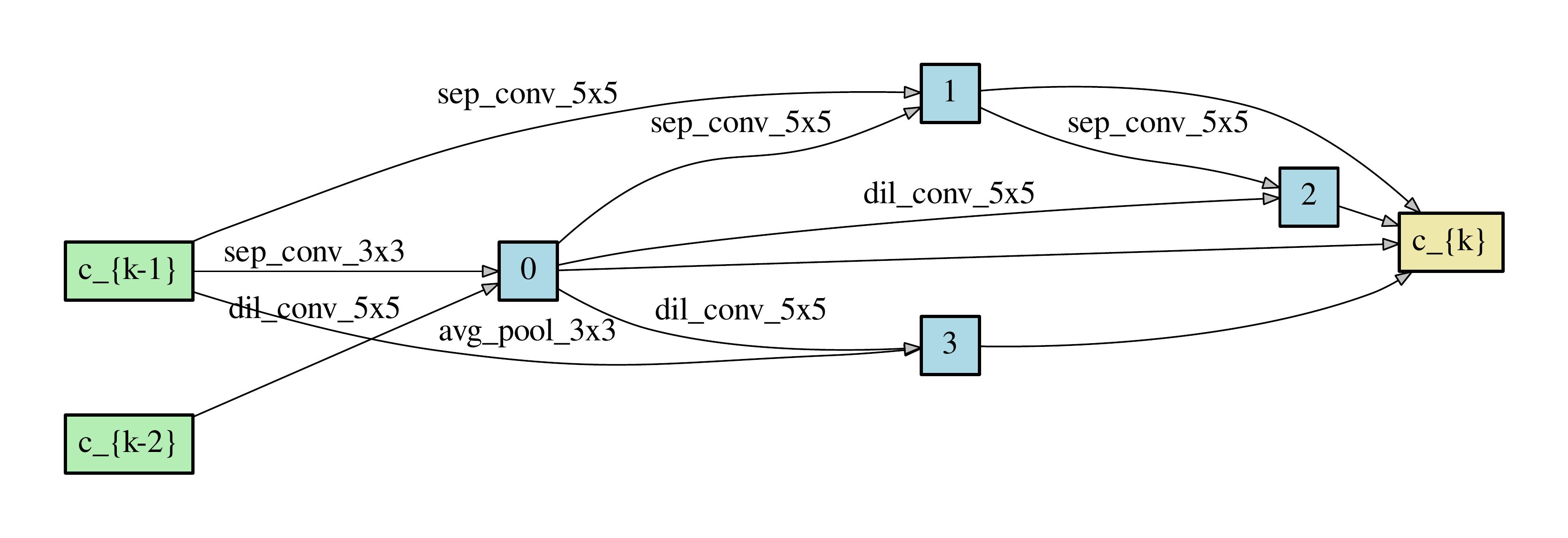}
  \end{minipage} 
  }
  
   \subfigure[normal and reduction cell by FreeDARTS with seed2]{
  \begin{minipage}{4cm}
      \includegraphics[width=4cm,height=2.5cm]{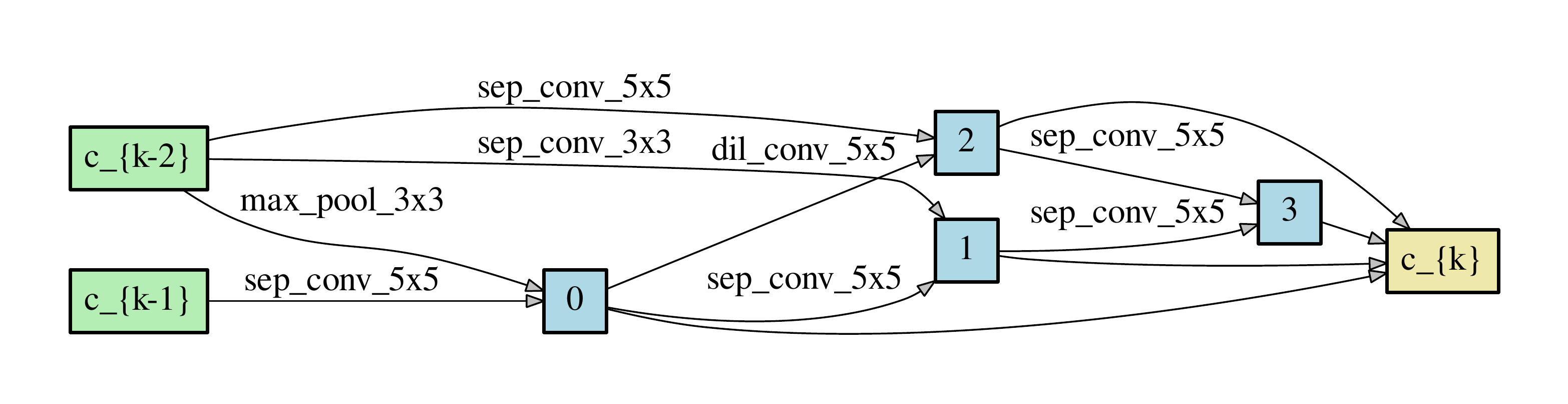}
  \end{minipage}
  \begin{minipage}{4cm}
      \includegraphics[width=4cm,height=2.5cm]{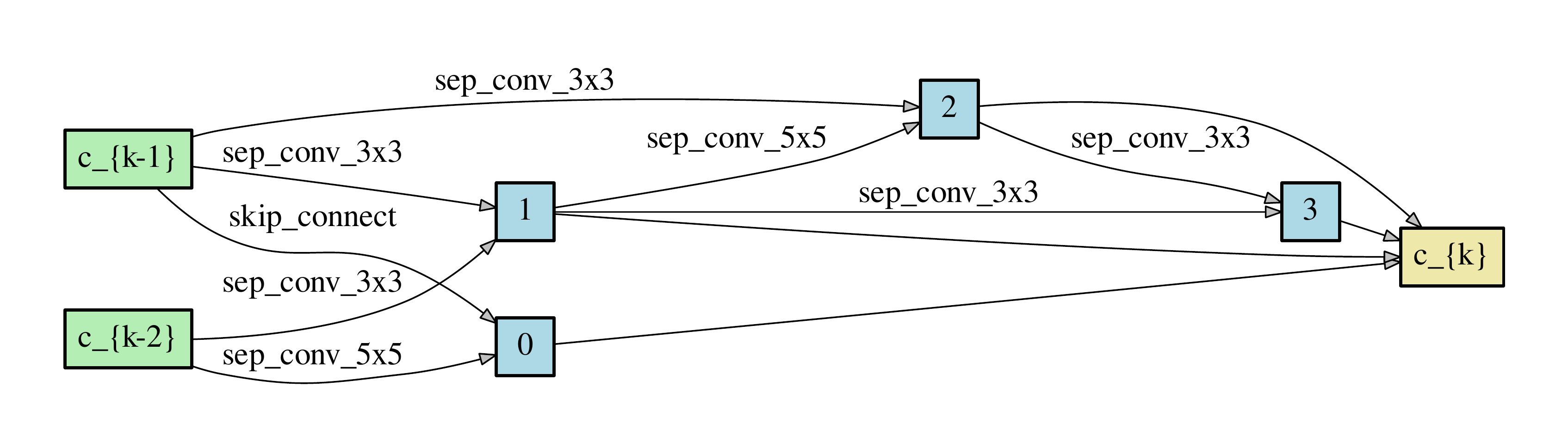}
  \end{minipage} 
  }
   \subfigure[normal and reduction cell by FreeDARTS with seed3]{
  \begin{minipage}{4cm}
      \includegraphics[width=4cm,height=2.5cm]{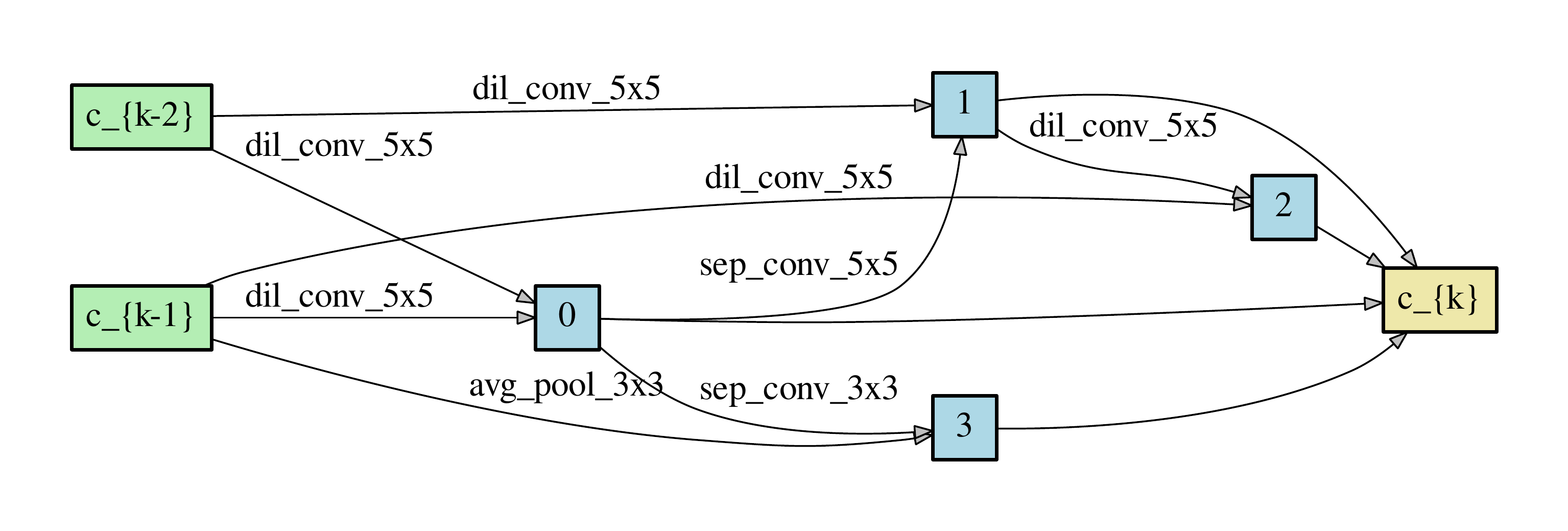}
  \end{minipage}
  \begin{minipage}{4cm}
      \includegraphics[width=4cm,height=2.5cm]{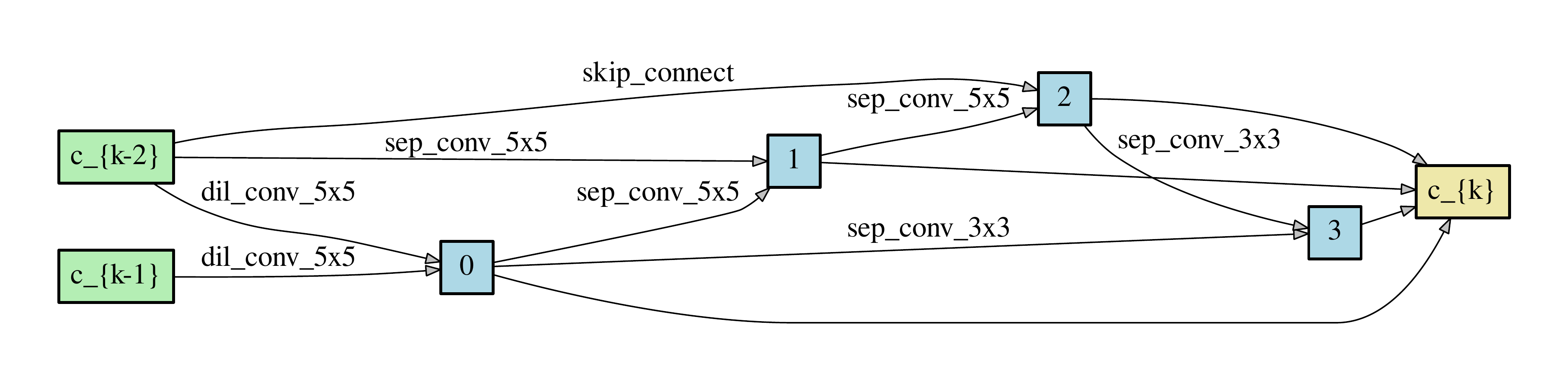}
  \end{minipage} 
  }
  
   \subfigure[normal and reduction cell by FreeDARTS$\dagger$ with seed0]{
  \begin{minipage}{4cm}
      \includegraphics[width=4cm,height=2.5cm]{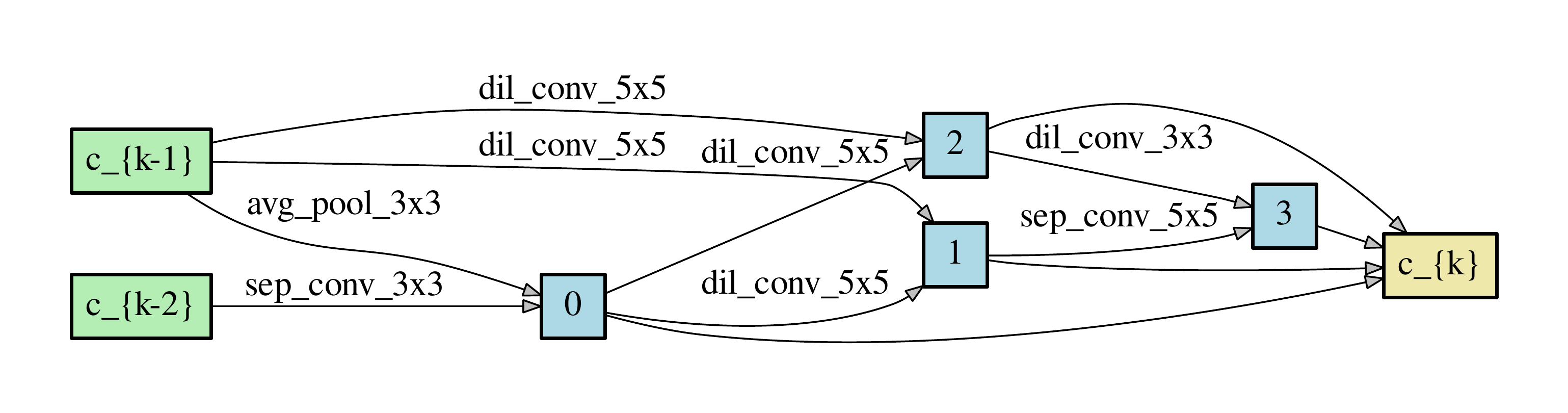}
  \end{minipage}
  \begin{minipage}{4cm}
      \includegraphics[width=4cm,height=2.5cm]{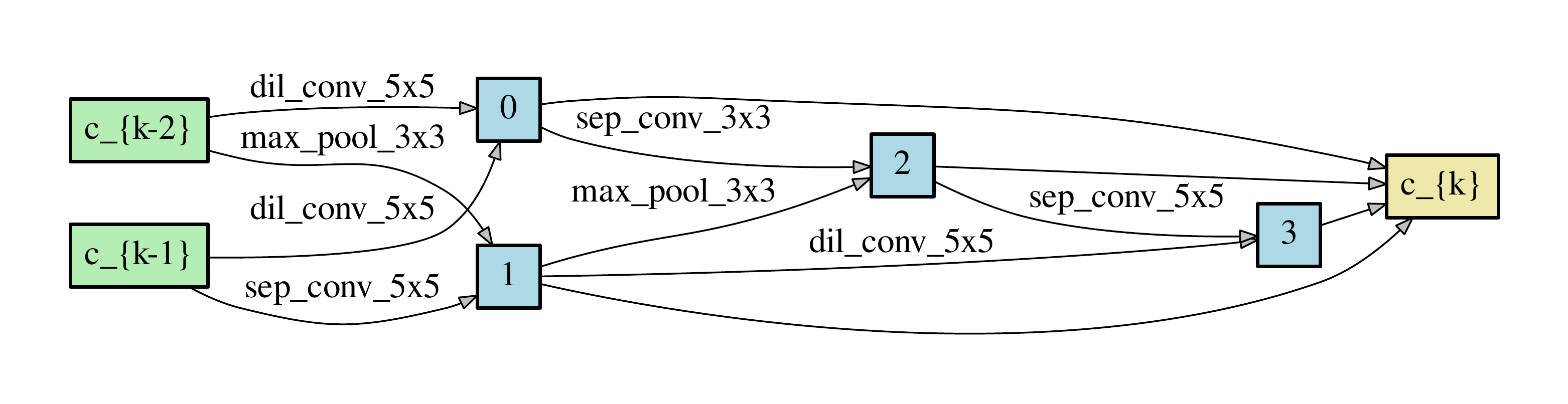}
  \end{minipage} 
  }
   \subfigure[normal and reduction cell by FreeDARTS$\dagger$ with seed2]{
  \begin{minipage}{4cm}
      \includegraphics[width=4cm,height=2.5cm]{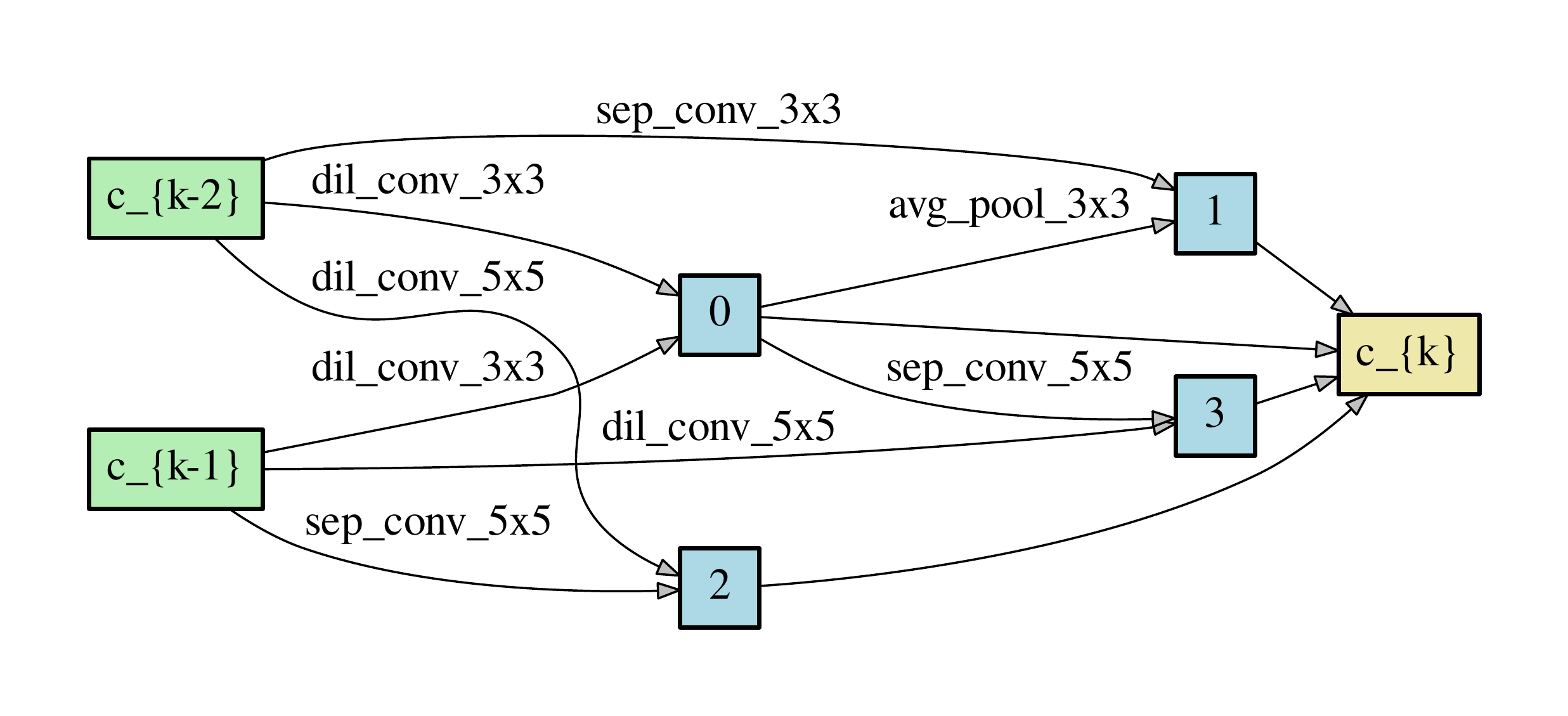}
  \end{minipage}
  \begin{minipage}{4cm}
      \includegraphics[width=4cm,height=2.5cm]{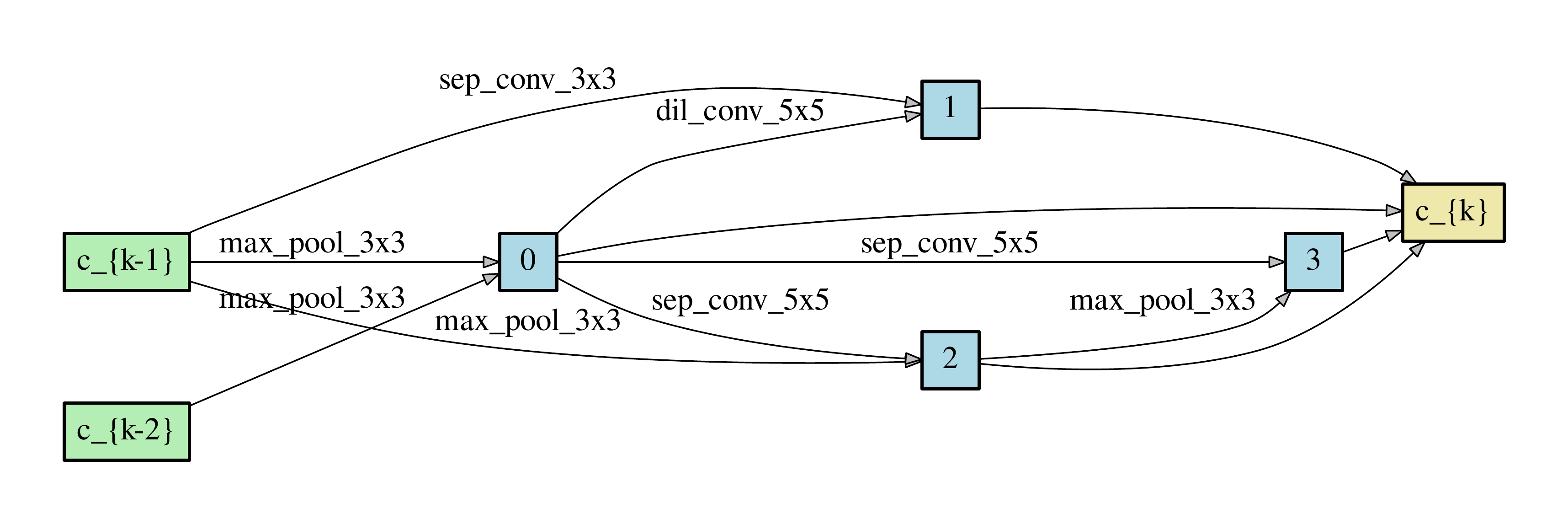}
  \end{minipage} 
  }
  
   \subfigure[normal and reduction cell by FreeDARTS$\dagger$ with seed3]{
  \begin{minipage}{4cm}
      \includegraphics[width=4cm,height=2.5cm]{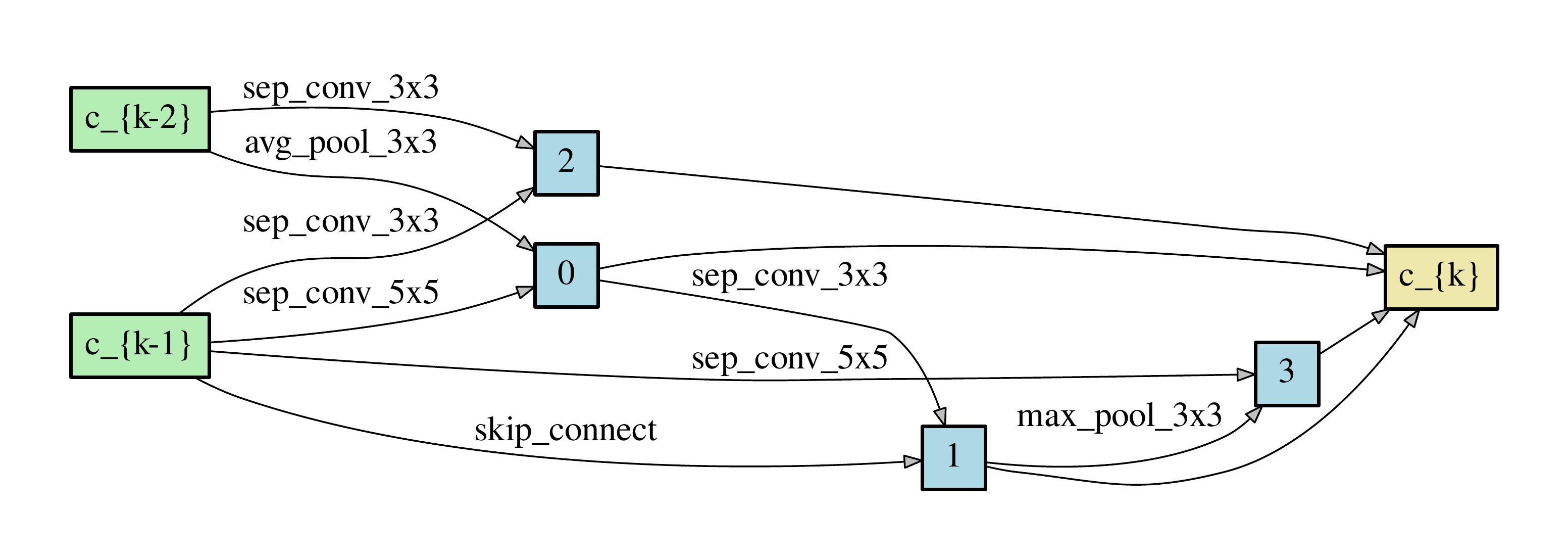}
  \end{minipage}
  \begin{minipage}{4cm}
      \includegraphics[width=4cm,height=2.5cm]{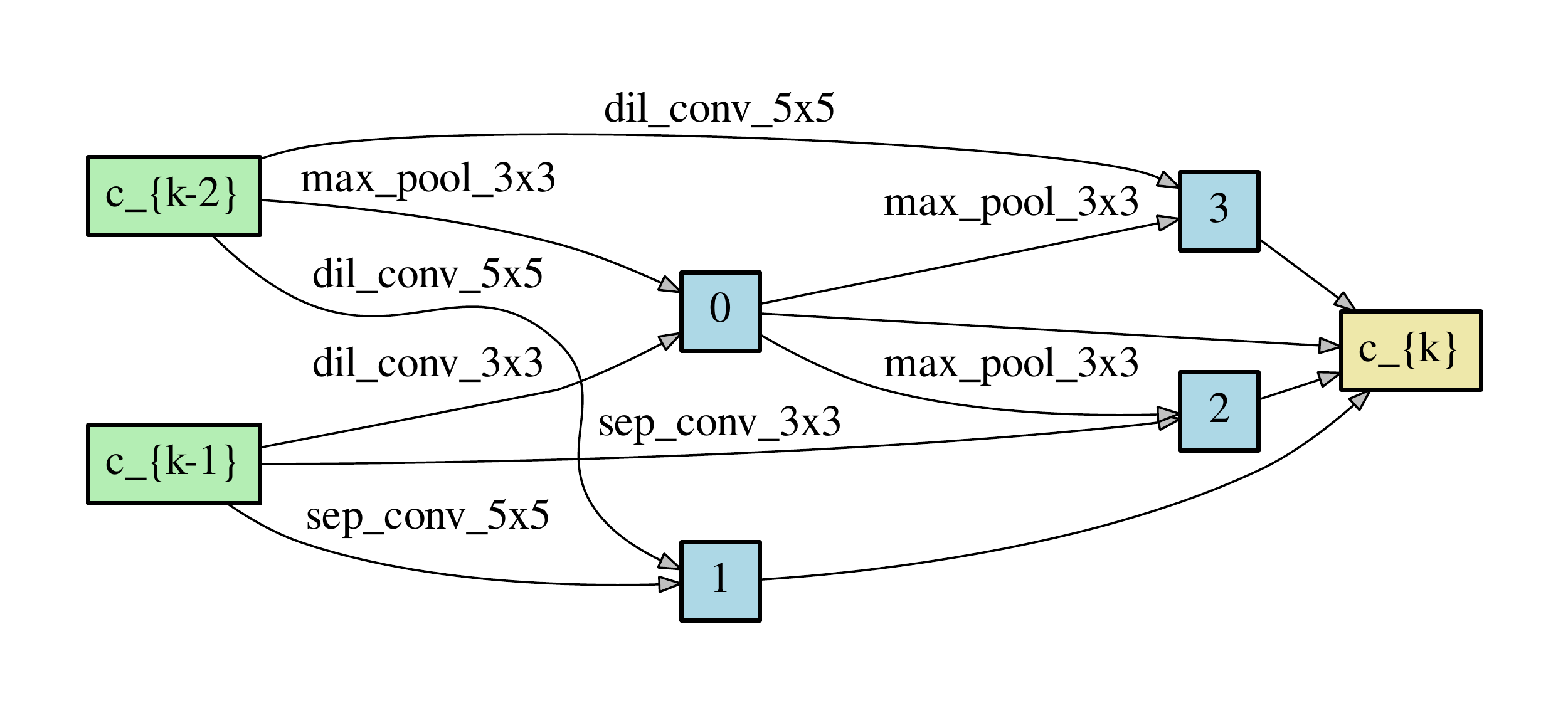}
  \end{minipage} 
  }
   \subfigure[normal and reduction cell by FreeDARTS$\dagger$ with seed4]{
  \begin{minipage}{4cm}
      \includegraphics[width=4cm,height=2.5cm]{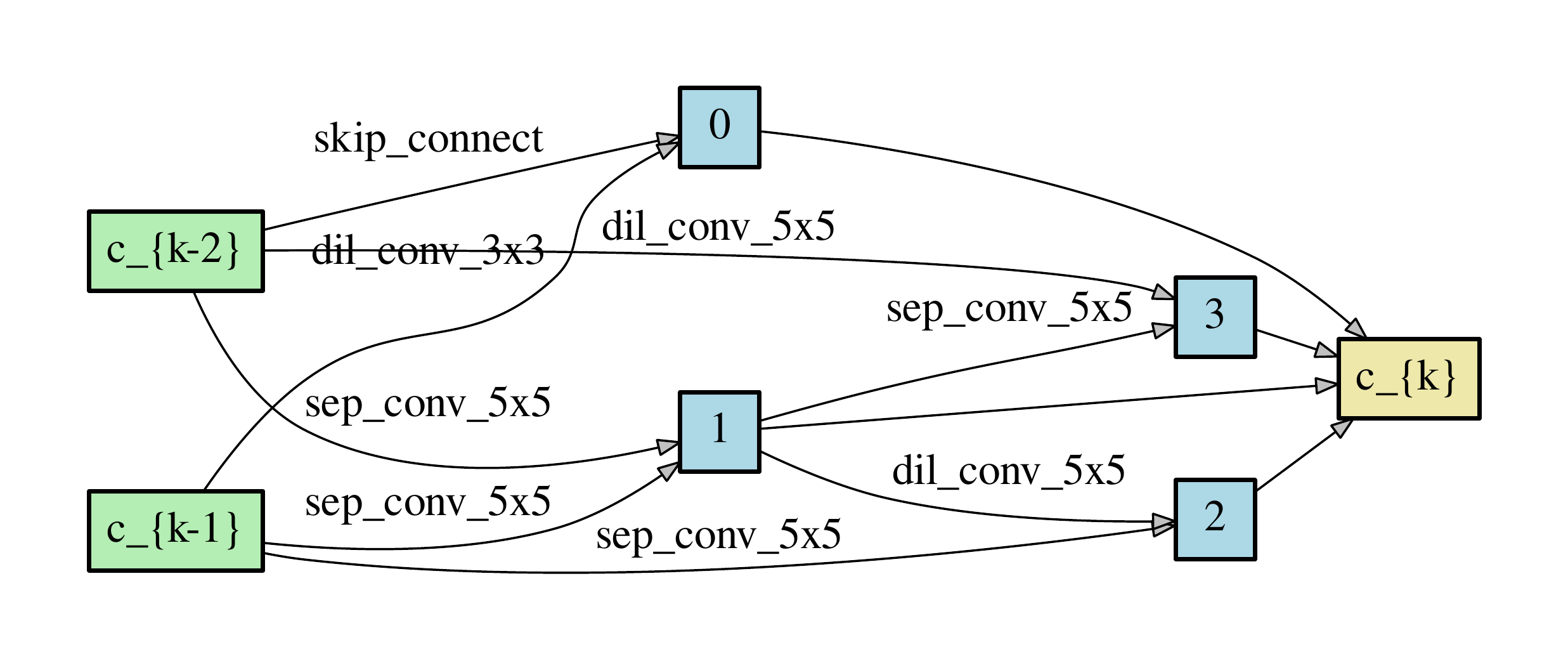}
  \end{minipage}
  \begin{minipage}{4cm}
      \includegraphics[width=4cm,height=2.5cm]{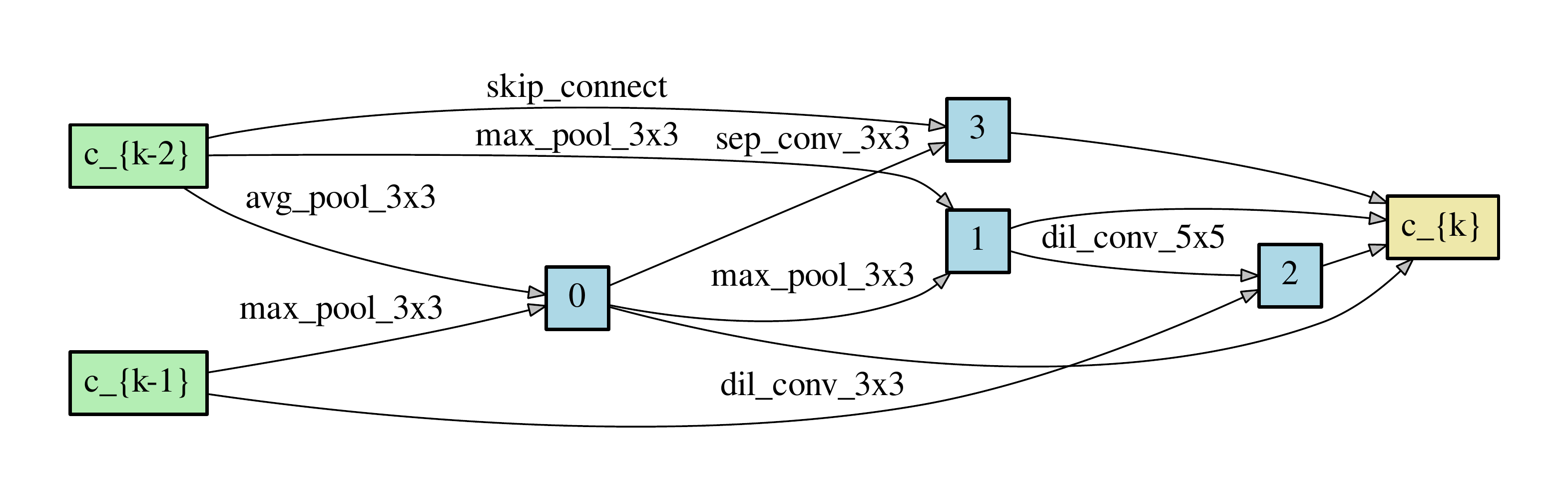}
  \end{minipage} 
  }
  
   \subfigure[normal and reduction cell by FreeDARTS$\ddagger$ with seed0]{
  \begin{minipage}{4cm}
      \includegraphics[width=4cm,height=2.5cm]{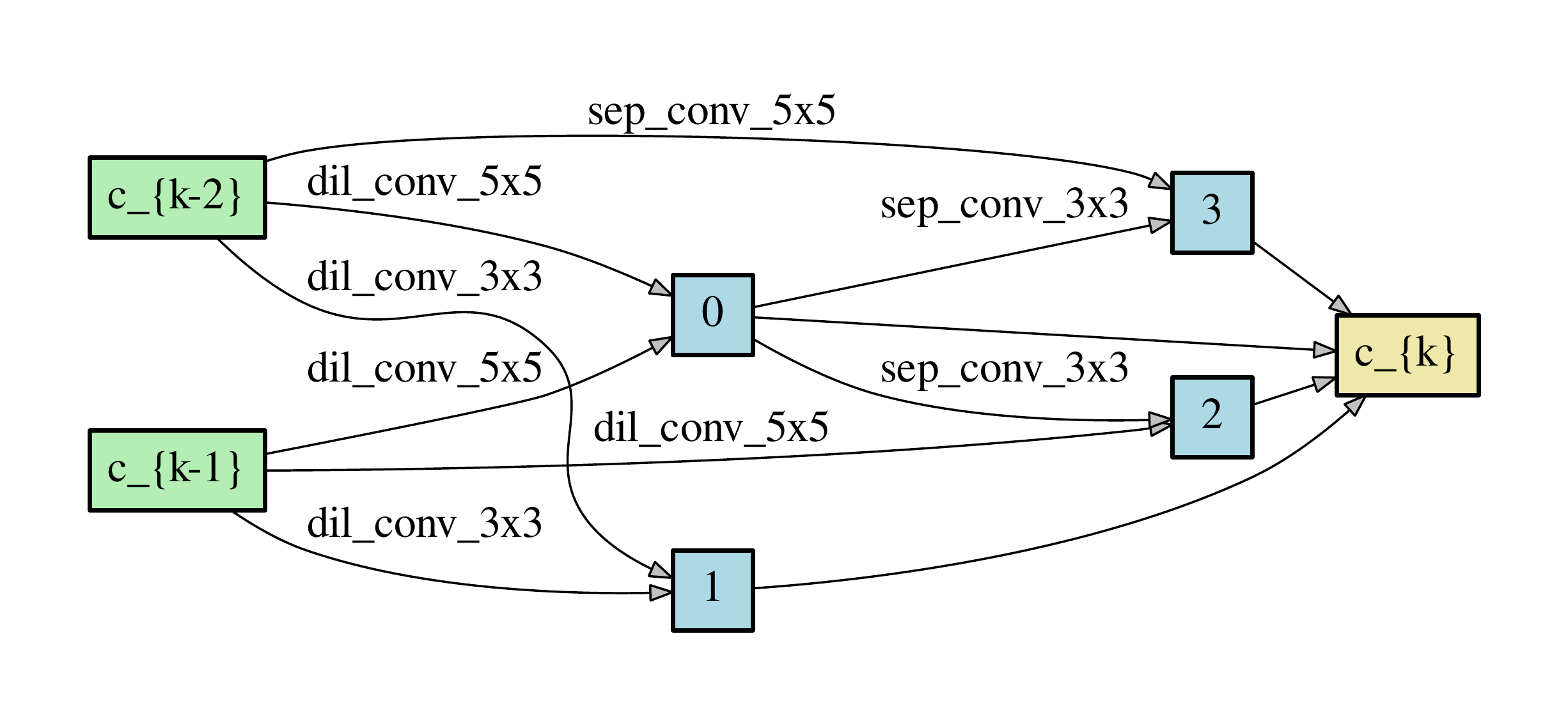}
  \end{minipage}
  \begin{minipage}{4cm}
      \includegraphics[width=4cm,height=2.5cm]{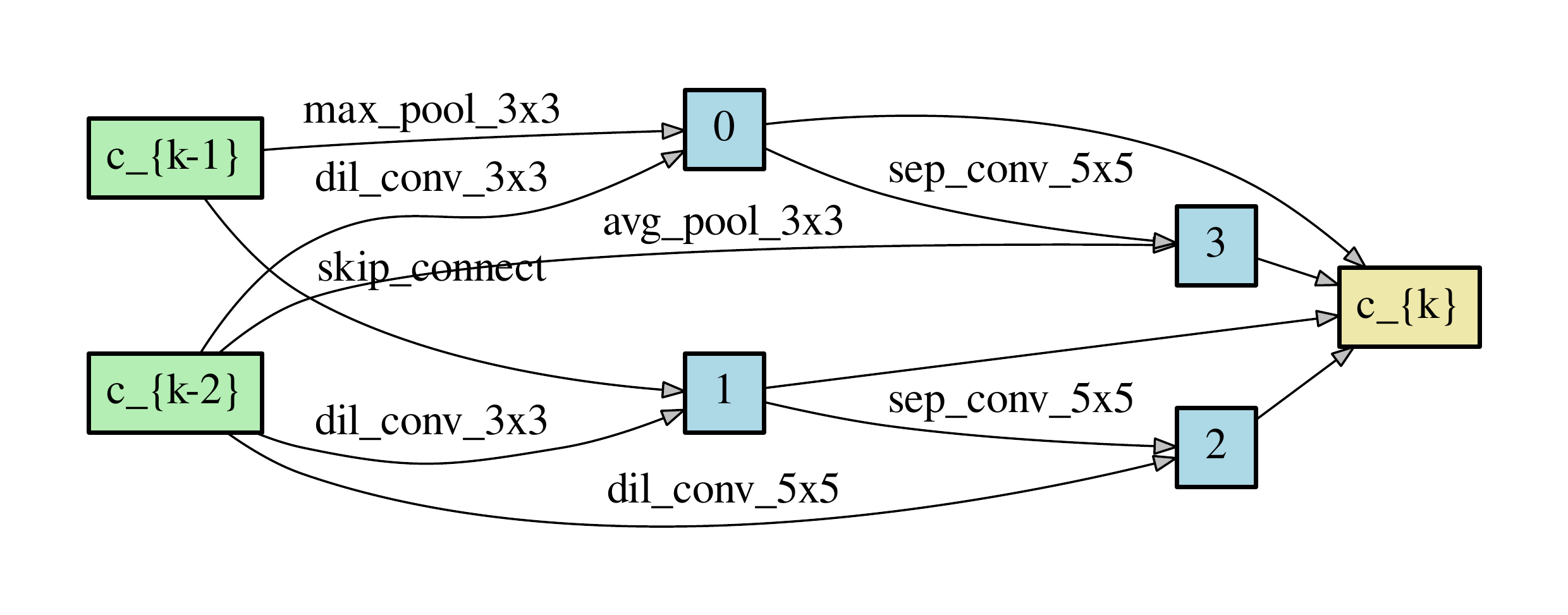}
  \end{minipage} 
  }
   \subfigure[normal and reduction cell by FreeDARTS$\ddagger$ with seed1]{
  \begin{minipage}{4cm}
      \includegraphics[width=4cm,height=2.5cm]{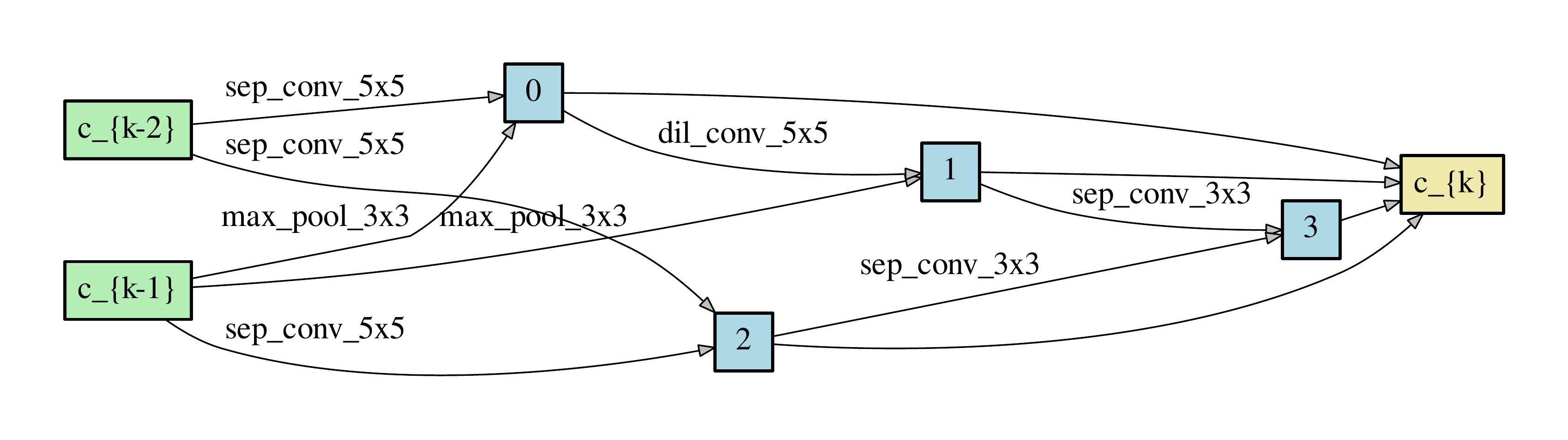}
  \end{minipage}
  \begin{minipage}{4cm}
      \includegraphics[width=4cm,height=2.5cm]{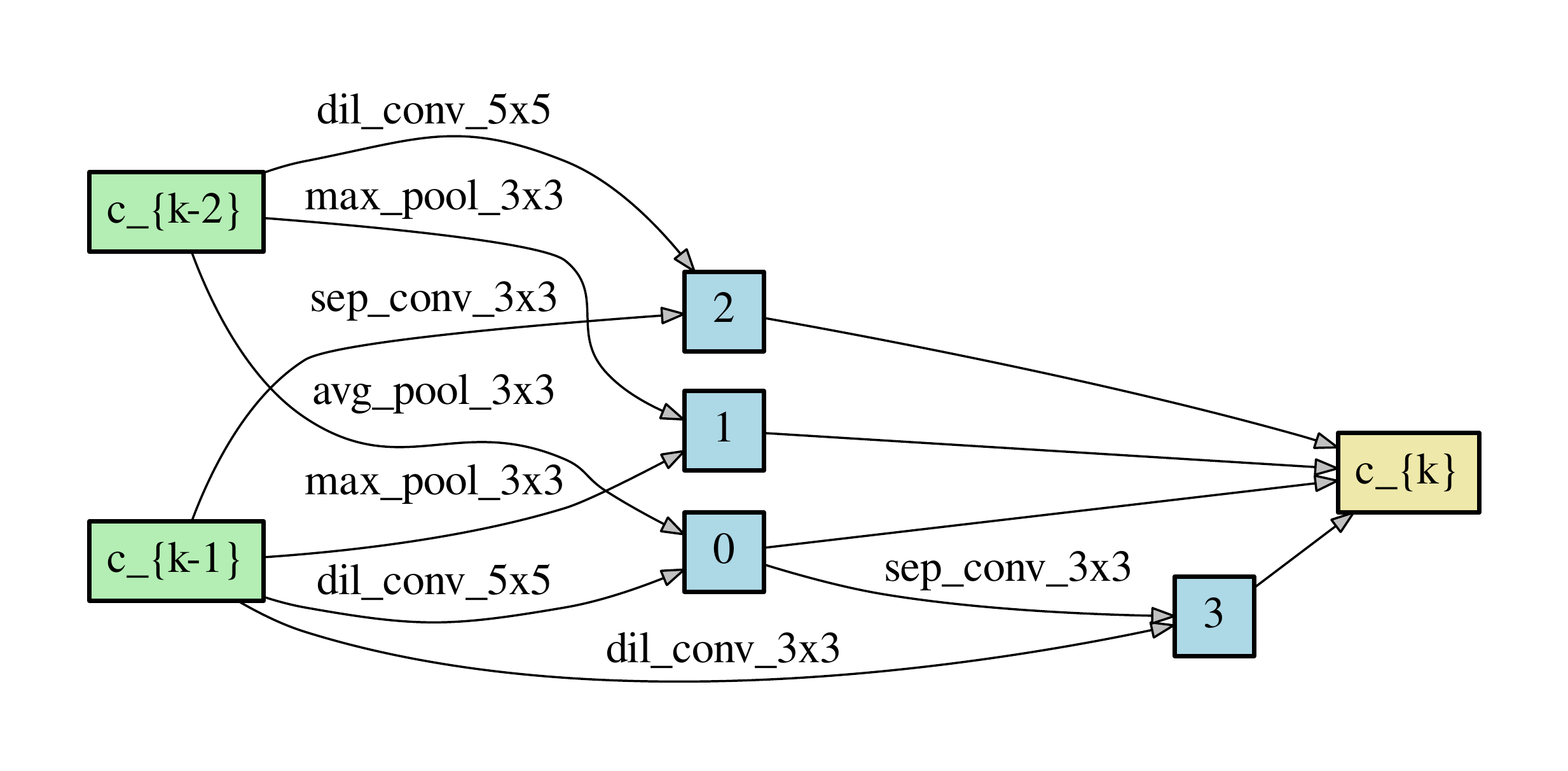}
  \end{minipage} 
  }
  
   \subfigure[normal and reduction cell by FreeDARTS$\ddagger$ with seed2]{
  \begin{minipage}{4cm}
      \includegraphics[width=4cm,height=2.5cm]{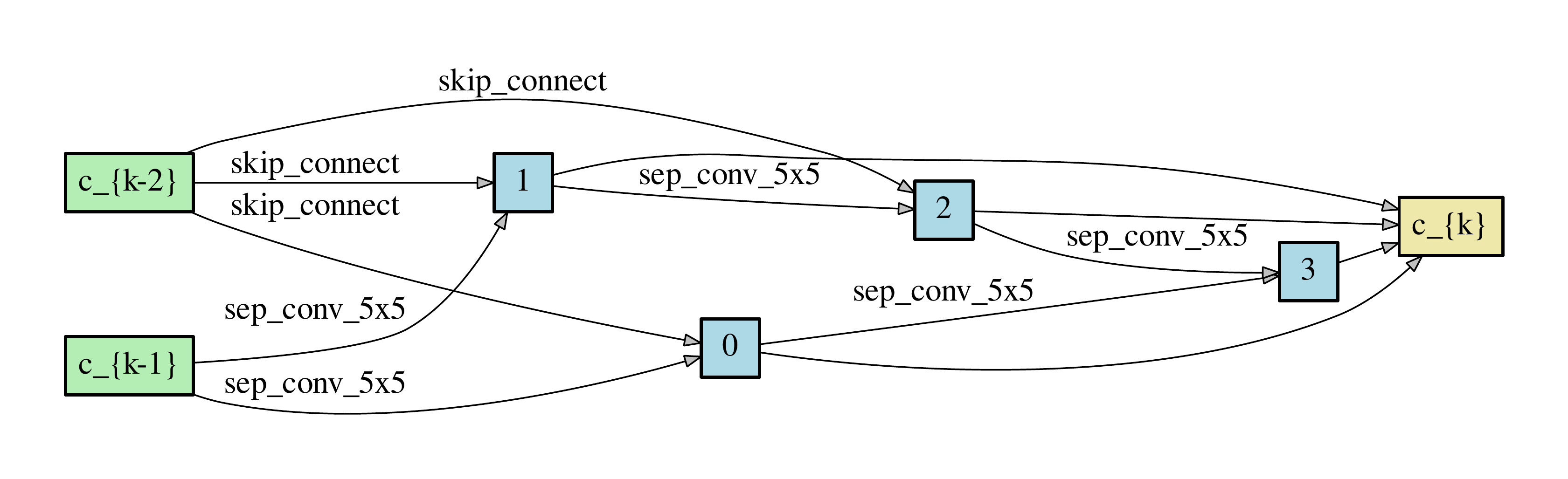}
  \end{minipage}
  \begin{minipage}{4cm}
      \includegraphics[width=4cm,height=2.5cm]{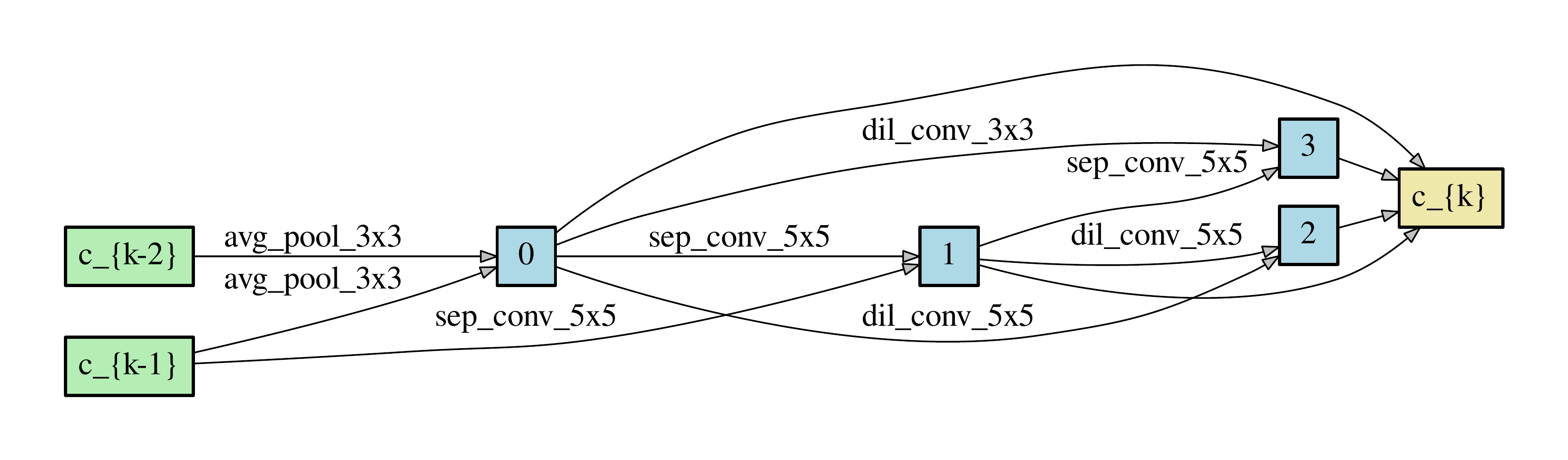}
  \end{minipage} 
  }
   \subfigure[normal and reduction cell by FreeDARTS$\ddagger$ with seed4]{
  \begin{minipage}{4cm}
      \includegraphics[width=4cm,height=2.5cm]{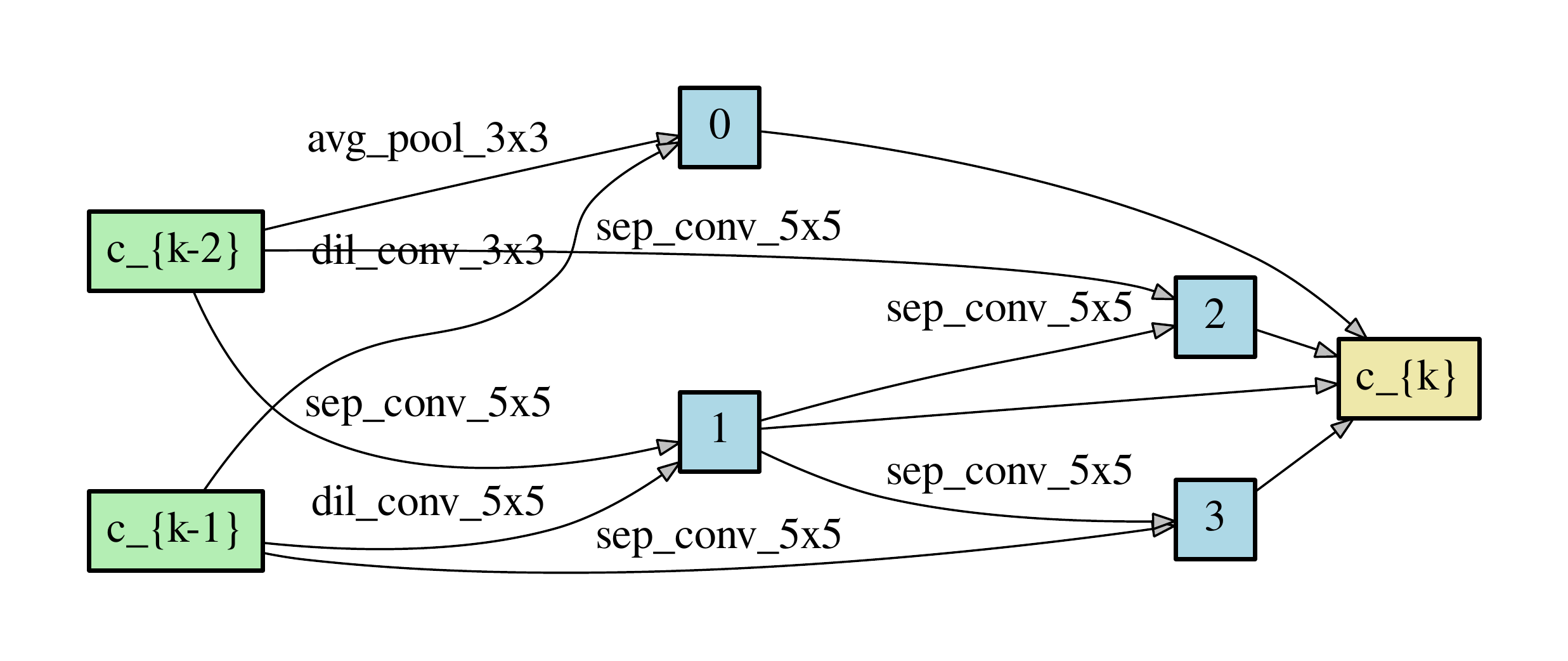}
  \end{minipage}
  \begin{minipage}{4cm}
      \includegraphics[width=4cm,height=2.5cm]{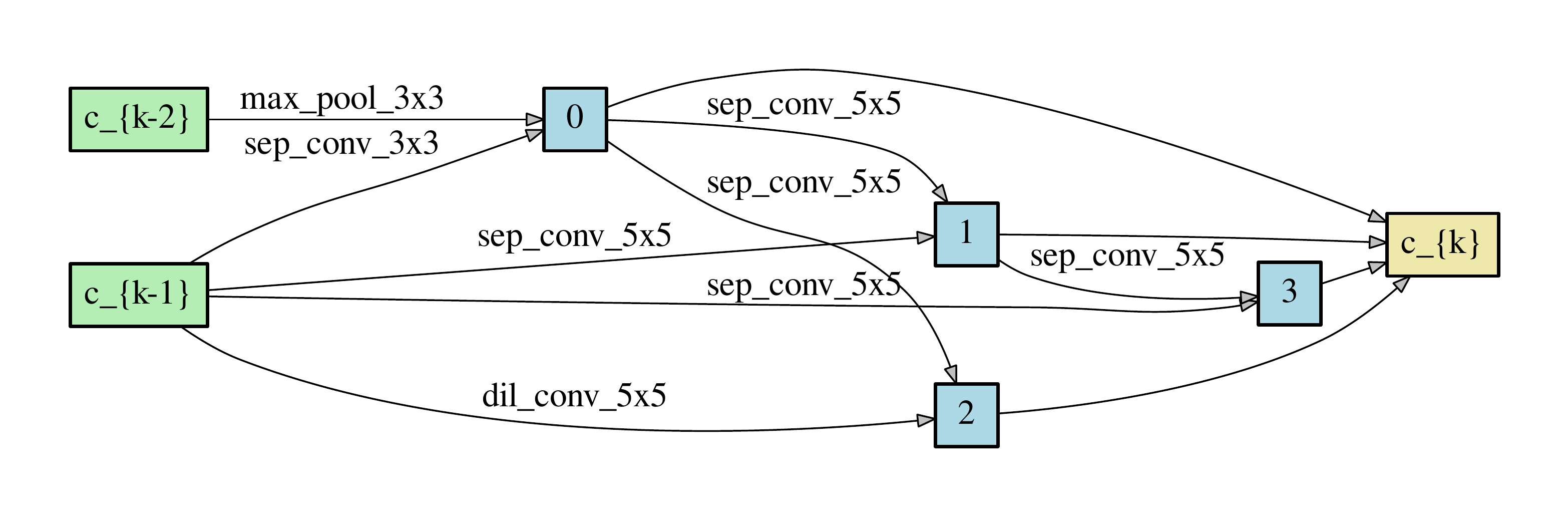}
  \end{minipage} 
  }  
  
 \caption{The cells discovered by FreeDARTS,  FreeDARTS, and FreeDARTS$\ddagger$ with different \textit{random seeds} on the DARTS space.}
 \label{fig:archs_remianings}
\end{figure*}

\section{Proofs}\label{proof_section}

\begin{proof} [Proof of Theorem \ref{lemma_complexity}]

According to \cite{jacot2018neural}, as the width tends to be infinity, the neural tangent kernel follows a recursive relation:
\begin{align} \label{eq:agge}
    \Theta^{(l)}(x,x') = \Theta^{(l-1)}(x,x') \dot \Sigma^{(l)}(x,x') + \Sigma^{(l)}(x,x')  \\
    \hat \Theta^{(l)}(x,x') = \hat \Theta^{(l-1)}(x,x') \dot {\hat \Sigma}^{(l)}(x,x') + \hat \Sigma^{(l)}(x,x')
\end{align}
where we use $\hat \Theta$ to denote neural tangent kernel for network $\hat f$. Then NNGP matrix $\Sigma$ and $\dot \Sigma$ are defined as follows:
\begin{align*}
    {\Sigma}^{(l)}({x,x'}) =  \mathbb{E} \left[ \sigma\left(  h^{(l-1)}(x) \theta^{(l)} \right) \sigma\left(  h^{(l-1)}(x') \theta^{(l)} \right) \right] \\
    \dot {\Sigma}^{(l)}({x,x'}) =  \mathbb{E} \left[ \dot \sigma\left(  h^{(l-1)}(x) \theta^{(l)} \right) \dot \sigma\left(  h^{(l-1)}(x') \theta^{(l)} \right) \right]
\end{align*}
Assume all the previous layers are the same for two neural network, thus at the limit, we have
\begin{align*}
    {\Theta}^{(l-1)}({x,x'}) = \hat {\Theta}^{(l-1)}({x,x'})
\end{align*}

By central limit theorem, we have 
\begin{align*}
    \hat {\Sigma}^{(l)}({x,x'}) &= \rho {\Sigma}^{(l)}({x,x'})
\end{align*}
Similarly, we have,
\begin{align*}
    \hat {\dot \Sigma}^{(l)}({x,x'}) &= \rho {\dot \Sigma}^{(l)}({x,x'})
\end{align*}
Plugging above equations together, we arrive at a conclusion that
\begin{align*}
   \hat \Theta^{(h)}(x,x') = \rho \Theta^{(h)}(x,x')
\end{align*}
Recall that
\begin{align*}
    \mathcal{M}_\textup{Trace}(\Theta)=\sqrt{\left \| \Theta  \right \|_{\textup{tr}}/n} 
\end{align*}
We then have
\begin{align*}
     \lim_{ m \rightarrow \infty }  \mathcal{M}_\textup{Trace}(\hat f)  =  \sqrt{\rho} \lim_{ m \rightarrow \infty } \mathcal{M}_\textup{Trace}(f) 
\end{align*}

\end{proof}

\begin{proof}[Proof of Theorem \ref{lemma:trace}]

For ZEROS score, we have 
\begin{align*}
   \mathcal{F}({\alpha_{k}})  & =  \left | \frac{1}{n}\sum_{i=1}^n \frac{\partial \ell_i(W, \alpha)} { \partial \alpha_{k}} \cdot \alpha_k \right |  = |\alpha_k | \cdot \left | \frac{1}{n} \sum_{i=1}^n \frac{\partial \ell_i(W, \alpha)} { \partial f_i }  \frac{\partial f_i(W, \alpha)} { \partial \alpha_k} \right | \\
  & = |\alpha_k | \cdot \left | \frac{1}{n} \sum_{i=1}^n \frac{\partial \ell_i(W, \alpha)} { \partial f_i }  \frac{\partial f_i}{\partial h_i} \frac{\partial h_i}{\partial \alpha_k} \right |  \le \frac{|\alpha_k|}{n} \cdot \sum_{i=1}^n \left |   \frac{\partial \ell_i(W, \alpha)} { \partial f_i } \frac{\partial f_i}{\partial h_i} \frac{\partial h_i}{\partial \alpha_k}  \right |  \\
  & \le \frac{|\alpha_k|}{n} \cdot \sum_{i=1}^n \left |   \frac{\partial \ell_i(W, \alpha)}  { \partial f_i }\right |  \left \|  \frac{\partial f_i}{\partial h_i}  \right \|_2   \left  \| h_{ik} \right \|_2  \le \frac{|\alpha_k| \beta B}{n} \cdot \sum_{i=1}^n  \| h_{ik}   \|_2  \\
   & \le \frac{|\alpha_k| \beta B}{n} \cdot \sqrt{ n \sum_{i=1}^n  \| h_{ik}   \|^2_2 } \le  {|\alpha_k| \beta B}  \mathcal{M}_\mathrm{trace}(\Sigma_k)   \\
   & \le   {|\alpha_k|  \beta B }    \mathcal{M}_\mathrm{trace}(\Theta_k)
\end{align*}
where $\mathcal{M}_\textup{Trace}(\Sigma_k)=\sqrt{\left \| \Sigma_{k}  \right \|_{\textup{tr}}/n }$. We have used $ \left \|  \frac{\partial f_i}{\partial h_i}  \right \|_2  \le B $, and $ [\Sigma_k]_{ij} = \langle h_{ik}, h_{jk} \rangle $ is the NNGP. Last inequality has used the property of $ \Theta \succ \Sigma $ according to Eq. (\ref{eq:agge}).

\end{proof}

\begin{proof}[Proof of Theorem \ref{theorem:generalization}]

By Theorem 5.1 of \cite{arora2019fine}, the population risk $\mathcal{L}_{\mathcal{D}} = \mathbb{E}_{(x,y)\sim \mathcal{D}(x,y)}[\ell(f_T(x; \theta),y)]$ that is bounded as follows:
\begin{equation}
\begin{aligned}
   \mathcal{L}_\mathcal{D} \le \sqrt{\frac{2 { y^\top  {\Theta}^{-1}(X, X) y} }{n}}+O \bigg(\sqrt{\frac{\log \frac{n}{\lambda_0 \delta}}{n}} \bigg).
\end{aligned}
\end{equation}
Then we find that,
\begin{align*}
 &  \| \Theta  \|_{\mathrm{tr}}  \cdot \| \Theta^{-1}\|_{\mathrm{tr}}   =  \sum_{i=1}^n \lambda_i   \sum_{i=1}^n \frac{1}{\lambda_i}    \le \frac{\lambda_{\max}}{\lambda_{\min}}  n^2 \\
\end{align*} 
where we have used the decomposition for NTK $\Theta = V^\top \Lambda V$, with $V = [v_1, v_2, \cdots, v_n] $, and $\Lambda = \mathrm{diag}( \lambda_1, \lambda_2, \cdots, \lambda_n)$.
Thus we have,
\begin{align*}
    \sqrt{\frac{2 { y^\top  {\Theta}^{-1}(X, X) y} }{n}} \le \sqrt{ \frac{2 \| y\|^2_2 \| \Theta^{-1} \|_\mathrm{tr} }{n}} \le 
    \sqrt{ 2 \frac{\lambda_{\max} n \| y\|^2_2}{\lambda_{\min} \|\Theta \|_\mathrm{tr}}  } = \sqrt{ \frac{2 \lambda_{\max}  \| y\|^2_2 } {\lambda_{\min}}}    /\mathcal{M}_\mathrm{trace}(\Theta,\alpha)
\end{align*}

Finally, the NTK of supernet can be expressed as,
\begin{align*}
    \Theta(\theta,\alpha) & =  \sum_{l=1}^L \sum_{k=1}^{M} \frac{\partial f(\alpha)}{ \partial  \theta^{(l)}_{k} } \left(\frac{\partial f(\alpha)}{ \partial \theta_{k} } \right)^\top = \sum_{l=1}^L \sum_{k=1}^{M} \alpha^2_{lk} \Theta_{lk}
\end{align*}
therefore, we have the following result to complete the proof:
\begin{align*}
\mathcal{M}^2_\mathrm{Trace}(\Theta, \alpha)=\sum_{l=1}^{L}\sum_{k=1}^{M} \alpha^2_{lk} \mathcal{M}^2_{\mathrm{Trace}}(\Theta_{lk})
\end{align*}

\end{proof}


\end{document}